\documentclass[12pt]{article}
\usepackage{amsmath}

\usepackage{graphicx}
\usepackage{cancel}
\usepackage{enumerate}
\usepackage{natbib}
\usepackage{tcolorbox}
\tcbuselibrary{skins}
\usepackage{xr}
\usepackage{url} 

\usepackage{tikz}

\newcommand{\blind}{0}

\usepackage[color]{changebar}
\cbcolor{blue}
\usepackage{hyperref}
\usepackage{amssymb}
\usepackage{xcolor}
\usepackage[ruled,vlined]{algorithm2e}
\usepackage{amsthm}
\usepackage{setspace}
\usepackage[color]{changebar}
\usepackage[normalem]{ulem} 

\definecolor{mutedred}{RGB}{160, 50, 50}
\definecolor{mutedblue}{RGB}{50, 50, 140}
\cbcolor{mutedblue}
\newtheorem{assumption}{Assumption}
\newtheorem{theorem}{Theorem}
\newtheorem{remark}{Remark}
\newtheorem{lemma}{Lemma}
\newtheorem{definition}{Definition}
\newtheorem{corollary}{Corollary}
\newtheorem{proposition}{Proposition}

\newtheorem{STheorem}{Theorem}

\newtheorem{Scorollary}{Corollary}

\newtheorem{SAssumption}{Assumption}

\hypersetup{
    colorlinks=true}
\usepackage{caption}
\usepackage{bbm}

\begin{document}


\def\spacingset#1{\renewcommand{\baselinestretch}%
{#1}\small\normalsize} \spacingset{1}


\if0\blind
{
  \title{\bf POLAR: A Pessimistic Model-based Policy Learning Algorithm for Dynamic Treatment Regimes}
   \author{\small
    Ruijia Zhang$^{1}$, \quad Xiangyu Zhang$^{1}$, \quad Zhengling Qi$^{2}$, \quad Yue Wu$^{1}$,  \quad Yanxun Xu$^{1, *}$\\[1ex]
   \small $^1$ Department of Applied Mathematics and Statistics, Johns Hopkins University \\
   \small  $^2$ School of Business, The George Washington University \\
   \small $^*$ Correspondence should be addressed to yanxun.xu@jhu.edu}
    \date{}
  \maketitle
} \fi

\begin{abstract}
\spacingset{1.8}

Dynamic treatment regimes (DTRs) provide a principled framework for optimizing sequential decision-making in domains where decisions must adapt over time in response to individual trajectories, such as healthcare, education, and digital interventions. However, existing statistical methods often rely on strong positivity assumptions and lack robustness under partial data coverage, while offline reinforcement learning approaches typically focus on average training performance, lack statistical guarantees, and require solving complex optimization problems. To address these challenges, we propose POLAR, a novel pessimistic model-based policy learning algorithm for offline DTR optimization. POLAR estimates the transition dynamics from offline data and quantifies uncertainty for each history-action pair. A pessimistic penalty is then incorporated into the reward function to discourage actions with high uncertainty. Unlike many existing methods that focus on average training performance or provide guarantees only for an oracle policy, 
POLAR directly targets the suboptimality of the final learned policy and offers theoretical guarantees, without relying on computationally intensive minimax or constrained optimization procedures. To the best of our knowledge, POLAR is the first model-based DTR method to provide both statistical and computational guarantees, including finite-sample bounds on policy suboptimality. Empirical results on both synthetic data and the MIMIC-III dataset demonstrate that POLAR outperforms state-of-the-art methods and yields near-optimal, history-aware treatment strategies.


\end{abstract}


\spacingset{1.8} 

\section{Introduction}

Effective decision-making that dynamically responds to evolving circumstances has become increasingly important across many domains, particularly in healthcare. For example, in chronic conditions like diabetes, HIV, and kidney disease, treatment strategies must adapt over time to account for an individual’s disease progression and responses to prior interventions, with the goal of improving long-term outcomes \citep{Tang2012DevelopingAP, Xu2016}.  These scenarios are often formalized as dynamic treatment regimes (DTRs) \citep{murphy2003optimal}, which provide sequential, adaptive frameworks for recommending treatments at multiple decision points.

A fundamental challenge of optimizing sequential decisions from observational data (e.g., electronic medical records) is handling distribution shift, where the learned policy may encounter states and/or actions different from those in the training data \citep{levine2020offline}. This issue is exacerbated in healthcare settings, where clinical decisions are often constrained by physician preferences, institutional guidelines, or resource limitations, leading to limited coverage of the history-action space. Within the statistical literature, substantial work has been developed through the sequential potential outcomes framework, using methods such as the Q-learning, inverse probability weighting, and doubly robust estimators \citep{murphy05a,robins2004optimal,Qian_2011,hua2022personalized,chakraborty2014dynamic,zhao2015new,zhang2013robust}. A related problem is offline reinforcement learning (RL), which typically addresses the distribution shift by adopting conservative strategies. These include model-based methods that penalize reward or transition functions for poorly represented state-action pairs in the offline dataset \citep{uehara2022pessimistic,kidambi2020morel,yu2020mopo,yin2021near}, and model-free algorithms that establish conservative lower bounds for the optimal Q-function \citep{kumar2019stabilizing,kumar2020conservative,fujimoto2019off,buckman2020importance,jin2021pessimism,fu2022offline,bai2022pessimistic,zanette2021provable,wu2019behavior}. 

Despite these advances, both statistical DTR and offline RL approaches face limitations. Statistical DTR methods commonly rely on the positivity assumption, which requires that each action has a non-zero probability of being observed for any given patient history. However, this assumption is often violated in real-world observational data due to clinical guidelines and physician preferences, resulting in limited coverage of possible patient history-action combinations \citep{luckett2020estimating}. Offline RL methods, on the other hand, have been developed to address distribution shift in sequential decision-making but often lack statistical guarantees measured by sample size and the number of decision points. Additionally, offline RL algorithms typically rely on the Markov assumption, which is often violated in clinical settings where treatment decisions depend on the patient’s full history rather than just the current state, limiting the effectiveness of these methods.  More importantly, many existing offline RL algorithms focus on metrics such as average policy performance during training or $\max_{0 \le t \le T}$ bounds \citep{hong2024moma,zanette2021provable}, or on the performance of an oracle policy within a pessimistic model \citep{chang2021mitigating,uehara2022pessimistic},  rather than directly evaluating the suboptimality of the final learned policy.
\cbend
To address these limitations, we develop POLAR, a novel pessimistic model-based policy learning algorithm for optimizing DTRs in offline settings. Unlike existing statistical and RL methods, POLAR explicitly accounts for uncertainty in poorly represented regions of the data, effectively addressing challenges related to partial data coverage and violations of the Markov assumption. POLAR first estimates transition dynamics from offline data   and quantifies uncertainty for all history-action pairs.  By incorporating a pessimistic penalty into the estimated reward function, POLAR discourages selecting actions associated with high uncertainty.  Using this modified reward function and the estimated transition dynamics, POLAR iteratively updates the policy within an actor-critic framework, thus enabling efficient policy learning without solving complex optimization subproblems. 

POLAR contributes to the literature in several ways. First, to the best of our knowledge, it is the first model-based DTR method to provide both statistical and computational guarantees, offering rigorous convergence bounds on the suboptimality of the learned policy as a function of sample size and decision stages.    Second, POLAR addresses the challenge of partial coverage in DTR optimization, where the offline data distribution does not fully span the history-action space. By incorporating principled uncertainty penalization, POLAR establishes a robust framework to mitigate the effects of partial data coverage. Third, POLAR supports general function classes, including flexible nonparametric estimators for transition dynamics, enhancing its accuracy and adaptability to complex, high-dimensional healthcare data. Finally, unlike many existing methods that optimize average performance during training, POLAR explicitly minimizes the suboptimality of the final learned policy, contributing new theoretical insights and advancing the practical effectiveness of DTR methods in real-world applications.

The rest of the paper is organized as follows. Section~\ref{sec: preliminary} introduces the preliminary notations and definitions. Section~\ref{sec: method} describes the proposed POLAR algorithm, providing details on its key components and implementation. In Section~\ref{sec: theory}, we establish theoretical guarantees for POLAR. Section~\ref{sec: simulation} evaluates POLAR through simulation studies, comparing its performance against alternative methods. In Section~\ref{sec: real data}, we apply POLAR to the real-world MIMIC-III dataset, showcasing its practical utility and effectiveness in a healthcare setting. Finally, Section~\ref{sec: discussion} concludes the paper with a discussion.

\section{Preliminaries}
\label{sec: preliminary}
In this section, we introduce the notation and formally define the problem within the DTR framework, which aims to determine the optimal policy for a decision-making process with $K$ stages. At each stage $k$, the decision-maker observes the current state $S_k$ and selects an action $A_k$. Denote $H_k$ as the  sequence of state-action pairs up to stage $k$, referred to as the history. It can be recursively defined as: 
$H_k = (H_{k-1}, A_{k-1}, S_k)$ with $H_1 = S_1.$ 
The entire decision-making process over $K$ stages can then be represented as:
$(S_1, A_1, S_2, A_2, \dots, S_k, A_k, S_{k+1}, \dots, S_K, A_K),$
where $S_k \in \mathcal{S}_k$, $A_k \in \mathcal{A}_k$, and $H_k \in \mathcal{H}_k$. For simplicity, we assume (i) The action spaces $\mathcal{A}_k$ are finite and known; (ii) The state spaces $\mathcal{S}_k$ are compact subsets of Euclidean spaces and known.

A policy $\pi = (\pi_1, \dots, \pi_K)$ is defined as a sequence of decision rules, where $\pi_k$ specifies a history-dependent rule for selecting the action $A_k$ at stage $k$ based on the history $H_k$.  
Formally,  \( \pi_k \) maps \( H_k \) to a probability distribution over the action space \( \mathcal{A}_k \), such that $A_k \sim \pi_k(\cdot \mid H_k).$
Given the current history $H_k$ and action $A_k$, the next state $S_{k+1}$ is determined by the transition dynamics $P_k$: $S_{k+1} \sim P_k(\cdot \mid H_k, A_k).$
We denote the sequence of transition dynamics across all $K$ stages as $P = (P_1, \dots, P_K)$.
At each stage $k$, an intermediate reward $R_k$ is observed after taking the action $A_k$. Without loss of generality, $R_k$ is defined as $R_k = \bar{r}_k(H_k, A_k, S_{k+1}),$ where $\bar{r}_k : \mathcal{H}_k \times \mathcal{A}_k \times \mathcal{S}_{k+1} \to \mathbb{R}$ is a deterministic function. The randomness in $R_k$ comes entirely from $S_{k+1}$, conditional on the history $(H_k, A_k)$. The corresponding expected reward function is given by: $r_k(h_k, a_k) = \mathbb{E}[R_k \mid H_k = h_k, A_k = a_k] = \mathbb{E}_{S_{k+1} \sim P_k(\cdot \mid h_k, a_k)}[\bar{r}_k(h_k, a_k, S_{k+1})].$ We assume that $r_k$ is uniformly bounded for all $k$.
With these components, a DTR model can be formally represented as:
\[
M = \Big(\{\mathcal{S}_k\}_{k=1}^K, \{\mathcal{A}_k\}_{k=1}^K, \{P_k\}_{k=1}^K, \{r_k\}_{k=1}^K, \mu_1 \Big),
\]
where $\{\mathcal{S}_k\}_{k=1}^K$ are the state spaces, $\{\mathcal{A}_k\}_{k=1}^K$ are the action spaces, $\{P_k\}_{k=1}^K$ are the transition dynamics, $\{r_k\}_{k=1}^K$ are the expected reward functions, and $\mu_1$ is the known initial state distribution. For brevity, we write $M = (P = \{P_k\}_{k=1}^K, r = \{r_k\}_{k=1}^K).$ 

Given the DTR model $M$ and a policy $\pi$, we evaluate the performance of $\pi$ using its cumulative rewards. Specifically, the value function at stage $k$ is defined as: $V_{k, M}^\pi(h_k) := \mathbb{E}^\pi_{M} \left[ \sum_{t=k}^K R_t \mid H_k = h_k \right],$
which measures the expected cumulative reward of policy $\pi$ starting from stage $k$ with the history $h_k$.
Additionally, we define the action-value function (Q-function) as: $Q_{k, M}^\pi(h_k, a_k) := \mathbb{E}^\pi_{M} \left[ \sum_{t=k}^K R_t \mid H_k = h_k, A_k = a_k \right].$ Both expectations are taken with respect to the probability distributions induced by the policy $\pi$ and the transition dynamics in $M$.
To evaluate the overall performance of a policy $\pi$, we consider the expected value function over the initial state distribution $\mu_1$. Since $H_1 = S_1$ by definition, the overall value of $\pi$ is defined as:
$V_M^\pi = \mathbb{E}_{H_1 \sim \mu_1} \left[ V_{1, M}^\pi(H_1) \right].$

Our goal is to find an optimal policy $\pi^*$ that maximizes the cumulative reward in a DTR problem using a static offline dataset $\mathcal{D}$. Specifically, we find $\pi^* \in \arg\max_{\pi \in \Pi} V_M^\pi,$
where $\Pi$ is a pre-specified policy class. The offline dataset $\mathcal{D}$ consists of $n$ independent and identically distributed (i.i.d.) trajectories generated by a behavior policy $\pi_b$ under the true model $M^*$:  $\mathcal{D} = \{( s_1^{(i)} , a_1^{(i)}, r_1^{(i)}, ..., s_K^{(i)}, a_K^{(i)}, r_K^{(i)})\}_{i=1}^n$.  Here, $M^* = \left( \{\mathcal{S}_k\}_{k=1}^K, \{\mathcal{A}_k\}_{k=1}^K, \{P_k^*\}_{k=1}^K, \{r_k\}_{k=1}^K, \mu_1 \right)$, where $P^* = \{P_k^*\}_{k=1}^K$ represents the ground truth transition dynamics.


\section{Method}
\label{sec: method}
We propose POLAR, the first pessimistic model-based algorithm for DTR with both statistical and computational convergence guarantees. We begin with an overview of POLAR in Section~\ref{subsec: method overview}, followed by a detailed construction of the modified DTR model in Section~\ref{subsec: method modified DTR}. We then describe the policy optimization procedure in Section \ref{sec:ACframe}.

\subsection{Overview of POLAR}
\label{subsec: method overview}
POLAR aims to learn an optimal policy $\pi^*$ within a pre-specified policy class by constructing a modified DTR model that penalizes state-action regions with high uncertainty in the reward function.  The algorithm begins by estimating the transition dynamics \( \widehat{P} = \{ \widehat{P}_k \}_{k=1}^K \) to approximate the true transition dynamics \( P^* \) using the offline dataset \( \mathcal{D} \) through maximum likelihood estimation (MLE). Specifically, \( \widehat{P}_k(\cdot \mid h_k, a_k) \) represents the estimated transition kernel conditional on the history \( h_k \) and the current action \( a_k \).
Next, POLAR computes the estimated reward functions \( \widehat{r} = \{ \widehat{r}_k \}_{k=1}^K \), where
$\widehat{r}_k(h_k, a_k) = \mathbb{E}_{S_{k+1} \sim \widehat{P}_k(\cdot \mid h_k, a_k)} \left[ \overline{r}_k(h_k, a_k, S_{k+1}) \right]$ and $\overline{r}_k$ denotes the known deterministic reward function at stage \( k \).

Directly optimizing a policy using the estimated model \( \widehat{M} = (\widehat{P}, \widehat{r}) \) can be suboptimal due to inaccuracies in \( \widehat{P} \), especially in poorly represented state-action regions of the offline dataset. This arises from the distributional shift between the offline data and the distributions induced by policies other than the behavior policy. To mitigate this issue, we quantify the estimation uncertainty of $\widehat{P}_k$ by $\Gamma_k (h_k,a_k)$,  and incorporate it into a penalized reward function.  Specifically, we define a modified reward function  $\widetilde{r} = \{\widetilde{r}_k \}_{k=1}^K$ as:
$\widetilde{r}_k (h_k,a_k) = \widehat{r}_k (h_k,a_k) - \widetilde{c}_k \Gamma_k (h_k,a_k)$, where $\widetilde{c}_k \geq 0$ is a tuning multiplier that controls the degree of pessimism. This penalized reward function discourages the policy from visiting state-action pairs with high uncertainty.
The modified DTR model, denoted as  $\widetilde{M} = \left( \widehat{P}, \widetilde{r} \right)$, combines the estimated transition dynamics $\widehat{P}$ with the pessimistically adjusted reward functions $\widetilde{r}$. POLAR optimizes the policy by maximizing the value function under this modified model.  We will discuss the detailed construction of $\{ \widehat{P}_k \}_{k=1}^K$, $\{ \widehat{r}_k \}_{k=1}^K$ and $\{ \Gamma_k \}_{k=1}^K$ in Section \ref{subsec: method modified DTR}.

To facilitate policy optimization in general state spaces, we consider a parameterized policy class using soft-max policies:
\begin{equation*}
\Pi = \{\pi = (\pi_1,...,\pi_K): \pi_k \in \Pi_k, \forall k \}, \quad \text{where}
\end{equation*}
\begin{equation}
\label{eq: policy class}
\Pi_k= \left\{ \pi_k \, \mid \, \pi_k(a_k|h_k, \theta_k) = \frac{\exp(f_k(\theta_k, h_k, a_k))}{\sum_{a'\in \mathcal{A}_k} \exp(f_k(\theta_k, h_k, a'))} : \theta_k \in \Theta_k \right\}.
\end{equation}
Here $f_k(\theta_k,h_k,a_k)$ is an importance-weighting function defined on $\Theta_k\times \mathcal{H}_k \times \mathcal{A}_k$, capturing the relative importance of action $a_k$ at stage $k$. The soft-max policy $\pi_k(\cdot|h_k, \theta_k)$, commonly used in the literature, defines a probability distribution over $\mathcal{A}_k$  and is parameterized by $\theta_k$. For clarity, we use $\pi(\theta)=(\pi_1(\theta_1),...,\pi_K(\theta_K))$ with $\theta = (\theta_1,...,\theta_K)$. 
The POLAR algorithm, summarized in Algorithm~\ref{algo: general}, follows a model-based actor-critic \citep{SuttonBarto2018} framework to optimize the policy under the modified DTR model within the policy class $\Pi$. 
The actor is the parameterized policy $\pi(\theta)$, and the critic estimates the action-value function under the current policy. At each iteration $t$, the critic first computes the stage-wise $Q$-functions $Q^{\pi^{(t)}}_{k,\widetilde M}(h_k,a_k)$ under the modified model $\widetilde M$. The actor then updates $\theta$ so that actions with larger estimated $Q$ values are
selected with higher probability under the soft-max policy. Full details of our actor-critic framework are given in Section~\ref{sec:ACframe}.
\cbend


\SetAlgoLined
\LinesNumbered
\SetAlgoLongEnd

\begin{algorithm}[ht]
\caption{An overview of the proposed algorithm}
\label{algo: general}
\KwIn{Offline dataset $\mathcal{D}$; stepsizes $\{\eta_k^{(t)}\}$ at $k$-th stage and $t$-th iteration; hyperparameters $\{\widetilde{c}_k\}$}

Construct estimated transition dynamics $\widehat{P}$, estimated reward functions $\widehat{r}$ and uncertainty quantifiers $\Gamma$ from the offline data $\mathcal{D}$ \\
Construct the modified DTR model $\widetilde{M} = (\widehat{P}, \widetilde{r})$, with $\widetilde{r}_k = \widehat{r}_k - \widetilde{c}_k \Gamma_k$ \\
Initialize \( \theta^{(0)} = (\theta^{(0)}_1, \dots, \theta^{(0)}_K) \) such that \( \pi_k(\cdot \mid h_k, \theta^{(0)}_k) \) is uniformly distributed over \( \mathcal{A}_k \) for all \( k \) and any \( h_k \in \mathcal{H}_k \).
\\
$\pi^{(0)}\gets \pi(\theta^{(0)}), t=0 $\\
\textbf{For} $t = 0, 1, 2, \cdots, T-1$
\quad Compute the modified Q-function for each $k$: $\widehat{Q}_k^{(t)}\gets Q_{k, \widetilde{M}}^{\pi^{(t)}}$\\
\quad Update $\theta$: for each $k$, $\theta_k^{(t+1)}\gets \arg\min_{\theta\in \Theta_k} \mathbb{E}_{h_k\sim unif(\mathcal{H}_k),a_k\sim unif(\mathcal{A}_k)}\left( f_k(\theta, h_k,a_k) - f_k(\theta_k^{(t)},h_k,a_k) - \eta_k^{(t)} \widehat{Q}_k^{(t)}(h_k,a_k) \right)^2.$ \\
\KwOut{$\pi^{(T)}$}
\end{algorithm}

\subsection{The modified DTR}
\label{subsec: method modified DTR}


We first estimate the transition model $\widehat{P}$ within a pre-specified transition model class $\mathcal{P}$ via MLE: $\widehat{P}_k = \arg \max_{P\in \mathcal{P}} \sum_{i=1}^n  \log P( s_{k+1}^{(i)} \mid h_{k}^{(i)}, a_{k}^{(i)})$, where ${P}{(\cdot \mid h_k, a_k)} \in \mathcal{P}$ is the parameterized transition dynamics given the history  $h_k$ and the current action $a_k$. We then introduce an uncertainty quantifier to measure the estimation uncertainty in \( \widehat{P} \), which arises from the randomness in the dataset \( \mathcal{D} \sim \mathbb{P}_\mathcal{D} \), where \( \mathbb{P}_\mathcal{D} \) represents the data collection process.   To characterize this uncertainty, we define a family of uncertainty quantifiers that provide bounds on the deviation between the estimated model \( \widehat{P} \) and the true model \( P^* \).

\begin{assumption}[Uncertainty quantifier]
\label{assumption: uncertainty quantifier} For each of the following events,
$$\mathcal{E}_k = \left\{ \left\|\widehat{P}_k(\cdot|h_k,a_k) - P^*_k(\cdot |h_k,a_k)\right\|_1\le \Gamma_k(h_k,a_k),\ \forall (h_k,a_k)\in \mathcal{H}_k \times \mathcal{A} _k \right\},$$
it holds with probability at least $1-\delta$ (with respect to $\mathbb{P}_\mathcal{D}$). That is, $\mathbb{P}_ \mathcal{D} (\mathcal{E}_k) \ge 1-\delta$ for each $k=1,...,K$. 
\end{assumption}
Assumption \ref{assumption: uncertainty quantifier} directly implies that $\mathbb{P}_\mathcal{D}(\mathcal{E}_1 \cap ... \cap \mathcal{E}_K) \ge 1-K\delta$. It guarantees that with high probability, the $L_1$ distance between $\widehat{P}_k$ and $P^*_k$ can be upper bounded by the uncertainty quantifier $\Gamma_k$ uniformly over the state-action space. We present two examples of uncertainty quantifiers that satisfy Assumption  \ref{assumption: uncertainty quantifier}, with proof and more details deferred to Supplementary Section \ref{sec: discussions on UQ}. 


\noindent \textbf{Example 1:} (Linear Transition Model)
Suppose the ground truth transition model follows a linear structure: 
\begin{equation}
\label{eq: linear transition model}
s_{k+1} = W_k \phi_k(h_k,a_k) + \varepsilon_k, \quad \varepsilon_k \sim F_k, 
\end{equation}
where $\phi_k: \mathcal{H}_k \times \mathcal{A}_k \to \mathbb{R}^{\dim(\phi_k)}$ is a known bounded feature mapping, $F_k$ is a known distribution on $\mathbb{R}^{d_{k+1}^{(s)}}$, and $d_{k+1}^{(s)}$ is the dimension of the $(k+1)$-th state space $\mathcal{S}_{k+1}$. $W_k\in \mathbb{R}^{d_{k+1}^{(s)} \times \dim(\phi_k)}$ is a weight matrix.
\cbend
The uncertainty quantifiers $\Gamma_k$ take the following form:
\begin{equation}
\label{eq: KNR, Gamma}
\Gamma_k(h_k, a_k) = \min \left\{2, 2C_2 \sqrt{\phi_k^T (h_k,a_k) \Lambda_{k,n}^{-1} \phi_k(h_k,a_k)} \right\},
\end{equation}
where $C_2$ is a constant depending on the noise $\varepsilon_k$, the dimension of the state space $\mathcal{S}_{k+1}$,  and the conditional distribution $P_k(\cdot | h_k,a_k).$ The exact form of $C_2$  is provided in Supplement \ref{subsec: kernelized}. The term $\Lambda_{k,n}$ arises from the ridge regularization in estimating $W_k$ and is defined as:
$\Lambda_{k,n} = \sum_{i=1}^{n} \phi_k( h_k^{(i)}, a_k^{(i)}) \phi_k( h_k^{(i)}, a_k^{(i)})^T + \lambda I$,
where $\lambda \ge 0$ is the  regularization parameter.

\noindent \textbf{Example 2: }(Gaussian Process)
Suppose the ground truth transition model follows a Gaussian Process (GP) structure: $s_{k+1} = g_k(h_k, a_k) + \varepsilon_k$, where $\varepsilon_k \sim \mathcal{N}(0, \sigma^2 I)$.  Define $x:= (h_k, a_k)$ when $k$ is fixed. The function $g: \mathcal{X}\to \mathbb{R}^{d_{k+1}^{(s)}}$ is the mean transition dynamics, with ${d_{k+1}^{(s)}}$ denoting the dimension of the next state. We assume $g$ follows a GP prior: $g \sim GP(0, h(\cdot,\cdot))$, with zero mean and covariance function $h:  \mathcal{X}\times \mathcal{X} \to \mathbb{R}$, where $h$ is a known symmetric positive semidefinite kernel.

Given observations $\{ (x_i,y_i) \}_{i=1}^n$, the posterior distribution of $g$ remains a GP with mean function $\widehat{g}(\cdot)$ and covariance function $\widehat{h}(\cdot,\cdot).$ The uncertainty quantifier $\Gamma$ takes the form:

\begin{equation}
\label{eq: GP, eq 3}
\Gamma(x) = \frac{\beta_n}{\sigma} \sqrt{\widehat{h}(x,x)},
\end{equation}
where \( \beta_n \) is a constant depending on the state space dimension \( d_{k+1}^{(s)} \) and the sample size \( n \). The form of \( \beta_n \) is specified in  Supplement \ref{subsec: GP}.



After obtaining the uncertainty quantifier, we define the modified expected reward function at stage $k$  as 
$ \widetilde{r}_k(h_k,a_k)= \widehat{r}_k(h_k,a_k)- \widetilde{c}_k \Gamma_k(h_k,a_k)$, 
 for some positive constant $\widetilde{c}_k$ and all $(h_k,a_k) \in \mathcal{H}_k \times \mathcal{A}_k$.
This uncertainty penalization approach follows the pessimism principle in RL, where similar techniques have been used to modify the expected reward function \citep{yu2020mopo, chang2021mitigating}. However, these works do not establish convergence guarantees, whereas we provide a rigorous theoretical analysis in Section~\ref{sec: theory}.

\subsection{The actor-critic framework}
\label{sec:ACframe}
We now describe how to estimate the optimal policy after constructing the modified DTR model. POLAR algorithm adopts an actor-critic framework, where the critic evaluates the Q-function under the modified model, and the actor updates the policy based on this evaluation.

The initialization step (line 3) in Algorithm \ref{algo: general} assumes that the function class \( \{f_k(\theta_k, \cdot, \cdot): \theta_k \in \Theta_k\} \) contains at least one function that is constant with respect to \( a_k \) for any given $h_h$.  This is a mild assumption that is typically satisfied by common function classes, such as linear functions or neural networks. Following initialization, the algorithm iteratively updates the policy parameters. At each iteration, the Q-function of the current policy  \( \pi(\theta^{(t)}) \) is evaluated under the modified DTR model \( \widetilde{M} \), then the policy parameters are updated by solving the optimization problem:  
\begin{equation}
\label{eq: optimization step 1}
\arg \min_{\theta\in \Theta_k} \mathbb{E}_{h_k\sim unif(\mathcal{H}_k),a_k\sim unif(\mathcal{A}_k)}\left( f_k(\theta, h_k,a_k) - f_k(\theta_k^{(t)},h_k,a_k) - \eta_k^{(t)} \widehat{Q}_k^{(t)}(h_k,a_k) \right)^2.
\end{equation}
For soft-max policies as defined in \eqref{eq: policy class}, this minimization is equivalent to finding a policy update such that 
\begin{equation*}
    \begin{aligned}
        \pi^{(t+1)} (a_k|h_k) &\propto \exp ( f_k( \theta_k^{(t+1)}, h_k, a_k )) \\&\approx \exp ( f_k(\theta_k^{(t)}, h_k,a_k) + \eta_k^{(t)} \widehat{Q}_k^{(t)}(h_k,a_k) )\\
        & = \pi^{(t)} (a_k|h_k) \exp (\eta_k^{(t)} \widehat{Q}_k^{(t)}(h_k,a_k)).
    \end{aligned}
\end{equation*}
This update step is analogous to the one-step natural policy gradient (NPG) update under the soft-max parameterization in tabular MDPs \citep{Kakade2001NPG, agarwal2021theory}.


Solving the optimization problem in \eqref{eq: optimization step 1} can be computationally challenging for general function classes of $\mathcal{F}_k:=\{ f_k(\theta_k, \cdot,\cdot): \theta_k\in \Theta_k\}$. Moreover, computing $\widehat{Q}_k ^{(t)}$ over the entire space $\mathcal{H}_k\times \mathcal{A}_k$ is difficult due to the lack of an analytical expression for $Q_{k, \widetilde{M}} ^{\pi^{(t)}} (h_k,a_k)$ (line 5). To address this, we employ a Monte Carlo approximation, estimating $\widehat{Q}_k ^{(t)} (h_k,a_k)$   at a finite set of state-action pairs.  This leads to an empirical approximation of \eqref{eq: optimization step 1}: 
\begin{equation}
\label{eq: optimization step 2}
\min_{\theta\in \Theta_k} \frac{1}{m} \sum_{i=1}^{m} \left( f_k(\theta, h_k^{(i)}, a_k^{(i)}) - f_k(\theta_k^{(t)}, h_k^{(i)}, a_k^{(i)}) - \eta_k^{(t)} \widehat{Q}_k^{(t)}(h_k^{(i)}, a_k^{(i)}) \right)^2,
\end{equation}
where $\{(h_k^{(i)}, a_k^{(i)})\} _{i=1}^m$ are sampled uniformly from $\mathcal{H}_k \times \mathcal{A}_k$. 


To make the least-squares regression problem in \eqref{eq: optimization step 2}  computationally tractable, we adopt a linear sieve approach~\citep{chen2007large} to parameterize the function classes $\{\mathcal{F}_k\}_{k=1}^K$. At each stage $k$, we denote the state and action histories as \(\bar{s}_k := (s_1, s_2, \dots, s_k)\) and \(\bar{a}_k := (a_1, a_2, \dots, a_k)\), respectively. Since actions are discrete, the total number of possible action histories is finite and bounded by $N_k^{(A)} = \prod_{j=1}^k |\mathcal{A}_j| < \infty$. Each function class $\mathcal{F}_k$ consists of linear functions  whose complexity increases with the sample size. Specifically, we represent each function $f_k\in \mathcal{F}_k$ as 
\begin{equation}
\label{eq: sieve function class}
f_k(\theta_k, h_k, a_k) = \theta_k[\bar{a}_k]^\top \Upsilon_{L_k}(\bar{s}_k),
\end{equation}
where $\theta_k\in \Theta_k:= \mathbb{R}^{L_k \times N_k^{(A)}}$ is the parameter matrix, and $\theta_k[\bar{a}_k]\in \mathbb{R}^{L_k}$ denotes the column of $\theta_k$ corresponding to action history $\bar{a}_k$. The feature vector $\Upsilon_{L_k}(\bar{s}_k)\in \mathbb{R}^{L_k}$ consists of $L_k$ basis functions evaluated on the state history $\bar{s}_k$.  We define  $\Upsilon_{L_k}(\cdot) : = \left( \upsilon_{L_k,1}(\cdot), \upsilon_{L_k,2}(\cdot), \dots, \upsilon_{L_k,L_k}(\cdot) \right)^\top$, where each $\upsilon_{L_k,\ell}(\cdot)$ is a basis function defined on the joint state space $\mathcal{S}_1 \times \mathcal{S}_2 \times \cdots \times \mathcal{S}_k$. We construct these basis functions using B-splines~\citep{Chen_2015}, a widely used and computationally efficient class for approximating smooth functions over compact subsets of Euclidean space. 


After incorporating the approximation technique, a practical implementation of the general procedure described in Algorithm~\ref{algo: general} is provided in Algorithm~\ref{algo: sieve regression} (see Supplement~\ref{subsec: summarized algorithm}).

\section{Theory}
\label{sec: theory}
In this section, we establish both statistical and algorithmic convergence guarantees for POLAR by analyzing the suboptimality of the final policy $\pi^{(T)}$  returned by Algorithm~\ref{algo: general}.  Specifically, the suboptimality of a policy $\pi$ is defined as the difference between the value of the optimal policy $\pi^\dagger$ and that of $\pi$, i.e., $\text{Subopt}(\pi; M^*) = V_{M^*} ^{\pi^\dagger} - V_{M^*} ^{\pi}$, 
where $M^*$ denotes the true underlying DTR model introduced in Section~\ref{sec: preliminary}. Our goal is to derive an upper bound on $\text{Subopt}(\pi^{(T)}; M^*)$, ensuring that the learned policy is near-optimal as the sample size and number of training iterations increase.

The suboptimality of any policy $\pi$ can be decomposed into three terms:
\begin{equation}
\label{eq: split into 3 terms}
\begin{split}
\text{Subopt}(\pi; M^*) &= V_{M^*}^{\pi^\dagger} - V_{M^*}^{\pi}\\
&= \underbrace{V_{M^*} ^{\pi^\dagger} -V_{\widetilde{M}} ^{\pi^\dagger}}_{(a)} + \underbrace{V_{\widetilde{M}}^{\pi^\dagger} - V_{\widetilde{M}}^{\pi}}_{(b)} + \underbrace{ V_{\widetilde{M}}^{\pi} -V_{M^*}^{\pi}}_{(c)}.
\end{split}
\end{equation}

Terms $(a)$ and $(c)$ together capture the model shift error, i.e., the discrepancy between the true and estimated 
transition models. Since the modified DTR model is defined as $\widetilde{M} = \left( \widehat{P}, \widetilde{r} \right)$ and the true model as $M^*=\left( P^*, r \right)$, bounding $(a)$ and $(c)$ reduces to quantifying the difference between the estimated transition dynamics $\widehat{P}$ and the ground truth dynamics $P^*$, leveraging the construction of $\widetilde{r}$ in Section \ref{subsec: method modified DTR}.
Term $(b)$ captures the policy optimization error under the modified model \( \widetilde{M} \).  
  We present our main theoretical result in Section \ref{subsec: theory overview} and establish theoretical guarantees for two specific transition model cases in Section \ref{sec:heoretical Results of Certain Policy Classes}. All the main proofs and required technical assumptions are deferred to Supplement \ref{proof:proof}. Auxiliary lemmas with their proofs are provided in Supplement \ref{sec:lemmasproof}. 

\subsection{Main theoretical results}
\label{subsec: theory overview}

Our analysis on the suboptimality \( \text{Subopt}(\pi^{(T)}) \) accounts for two sources of error: model shift error due to imperfect estimation of transition dynamics, and policy optimization error caused by function approximation and finite training iterations.

We begin by characterizing the function approximation error that arises at each iteration. At stage $k$, the policy update step (line 6 of Algorithm \ref{algo: general}) can be viewed as approximating the target function $f_k(\theta_k^{(t)},\cdot,\cdot) + \eta_k^{(t)} Q^{\pi^{(t)}}_{k,\widetilde{M}}(\cdot,\cdot)$ using a new function $f_k(\theta_k^{(t+1)},\cdot,\cdot) \in \mathcal{F}_k = \{ f_k(\theta_k, \cdot,\cdot): \theta_k \in \Theta_k \}$. We define the corresponding approximation error at iteration $t$ as:
\begin{equation}
\label{eq: e_t}
e^{(t)}_k(h_k,a_k) := f_k(\theta_k^{(t+1)}, h_k, a_k) - f_k(\theta_k^{(t)}, h_k, a_k) - \eta_k^{(t)} Q^{\pi^{(t)}}_{k,\widetilde{M}}(h_k,a_k).
\end{equation}

To control this error on convergence, we impose the following assumption that bounds it uniformly over all history-action pairs:

\begin{assumption}
\label{assumption: sup norm of e}
For any \( 1\le k\le K \), there exists a quantity \( G_k \) such that for any \( 0\le t\le T-1 \),
\begin{equation}
\label{eq: sup norm of e}
\sup_{(h_k, a_k)} |e^{(t)}_k(h_k,a_k)| = \mathcal{O}_P (G_k),
\end{equation}
where \( \mathcal{O}_P(G_k) \) denotes a bound in probability by some multiple of \( G_k \) as the sample size grows. This means that as \( G_k \to 0 \), the approximation error converges to zero in probability.
\end{assumption}

We are now ready to present our main result.

\begin{theorem}
\label{thm: combined main thm}
Suppose Assumptions \ref{assumption: uncertainty quantifier} and \ref{assumption: sup norm of e} hold. Let the hyperparameters in Algorithm~\ref{algo: general} be set as $\widetilde{c}_k = \sum_{j=k}^K \|\overline{r}_j\|_\infty$, and choose the stepsizes as $\eta_k^{(t)} = \frac{c_\eta}{\sqrt{T}}$. Then, with probability at least $1-K\delta$, Algorithm \ref{algo: general} returns a policy $\pi^{(T)}$ satisfying:

\begin{equation*}
\mathrm{Subopt}(\pi^{(T)}; M^*) \le  \underbrace{V^{\pi^\dagger} _{P^*, \underline{r}}}_{\text{Model Shift Error}} + \underbrace{\mathcal{O}\left( \frac{K}{\sqrt{T}} \right) + K \cdot \mathcal{O}_P \left( T^{1/2} G_k + T^{3/2} G_k^2 \right)}_{\text{Policy Optimization Error}}.
\end{equation*}
Here $V^{\pi^\dagger} _{P^*, \underline{r}}$ denotes the value function of $\pi^\dagger$ under the DTR model with transition dynamics $\{P^*_k\}_{k=1}^K$ and expected reward functions $\{\underline{r}_k\}_{k=1}^K$, where  $\underline{r}_k = b_k \Gamma_k$ and $b_k = 2\widetilde{c}_k + \sum_{j=k+1}^K \widetilde{c}_j \|\Gamma_j \|_\infty $. 
\end{theorem}

The first term captures the model shift error arising from the discrepancy between the estimated and true transition dynamics. It is defined as
\begin{equation}
    \label{def:Value func for a+c}
    V^{\pi^\dagger} _{P^*, \underline{r}} = \sum_{k=1}^K \mathbb{E}_{ (h_k, a_k)\sim d^{\pi^\dagger} _{P^*, k} } \left[ \underline{r}_k (h_k, a_k) \right] = \sum_{k=1}^K b_k \mathbb{E}_{ (h_k, a_k)\sim d^{\pi^\dagger} _{P^*, k} } \left[ \Gamma_k (h_k, a_k) \right].
\end{equation}
 Since the uncertainty quantifier \( \Gamma_k \) upper bounds the $L_1$ error between the estimated and true transition models,  \( \| \widehat{P}_k - P^*_k \|_1 \) (see Assumption \ref{assumption: uncertainty quantifier}), this term vanishes as the offline sample size increases, provided the transition model estimator is consistent.
 The second term quantifies the policy optimization error due to finite iterations and function approximation. The first component, \( \mathcal{O}(K/\sqrt{T}) \),  arises from the convergence rate of gradient-based updates under stochastic approximation, which is standard when using diminishing step sizes such as \( \eta_k^{(t)} = c_\eta/\sqrt{T} \). The second component depends on the function approximation error bound \( G_k \), which characterizes the
expressiveness of the function class used for policy learning.

Theorem \ref{thm: combined main thm} provides a general upper bound on the suboptimality of the final policy learned by Algorithm \ref{algo: general}. In our practical implementation (Algorithm~\ref{algo: sieve regression}), we apply this result using linear sieve functions as the approximating class. The next proposition quantifies the resulting policy optimization error.

\begin{proposition}
\label{prop: G_k}
Suppose \(\widetilde{M}=(\widehat{P},\widetilde{r})\) is the modified DTR model constructed in Algorithm~\ref{algo: sieve regression}. Under regularity conditions and specific hyperparameters detailed in Supplement \ref{subsub:pf prop}, Algorithm~\ref{algo: sieve regression} returns a policy \(\pi^{(T)}\) satisfying that:

\begin{equation}
\label{main:b bound}
V_{\widetilde{M}}^{\pi^\dagger} - V_{\widetilde{M}}^{\pi^{(T)}} =  \mathcal{O} \left( \frac{K}{\sqrt{T}} \right) + K \sqrt{T} \cdot \mathcal{O}_p\left( \left( \frac{\log m_k}{m_k} \right)^ {\alpha_1} \right),
\end{equation}
where \( m_k \) is the Monte Carlo sample size for approximating \( \widehat{Q}_k^{(t)}(h_k,a_k) \). \( \alpha_1>0 \) is a constant that depends on the smoothness properties of the function class, which will be specified in Supplement \ref{subsub:pf prop}. 
\end{proposition}
By choosing  \( m_k = T^{\alpha_2} \) for any \( \alpha_2 > \frac{1}{2\alpha_1} \), this bound also vanishes in probability as \( T \to \infty \).  Therefore, under the stated assumptions, both the model shift error and policy optimization error diminish with increasing sample size and training iterations, and the learned policy \( \pi^{(T)} \) by Algorithm~\ref{algo: sieve regression} achieves near-optimal performance as the offline sample size \( n \to \infty \) and  \( T \to \infty \).

\subsection{Theoretical results of certain transition dynamics}
\label{sec:heoretical Results of Certain Policy Classes}
We now analyze the convergence behavior of the practical implementation (Algorithm~\ref{algo: sieve regression}) under specific transition models, including the linear transition model and GP model. Our theoretical analysis employs a partial coverage condition, which relaxes the restrictive full support assumption commonly required in offline RL. 

\subsubsection{Linear transition model}
\label{sec:linear model}

While classical RL and DTR literature often assumes strong positivity, i.e., full support over the entire state-action space \citep{uehara2022pessimistic, chang2021mitigating, jin2021pessimism}, we instead adopt the more realistic partial coverage assumption that allows for limited exploration in the offline data.

\begin{assumption}[Partial Coverage for Linear Transition Model]
\label{assumption: partial coverage linear}
For any \(1\le k\le K\), we define the partial coverage coefficient as
\begin{equation}
\label{eq: partial coverage linear}
C_k := \sup_{u\in \mathbb{R}^{ d_k^{(s)}} \setminus \{0\} } \frac{u^T \mathbb{E}_{(h_k,a_k) \sim d^{\pi^\dagger}_{P^*,k} } [\phi_k(h_k,a_k) \phi_k(h_k,a_k)^T ] u }{u^T \mathbb{E}_{(h_k,a_k) \sim \rho_k} [\phi_k(h_k,a_k) \phi_k(h_k,a_k)^T ] u},
\end{equation}
where \( \phi_k \) is the feature mapping defined in \eqref{eq: linear transition model}, \(\rho_k\) denotes the offline marginal distribution over history-action pairs \((h_k, a_k) \in \mathcal{H}_k \times \mathcal{A}_k\), 
induced by the behavior policy \(\pi_b\) and the ground truth dynamics \(P^*\), and \( d_{P^*,k}^{\pi^\dagger} \) is the state-action marginal distribution under the optimal policy $\pi^\dagger$ and the ground truth dynamics $P^*$. We assume that  $C_k < \infty$ for all $k$. 
\end{assumption}

This assumption ensures that the offline dataset provides sufficient coverage of the state-action space to estimate the transition model reliably. Under this coverage condition,  we analyze how model shift influences policy suboptimality by leveraging the explicit form of the uncertainty quantifier \( \Gamma_k \) from \eqref{eq: KNR, Gamma}.  In particular, we establish an explicit bound on the model shift error \( V^{\pi^\dagger} _{P^*, \underline{r}} \)  through \eqref{def:Value func for a+c}. By combining this bound with our policy optimization guarantee from \eqref{main:b bound}, we obtain the following convergence result.

\begin{theorem}[Suboptimality of the Linear Transition Model]
\label{thm: suboptimality for linear}
Suppose Assumptions \ref{assumption: uncertainty quantifier} and \ref{assumption: partial coverage linear} hold, along with the technical conditions in Supplement~\ref{proof:linear}. Then, with probability at least 
\( 1-K\delta \), the policy \( \pi^{(T)} \) returned by Algorithm \ref{algo: sieve regression} satisfies:
\[
\mathrm{Subopt}(\pi^{(T)}; M^*) \le \mathcal{O}\!\left( \frac{ K^2\sqrt{\log n} }{ \sqrt{n}}+ \frac{K^{7/2}}{n} \right)
+ \mathcal{O}\!\left( \frac{K}{\sqrt{T}} \right)
+ K \sqrt{T} \cdot \mathcal{O}_p\!\left( \left( \frac{\log m_k}{m_k} \right)^{\alpha_1} \right),
\]
where $\alpha_1>0$ is a
constant that depends on the smoothness properties of the function class, as specified in \eqref{main:b bound}. 
\end{theorem}

The first term controls the estimation error under finite offline sample size $n$, and vanishes as \( n \to \infty \). The second term represents the policy optimization error, which vanishes as \( T \to \infty \). The last term accounts for function approximation error in estimating the $Q$-function, which is controlled by the Monte Carlo sample size \( m_k \). In this bound, we focus on the asymptotic dependence on the sample size $n$, horizon $K$, and iteration count $T$. Problem-dependent constants including the coverage coefficients \(\{C_k\}_{k=1}^K\), the uncertainty quantifier norms \(\|\Gamma_k\|_\infty\), and the reward bounds \(\|\overline{r}_k\|_\infty\) are absorbed into the Big-$O$ notation, as is standard in finite-sample analyses where such quantities are treated as fixed.
Next, we establish a minimax lower bound, showing the rate is near-optimal.

\begin{theorem}[Minimax Lower Bound of the Linear Transition Model]
\label{thm:lower-bound-hp}
There exist a universal constant $C_{\text{lower}}>0$
and a class $\mathcal M$ of DTR models with linear transition kernels such that the following holds.
Given $n$ i.i.d.\ offline trajectories of length $K$, $\mathcal D
  =
  \big\{(s_1^{(i)}, a_1^{(i)}, r_1^{(i)}, \dots, s_K^{(i)}, a_K^{(i)}, r_K^{(i)})
  \big\}_{i=1}^n,$
for the final policy $\widehat\pi$ of any learning algorithm $\mathcal A$,
\[
  \inf_{\mathcal A}
  \sup_{M\in\mathcal M}
  \mathbb P_M\!\left(
    V_M^{\pi^*} - V_M^{\widehat\pi} \;\ge\; C_{\text{lower}} \frac{K^2}{\sqrt{n}}
  \right)
  \;\ge\; \frac{1}{4}.
\]
\end{theorem}

Theorem~\ref{thm:lower-bound-hp} establishes a minimax lower bound for the same class of linear-transition DTR models, proving that any offline algorithm based on $n$ trajectories must incur a worst-case value error of at least $\Omega(K^{2}/\sqrt{n})$. When the
number of policy optimization iterations $T \to \infty$,  the dependence of our upper bound on $K$ and $n$ matches this lower bound up to a $\sqrt{\log n}$ factor. Therefore, Algorithm~\ref{algo: sieve regression} is near-minimax optimal for this model class. Based on the above theorems, we derive the following corollary, which quantifies the required offline sample size to guarantee an $\epsilon$-optimal policy in the limit as $T\to \infty$.

\begin{corollary}[Sample Complexity for Linear Transition Model] 
\label{Cor:Sample Complexity of Linear Model}
Under the conditions of Theorem \ref{thm: suboptimality for linear}, with probability at least \( 1-K\delta \), the suboptimality gap satisfies \( \mathrm{Subopt}(\pi^{(T)}; M^*) \leq \epsilon \) as \( T, m_k \to \infty \), provided that the offline dataset size satisfies  
$n = \widetilde{\Theta} \left(\frac{K^4}{\epsilon^2} \right).$
\end{corollary}
\cbend

\begin{remark}
\citet{li2024settling} and \citet{xiong2023nearlyminimaxoptimaloffline} analyze offline RL in finite-horizon MDPs with a \emph{single} transition kernel, under tabular and linear function approximation settings, respectively. They establish a sample complexity of $n = \widetilde{\Theta}\!\left(K^{3}/\varepsilon^{2}\right)$ for achieving an $\varepsilon$-optimal policy. In contrast, our sample complexity bound in Corollary~\ref{Cor:Sample Complexity of Linear Model} is larger by a factor of \( K \), i.e., \( n = \widetilde{\Theta} \left( \frac{K^4}{\epsilon^2} \right) \) and is minimax-optimal for the considered DTR linear transition model class. The difference in horizon dependence arises from the structural differences between the models.  Our DTR model involves $K$ distinct stage-specific linear transition kernels, whose dimensions grow linearly with the horizon $K$. Therefore, the accumulation of stage-wise model shift errors inherently yields a larger dependence on the horizon.
\end{remark}

\subsubsection{Gaussian Process Transition Model}
\label{sec:GP-model}
We now analyze the suboptimality of Algorithm \ref{algo: sieve regression} under a nonparametric GP transition model. The GP framework offers a flexible, nonparametric approach for capturing complex transition dynamics through kernel-based function approximation. For our theoretical analysis, we extend the partial coverage condition from the linear setting to the more general framework of Reproducing Kernel Hilbert Spaces (RKHS). This extension enables us to handle the richer function class while maintaining our coverage guarantees.

\begin{assumption}[Partial Coverage for GP Transition Model]
\label{assumption:partial-coverage-GP}
For each decision stage \(1 \leq k \leq K\), we define the partial coverage coefficient as
\begin{equation}
\label{eq:partial-coverage-GP}
C_k := \sup_{f \in \mathcal{H} \setminus \{0\}} \frac{\mathbb{E}_{(h_k, a_k) \sim d^{\pi^\dagger}_{P^*, k}}\left[f(h_k,a_k)\right]^2}{\mathbb{E}_{(h_k, a_k) \sim \rho_k}\left[f(h_k,a_k)\right]^2},
\end{equation}
where \(\mathcal{H}\) denotes the RKHS induced by the GP kernel, and \(f \in \mathcal{H}\) represents a real-valued function over history-action pairs. We assume that \(C_k < \infty\) for all \(1 \leq k \leq K\).
\end{assumption}


This assumption ensures that the offline dataset adequately spans the state-action regions required for accurate estimation of GP-based transition dynamics.  With this in place, we derive the suboptimality bound for the GP setting by leveraging the closed-form uncertainty quantifier 
\( \Gamma_k \) from  \eqref{eq: GP, eq 3} and substituting it into the general bound given in  \eqref{main:b bound}.

\begin{theorem}[Suboptimality of GP Model]
\label{thm:suboptimality-GP}
Suppose Assumptions \ref{assumption: uncertainty quantifier}, \ref{assumption: sup norm of e}, and \ref{assumption:partial-coverage-GP} hold, along with the technical conditions in Supplement \ref{proof sec: GP}. For the policy \(\pi^{(T)}\) returned by Algorithm \ref{algo: sieve regression}, the suboptimality satisfies, with probability at least \(1-K\delta\),
$$\mathrm{Subopt}(\pi^{(T)}; M^*) \leq \mathcal{O}\left(K^3 (\log n)^2 n^{\frac{5}{2\omega}-\frac{1}{2}}\right) + \mathcal{O}\left(\frac{K}{\sqrt{T}}\right) + K \sqrt{T} \cdot \mathcal{O}_p\left(\left(\frac{\log m_k}{m_k}\right)^{\alpha_1}\right), $$
where \(\omega\) is the eigenvalue decay rate of the GP kernel, and \(\alpha_1\) is the same constant as defined in Theorem \ref{thm: suboptimality for linear}. 
\end{theorem}

This result parallels Theorem \ref{thm: suboptimality for linear} for the linear transition model. As each term vanishes in the appropriate limit: $n\to \infty, T\to \infty, m_k \to \infty$, the overall bound converges to zero. Using this bound, we establish a corresponding sample complexity result for achieving $\epsilon$-optimality under the GP setting.

\begin{corollary}[Sample Complexity of GP Model]
\label{cor:sample-complexity-GP}
Under the conditions of Theorem \ref{thm:suboptimality-GP}, with probability at least \(1-K\delta\), the suboptimality gap satisfies \( \mathrm{Subopt}(\pi^{(T)}; M^*) \leq \epsilon \) as \( T \to \infty \), provided that the offline dataset size satisfies: $n = \widetilde{\Theta}\left(\left(\frac{K^3}{\epsilon}\right)^{\frac{2\omega}{\omega-5}}\right).$
\end{corollary}

This corollary illustrates the trade-off between sample complexity and the smoothness of the function class defined by the GP kernel. The parameter \( \omega \) reflects the eigenvalue decay rate: a larger \( \omega \) implies faster decay and smoother function spaces, thereby reducing the number of samples needed to achieve a given level of suboptimality. In contrast, smaller values of \( \omega \) indicate slower decay, reflecting more complex function classes and resulting in higher sample complexity. To our knowledge, this is the first suboptimality guarantee for model-based offline learning in DTRs under a nonparametric GP transition model. Although GP models typically require more data than parametric alternatives such as linear transition models, this cost reflects their greater flexibility in capturing complex, nonlinear dynamics.

\section{Simulation Study}
\label{sec: simulation}


\subsection{Simulation setup}
We evaluate the performance of POLAR by simulation studies, and investigate the impact of pessimism on optimal policy estimation. 
We consider a 3-stage DTR setting, where each state-action trajectory takes the form $(s_1,a_1,s_2,a_2,s_3,a_3,s_4)$. Assume that the state space is a two-dimensional continuous space, denoted as  $s_k = (s_k^1, s_k^2)$, $k=1,\dots, 4$. The action space is binary, with $a_k \in \{0,1\}$ for each stage. The initial state $s_1$ is generated uniformly from the unit square $[0,1]^2$. At each stage $k$, we define the true transition dynamics $P_k^*(\cdot \mid h_k,a_k)$ using a linear transition model:  $s_{k+1}= W_k^a \phi(s_k)+\varepsilon_k$, 
where $\phi(s_k) = (1,s_{k}^1, s_k^2)$, $W_k^{a}$ is a $2\times 3$ matrix for each $k$ and $a$, and $\varepsilon_k = (\varepsilon_k^{1}, \varepsilon_k^{2})$ is the noise term. The values of $W_k^a$ and the specification of $\varepsilon_k$ are provided in Supplement \ref{sec:simulation_setup}. We assume that the reward is generated only at the final stage ($r_1 = r_2 = 0$), and it is defined as $
r=3.8\left[\left(\cos \left(-s_3^1 \pi\right)+2 \cos \left(s_3^2 \pi\right)+s_4^1+2 s_4^2\right)\left(1+a_3\right)-1.37\right]$.

Based on the ground truth transition models and reward function, we first approximate the optimal policy $\pi^*$ using Dynamic Programming \citep{bellman1957dynamic}.  The implementation details are provided in Supplement \ref{sec: Dynamic Programming}.  We then construct the behavior policy $\pi_b$, which takes the optimal action $a^*$ given by $\pi^*$ with probability $p$, and takes the non-optimal action $1-a^*$ with probability $1-p$. Offline data are generated using 
$\pi_b$, consisting of $n$ trajectories. We explore various behavior policies by varying the probability parameter $p \in \{0.95, 0.75, 0.55\}$ and the offline sample size $n \in \{50, 200, 1000, 5000, 20000\}$. 
For each combination of $(p, n)$, we conduct $N=100$ repeated  simulations. 

For each simulated dataset, we apply the proposed POLAR to estimate the optimal policy, resulting in the learned policy $\widehat{\pi}$. The value function of $\widehat{\pi}$ is then evaluated under the simulated true transition $P^*$ and reward function $r$ using Monte Carlo sampling.  As the hyperparameter  $c$ controls the degree of uncertainty penalization, we implement POLAR using a range of values of $c\in \{0,5,10,50,100\}$, where $c=0$ represents  no pessimism.  In our simulation study, $T=20$ steps proves sufficient for convergence. For comparison, we apply five alternative methods to each simulated dataset: (1) Q-learning for DTRs (DTR-\(Q\)), which estimates the state-action value function recursively \citep{murphy05a}, with implementation details provided in Supplement \ref{sec:standard_q_learning}; (2) a regression-based method for estimating optimal DTRs using the R package DTRreg \citep{DTRreg}, which fits a sequence of outcome regression models to identify optimal actions at each stage; (3) deep reinforcement learning with Double Q-learning (DDQN) \citep{van2016deep}, implemented in Python using deep neural networks to approximate the Q-function; (4) model-based offline policy optimization (MOPO) \citep{yu2020mopo}, which performs policy optimization using a learned dynamics model with an uncertainty-based penalty to discourage exploitation of out-of-distribution regions; and
(5) model-based imitation learning from offline data (MILO) \citep{chang2021mitigating}, which frames offline policy learning as a distribution-regularized imitation problem, constraining the policy to the behavior policy. \cbend

\normalcolor


\subsection{Simulation results}





We first examine the performance of POLAR by examining  the policy value derived from the estimated optimal policy. When the behavior policy $\pi_b$ is defined with $p=0.75$ and the offline data size is $n=200$, 
Figure~\ref{figure: simu-1} presents the policy value versus the number of iterations for POLAR under different values of the hyperparameter $c$ and all comparison methods, averaged over 100 repeated simulations.  As the iteration increases, the policy values of POLAR with different choices of $c$ all converge. When no uncertainty penalization is applied ($c=0$), POLAR underperforms, likely due to overfitting in state-action regions poorly covered by the offline data.   On the other extreme, with a large penalization parameter  ($c=100$), the policy becomes overly conservative, prioritizing uncertainty minimization at the cost of reward. The best performance is achieved when \(c=50\), where POLAR effectively balances between reward maximization and uncertainty penalization. This moderate level of pessimism helps the algorithm mitigate the effects of distributional shift while still learning a near-optimal policy.
Compared to the baseline methods, POLAR consistently achieves higher policy values than DTR-\(Q\), DDQN, DTRreg, MOPO, and MILO, \cbend  highlighting the benefits of incorporating a principled pessimism strategy and leveraging full historical information in offline RL.




\begin{figure}[ht]
    \centering
    \includegraphics[width=0.8\linewidth]{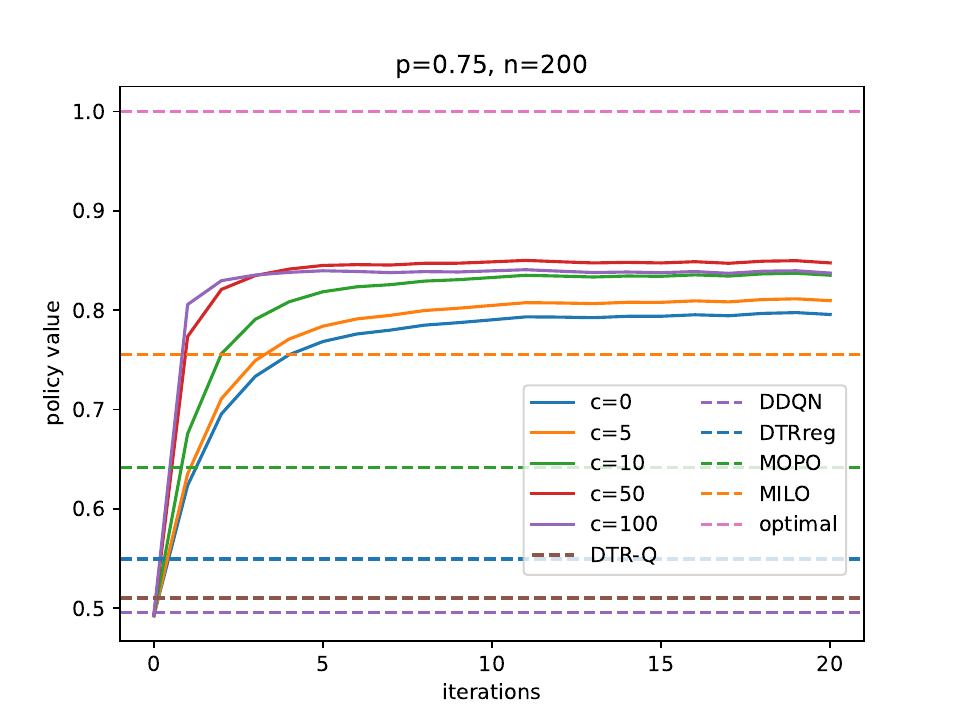}
    \caption{Policy value vs. iterations for different values of \(c\) ($p$=0.75, $n$=200).}. 
    \label{figure: simu-1}
\end{figure}


\begin{figure}
    \centering
    \includegraphics[width=1.0\linewidth]{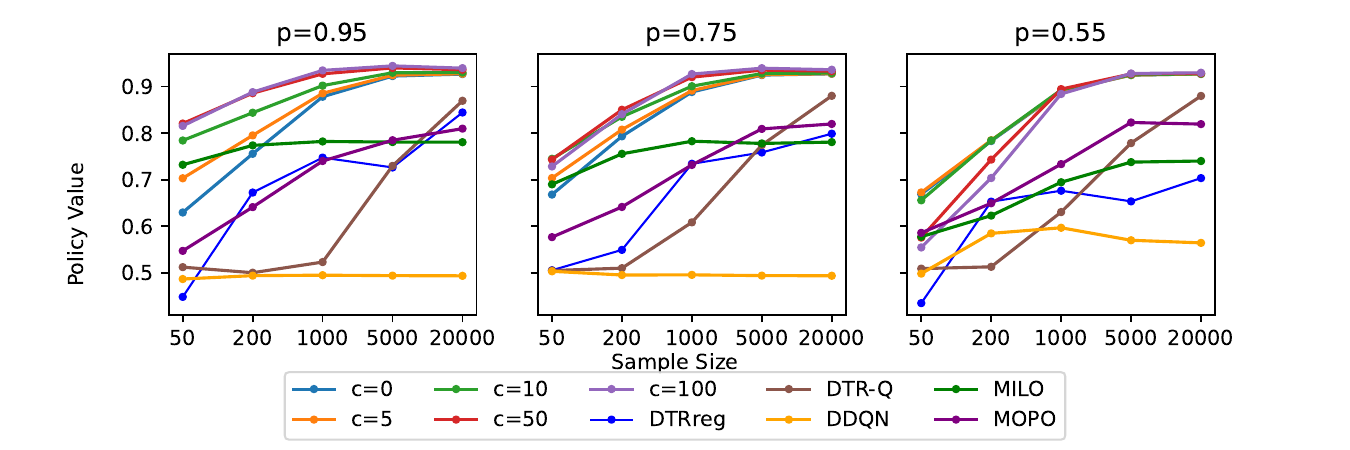}
    \caption{Policy value averaged over 100 repeated simulations, for different $n,p,c$. }
    \label{figure: simu-2}
\end{figure}

We then systematically analyze the performance of POLAR with respect to three factors: the parameter $p$, which determines the behavior policy $\pi_b$; pessimism hyperparameter $c$, which controls the degree of pessimism in the reward function; and the offline sample size $n$. For each combination of $(p,n,c)$, the average policy value across 100 repeated simulations is reported in Figure~\ref{figure: simu-2}.  When \(p = 0.95\) (Figure~\ref{figure: simu-2}a), the behavior policy selects the optimal action with high probability, leading to poor state-action coverage and substantial distributional shift. In this setting, incorporating pessimism is important: 
POLAR with a non-zero penalty (\(c > 0\)) outperforms the unpenalized case (\(c = 0\)) across all sample sizes \(n\). In contrast, when \(p = 0.55\), the behavior policy explores the actions more evenly, providing better state-action coverage.  In this setting, uncertainty penalization becomes less critical. Small penalties (e.g., \(c \in \{0, 5, 10\}\)) achieve the best performance, while larger penalties (e.g., \(c  \in \{50, 100\}\)) lead to overly conservative policies and reduced performance, especially at small sample sizes. \cbend

Next, we observe that across all values of $p$ and $c$, POLAR’s performance improves as the offline sample size $n$ increases. When $n$ is sufficiently large, the policy values for different choices of $c$ begin to converge, indicating that the influence of the uncertainty penalty diminishes. This is because the uncertainty quantifier \(\Gamma\) converges to zero at a rate  of \(\mathcal{O}_p(\sqrt{\log n / n})\), thereby reducing the impact of uncertainty penalization in large datasets. In contrast, when $n$ is small, POLAR significantly outperforms the baseline methods (DTR-\(Q\), DTRreg, DDQN, MOPO and MILO), demonstrating its robustness in low-data regimes.

We further compare POLAR to the baseline methods in more detail. DDQN, which relies on the Markov assumption and does not incorporate historical information, shows relatively flat performance across all settings. This highlights its limitations in the DTR setup, where past trajectories contain critical information about patient history and treatment response that is essential for optimal decision-making.
 DTRreg approximates the Q-function using linear models, which limits its flexibility and results in consistently inferior performance. DTR-\(Q\), on the other hand, improves steadily as $n$ increases and achieves performance comparable to POLAR when $n=100,000$ (not shown in the figure). This highlights a key distinction between model-based and model-free approaches: POLAR, by leveraging an explicit transition model and pessimism regularization, achieves near-optimal performance with far fewer samples, demonstrating the sample efficiency advantage of model-based approaches. 
While the policy values of MOPO and MILO increase with sample size $n$, they remain consistently lower than those estimated by POLAR. 
This result is expected. MILO focuses on imitating the behavior policy, which can constrain performance improvement. MOPO, designed for stationary Markov decision processes, relies on a single time-homogeneous transition model, introducing model mismatch in the multi-stage DTR setting.  In contrast, POLAR is explicitly designed for finite-horizon, history-dependent DTRs, modeling stage-specific transition kernels ${P_k(h_{k+1}\mid h_k,a_k)}_{k=1}^K$.

These results demonstrate the importance of incorporating pessimism in DTR, particularly under limited trajectory coverage. When the behavior policy is skewed toward optimal actions (i.e., large $p$), offline data tend to underrepresent suboptimal regions, increasing the risk of overfitting. In such cases, a principled pessimism strategy helps mitigate distributional shift by regularizing uncertain estimates in poorly supported areas of the state-action space. This setting is especially relevant in medical applications, where observational data often reflect clinician-driven decisions that are near-optimal but selectively explored. 

Finally, we evaluate POLAR's robustness through a sensitivity analysis. While the simulated data follow a linear transition model, we intentionally apply POLAR using a GP to model transitions, creating a model mismatch. Under this misspecification, POLAR continues to outperform all alternative methods, confirming its effectiveness. More details are provided in Supplement \ref{sec:sensitivity_analysis}. \cbend

\section{Real Data Analysis}
\label{sec: real data}
To demonstrate the practical utility of POLAR, we applied it to the MIMIC-III (Medical Information Mart for Intensive Care) database \citep{johnson2016mimic}, which contains comprehensive clinical data from patients admitted to intensive care units (ICUs). This dataset includes longitudinally collected medical records, capturing both physiological states and treatment decisions. Its temporal structure makes it well-suited for evaluating sequential decision-making methods. For this analysis, we focused on a cohort of sepsis patients as defined in \citet{Komorowski2018}. Sepsis, a life-threatening condition caused by a dysregulated immune response to infection, is a leading cause of mortality in ICUs. Managing sepsis involves complex and dynamic treatment decisions, such as adjusting intravenous (IV) fluid volumes and vasopressor (VP) dosages to stabilize hemodynamics and support organ function. Traditional RL methods based on the MDP assumption have been used to estimate optimal treatment strategies for sepsis patients in MIMIC-III \citep{Komorowski2018,raghu2017deep}. However, in clinical practice, treatment decisions often depend on a patient's historical information such as  prior treatments, physiological responses, and trends in laboratory results. Failing to account for this historical context can result in suboptimal policies. In contrast, the DTR framework explicitly incorporates the full history of a patient's trajectory into the decision-making process. 

In this analysis, we modeled a simplified 3-stage DTR problem by focusing on the first three decision points for each patient, resulting in \(n = 16,265\) patient trajectories. The action space was defined as a \(5 \times 5\) grid, created by discretizing both IV fluid dosages and maximum VP dosages into five levels, as described in \citet{raghu2017deep}. For the reward function, we used the Sequential Organ Failure Assessment (SOFA) score, a widely recognized metric for evaluating organ dysfunction in sepsis patients \citep{Lambden2019}. Higher SOFA scores indicate greater organ dysfunction; thus, we defined the reward as the negative value of the SOFA score. The state space at each decision point was represented by a 43-dimensional feature vector, constructed from patient-specific information including demographics (e.g., age, gender), lab values (e.g., lactate, creatinine), vital signs (e.g., heart rate, blood pressure), and intake/output events \citep{raghu2017deep,zhou2023optimizing}. The goal is to learn a policy that maximizes the expected cumulative reward, effectively minimizing organ dysfunction over time.

We applied the proposed POLAR method, using GP regression to estimate the transition dynamics \(\widehat{P}_k(\cdot\mid h_k, a_k)\) and the corresponding uncertainty quantifiers \(\Gamma_k(h_k, a_k)\) \citep{chang2021mitigating}, across various values of the pessimism parameter  \(c \in \{0, 5, 10, 20, 50\}\). 
For comparison, we considered five baseline methods used in the simulation studies: DDQN, DTR-\(Q\), DTRreg, MOPO and MILO. However, since DTRreg package cannot handle multiple action dimensions, we report results only for DDQN, DTR-\(Q\), MOPO, and MILO. To assess model performance, we randomly split the dataset into 50\% training and 50\% testing sets.  The training set was used to train the algorithms, while the testing set was employed to evaluate the performance of the learned policies through off-policy evaluation (OPE) with importance sampling \citep{uehara2022review}. This procedure was repeated \(N=200\) times with different random splits, and the final results are reported as averages across these repetitions.



\begin{figure}[ht]
    \centering
    \includegraphics[width=0.6\linewidth]{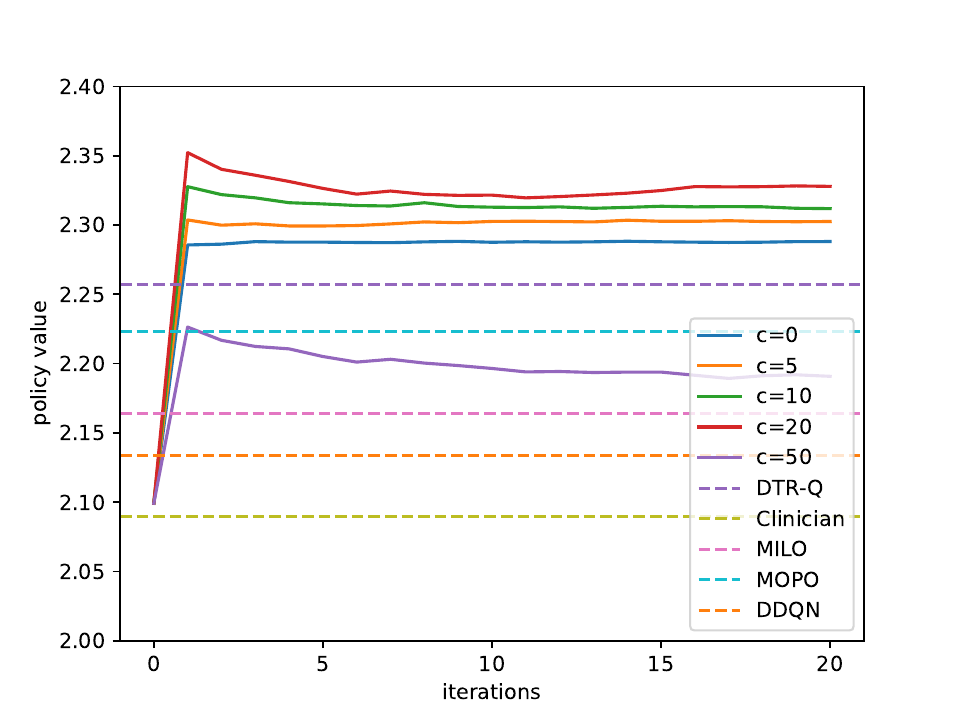}
    \caption{
    Policy value versus number of iterations for POLAR with different values of $c$,
    compared to DDQN, DTR-\(Q\), MOPO, and MILO.
    All policy values are evaluated using OPE.
    }

    \label{figure: realdata-1}
\end{figure}


\begin{figure}[ht]
    \centering
    \includegraphics[width=0.8\linewidth]{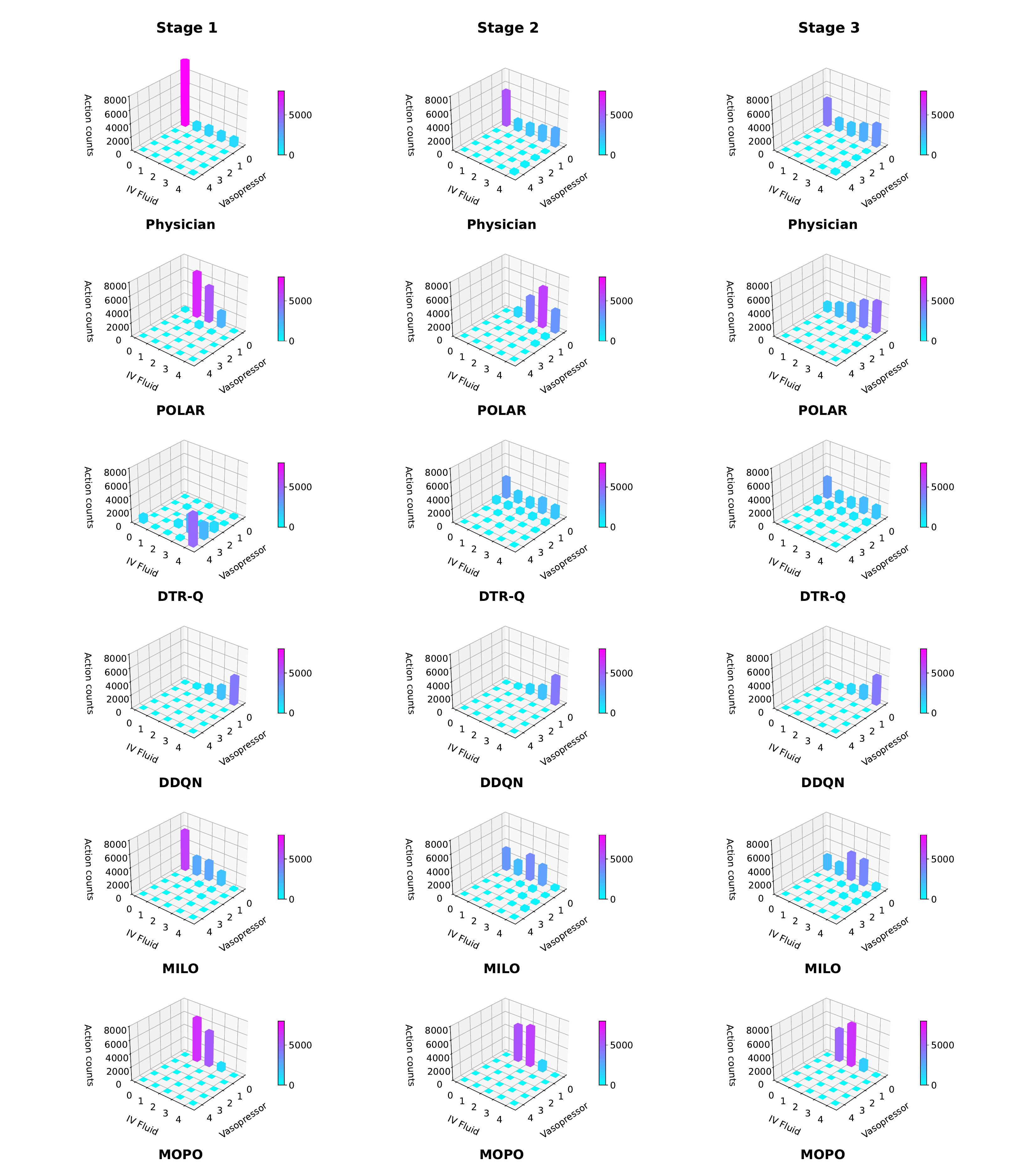}
    \caption{The action distributions generated by the physician policy, POLAR (\(c=20\)), DTR-\(Q\), DDQN, MILO and MOPO across three decision stages.}
        \label{figure: realdata-2}
\end{figure}

Figure~\ref{figure: realdata-1} presents the policy value versus iteration \(t\) for the proposed POLAR with different values of \(c\),
alongside four methods, DDQN, DTR-\(Q\), MOPO, and MILO. Overall, POLAR achieves higher policy values than all baselines, demonstrating its effectiveness in real-world offline settings. Introducing a moderate pessimistic penalty generally improves performance: configurations with $c=5, 10$, and 20 outperform the no-penalty case ($c=0$) throughout the iterations.

Among these, $c=20$ attains the highest policy value, suggesting that incorporating a conservative bias can be beneficial in managing uncertainty. However, excessively large penalties (e.g., $c=50$) lead to degraded performance, demonstrating the importance of balancing uncertainty penalization with reward maximization in offline policy learning.
  In practice, we recommend running POLAR with a candidate set of $c$ values and selecting the optimal one via cross-validated grid search.

To further compare the methods, we analyzed the action distributions generated by the physician policy (i.e., the behavior policy),  POLAR (\(c=20\)), DDQN, DTR-\(Q\), MOPO, and MILO \cbend across three decision stages (Figure~\ref{figure: realdata-2}). The physician policy demonstrated distinct treatment patterns: lower vasopressor dosages with varying IV fluid levels in Stage 1, followed by a shift toward higher vasopressor dosages in later stages. These changes reflect clinicians’ dynamic responses to evolving patient conditions and the importance of history-dependent decision-making in sepsis management.

POLAR captured key aspects of these temporal treatment shifts, adapting its action distributions across stages to reflect increasingly aggressive vasopressor use while maintaining diversity in fluid levels. By incorporating uncertainty-aware regularization and leveraging historical patient information, POLAR avoided overcommitting to uncertain actions and effectively balanced exploitation and exploration. Its distributions demonstrated clinically interpretable structure and were aligned with domain expectations for stage-specific intervention patterns.  
In contrast, DTR‐Q, DDQN, MILO, and MOPO exhibit limited adaptability across decision stages. DTR‐Q fails to learn effective physician actions in Stage 1, producing a highly dispersed action distribution and recommending excessively large vasopressor doses that would be dangerously aggressive in practice. DDQN displays nearly identical action distributions across all three stages and fails to capture the gradual vasopressor escalation pattern, highlighting its weakness in short-horizon sequential decision problems. MILO produces action distributions closely resembling the clinician's, but this imitation-driven design yields lower OPE values than POLAR. MOPO similarly generates nearly invariant action distributions across decision stages, indicating limited adaptability to the stage-specific dynamics of clinical DTRs. These findings demonstrate POLAR's ability to model the dynamic, context‐aware nature of clinical decision‐making. By leveraging historical information, sampling unexplored state–action combinations, and penalizing uncertain actions, POLAR learns interpretable, adaptive treatment strategies that more closely align with real‐world clinical expertise.

\section{Conclusion and Discussion}
\label{sec: discussion}

We developed POLAR, a novel model-based policy optimization method for learning DTRs from offline data. Our approach first estimates the transition dynamics using historical trajectories and incorporates the pessimism principle through uncertainty penalization to account for distributional shift and limited state-action coverage. The algorithm is practically implemented using an actor-critic-style iterative optimization procedure and yields a final policy with provable suboptimality guarantees.

We evaluated POLAR through   numerical experiments on both simulated environments and real-world clinical data from the MIMIC-III database. The results demonstrate that POLAR outperforms state-of-the-art baseline methods, particularly in settings with limited data or suboptimal behavior policies. In the real-data analysis, POLAR not only achieved higher policy values via OPE but also produced treatment policies that closely mirrored physician decision patterns, particularly in capturing dynamic, stage-specific adaptations. By explicitly incorporating uncertainty and leveraging patient history, POLAR offers a robust framework to support and improve clinical decision-making in sepsis care and other complex, sequential treatment settings where patient trajectories evolve over time.

The proposed method has several limitations and possible extensions. First, the computational cost of POLAR increases with the number of decision stages, as the algorithm requires evaluating Q-functions over sampled state-action pairs at each stage. This may pose challenges for high-frequency decision-making tasks or long-horizon treatment regimes. Second, the current framework assumes that decision points are aligned across patients, which often does not hold in real-world clinical settings where treatment decisions occur at irregular or patient-specific intervals.  One potential direction is to extend POLAR to handle asynchronous decision points by modeling decisions as events in continuous time, using tools such as point process models or irregular time-series methods.  We leave both challenges for future research.

\section*{Acknowledgment}
This work is supported in part by National Institute of Health grant R01MH128085 and R01AI197147. 

\bibliographystyle{apalike}
\bibliography{Bibliography-MM-MC}

\newpage
\appendix
\renewcommand{\thealgocf}{S\arabic{algocf}}
\setcounter{algocf}{0}
\renewcommand{\theassumption}{S\arabic{assumption}} 
\setcounter{assumption}{0} 
\renewcommand{\thelemma}{S\arabic{lemma}} 
\setcounter{lemma}{0} 
\renewcommand{\thetheorem}{S\arabic{theorem}} 
\setcounter{theorem}{0} 
\renewcommand{\theproposition}{S\arabic{proposition}} 
\setcounter{proposition}{0} 
\renewcommand{\thecorollary}{S\arabic{corollary}} 
\setcounter{corollary}{0} 

\setcounter{figure}{0}
\renewcommand{\thefigure}{S\arabic{figure}}

\renewcommand{\thedefinition}{S\arabic{definition}} 
\setcounter{definition}{0} 

\renewcommand{\theequation}{S\arabic{equation}}
\setcounter{equation}{0} 

\setcounter{page}{1} 

\begin{center}
    {\LARGE \textbf{Supplementary Materials}}
\end{center}

\section{Detailed Analysis of Uncertainty Quantifiers}
\label{sec: discussions on UQ}

In this section, we rigorously demonstrate how the constructed uncertainty quantifiers for both transition models satisfy Assumption \ref{assumption: uncertainty quantifier}. 

\subsection{Linear Transition Model}
\label{subsec: kernelized}

\subsubsection{Model Specification and Estimation}
\label{subsec: KNR model}

We model the transition dynamics at each stage $k$ with the following linear structure:
\[
s_{k+1} = W_k \phi_k(h_k,a_k) + \varepsilon_k, \quad \varepsilon_k \sim F_k,
\]
where $\phi_k: \mathcal{H}_k \times \mathcal{A}_k \to \mathbb{R}^{d_k^{(s)}}$ is a known bounded feature mapping defined by $\phi_k(h_k,a_k)\; :=\;\Upsilon_{L_k}(\bar{s}_k)\otimes e_{\bar{a}_k},$ with \(e_{\bar{a}_k}\in\mathbb{R}^{N_k^{(A)}}\) denoting the one‐hot encoding of the action‐history \(\bar{a}_k\), and \(N_k^{(A)}=\prod_{j=1}^k|\mathcal A_j|\).   We define  $\Upsilon_{L_k}(\cdot) : = \left( \upsilon_{L_k,1}(\cdot), \upsilon_{L_k,2}(\cdot), \dots, \upsilon_{L_k,L_k}(\cdot) \right)^\top$, where each $\upsilon_{L_k,\ell}(\cdot)$ is a basis function defined on the joint state space $\mathcal{S}_1 \times \mathcal{S}_2 \times \cdots \times \mathcal{S}_k$.
$F_k$ is a known distribution on $\mathbb{R}^{d_{k+1}^{(s)}}$, $d_{k}^{(s)}:=L_kN_k^{(A)}$ is the dimension of the $k$-th state space $\mathcal{S}_k$, and $W_k\in \mathbb{R}^{d_{k+1}^{(s)} \times d_k^{(s)}}$ is a stage-specific weight matrix. 
In this model, stochasticity in the state transition is captured via additive noise. Therefore, we impose the following mild conditions to ensure that the generated state always remains in the support of the next-stage state space, i.e., \( W_k \phi_k(h_k,a_k) + \varepsilon_k \in \mathcal{S}_{k+1} \) almost surely for all \( (h_k,a_k) \). 

\begin{SAssumption}
\label{assumption: noise distribution in KNR}
(i) The additive noise terms \( \{\varepsilon_k\}_{k=1}^{K} \) are independent across both stages and trajectories.\\
(ii) Each noise distribution \( F_k \) is zero-mean and bounded in the \( \ell_2 \)-ball of radius \( \sigma \):
\begin{equation*}
   \mathbb{E}[\varepsilon_k] = 0, \quad \mathbb{P}(\| \varepsilon_k \|_2 \leq \sigma) = 1. 
\end{equation*}
(iii) For all \( h_k \in \mathcal{H}_k \), \( a_k \in \mathcal{A}_k \), the conditional density \( P_k(\cdot|h_k,a_k) \) is Lipschitz continuous on \( \mathcal{S}_{k+1} \), with a universal Lipschitz constant \( C_L \) that does not depend on \( k \) or the history state-action pair.
\end{SAssumption}

This model class is closely related to the kernelized nonlinear regulator (KNR), which has been extensively studied in the context of dynamic systems and reinforcement learning~\citep{kakade2020information, chang2021mitigating, uehara2022pessimistic}. However, most prior works assume Gaussian noise with unbounded support, which leads to unbounded state spaces. In contrast, we adopt a variant of KNR with compact state spaces. Our model class is more closely aligned with the settings studied by~\cite{curi2020efficient,mania2020active}, where the noise terms are not necessarily Gaussian. 

To implement this model in the offline setting, we next describe how to estimate the transition dynamics from data.
For any $1 \leq k \leq K$, recall that the offline dataset $\mathcal{D}$ contains $n$ observations at stage $k$, given by $\{(h_k^{(i)}, a_k^{(i)}, s_{k+1}^{(i)} )\}_{i=1}^{n}$. According to the model specification, each transition is formulated as $s_{k+1}^{(i)} = W_k \phi_k(h_k^{(i)}, a_k^{(i)}) + \varepsilon_k^{(i)},$
where $\varepsilon_k^{(i)}$ are i.i.d. noise terms. Since the feature mapping $\phi_k(\cdot, \cdot)$ is known, we can estimate the weight matrix $W_k$ through ridge regression:
\begin{equation}
\label{eq: KNR, W hat}
\widehat{W}_k = \left( \sum_{i=1}^{n} s_{k+1}^{(i)} \phi_k( h_k^{(i)}, a_k^{(i)})^\top \right) \left( \sum_{i=1}^{n} \phi_k( h_k^{(i)}, a_k^{(i)}) \phi_k( h_k^{(i)}, a_k^{(i)})^\top + \lambda I \right)^{-1},
\end{equation}
for some regularization parameter $\lambda \geq 0$. Given the estimate $\widehat{W}_k$, we define the estimated transition model \( \widehat{P}_k \) as $s_{k+1} = \widehat{W}_k \phi_k(h_k, a_k) + \varepsilon_k.$

\subsubsection{Theoretical Properties}

In this subsection, we establish the theoretical properties of uncertainty quantifiers $\Gamma_k$. 

The uncertainty quantifiers $\Gamma_k$ in the linear transition model is \citep{chang2021mitigating}:
\begin{equation*}
\Gamma_k(h_k, a_k) = \min \left\{2, 2 C_2\sqrt{\phi_k^T (h_k,a_k) \Lambda_{k,n}^{-1} \phi_k(h_k,a_k)} \right\},
\end{equation*}
where $C_2=B_{d_{k+1}^{(s)}}\sigma C_L \beta_{k,n}$. $\sigma$ and $C_L$ are from Assumption \ref{assumption: noise distribution in KNR}, $B_{d_{k+1}^{(s)}}$ denotes the Lebesgue measure of a unit ball in $d_{k+1}^{(s)}$-dimension Euclidean space, and
\begin{equation}
\label{eq: KNR, eq 1}
\Lambda_{k,n} = \sum_{i=1}^{n} \phi_k( h_k^{(i)}, a_k^{(i)}) \phi_k( h_k^{(i)}, a_k^{(i)})^T + \lambda I,
\end{equation}
\begin{equation}
\label{eq: KNR, eq 2}
\beta_{k,n} = \sqrt{\lambda} \|W^*_k\|_2 + \sqrt{8\sigma^2 d_{k+1}^{(s)} \log(5) + 8\sigma^2 \log \left( \frac{\det (\Lambda_{k,n}) ^{1/2} / \det (\lambda I) ^{1/2}}{\delta} \right)},
\end{equation}
where \( \delta \in (0,1) \) is the confidence level.

The following proposition guarantees the desired property of our uncertainty quantifier.

\begin{proposition}
\label{thm: uncertainty quantifier KNR}
Suppose Assumption 
\ref{assumption: noise distribution in KNR} holds. Let $\widehat{W}_k$ and $\Gamma_k$ follow the form in (\ref{eq: KNR, W hat}) and (\ref{eq: KNR, Gamma}) respectively, then with probability at least $1-\delta$, 
\begin{equation*}
\left\|\widehat{P}_k(\cdot|h_k,a_k) - P^*_k(\cdot |h_k,a_k)\right\|_1\le \Gamma_k(h_k,a_k),\ \forall (h_k,a_k)\in \mathcal{H}_k \times \mathcal{A} _k .
\end{equation*}
\end{proposition}
By this proposition, $\{\Gamma_k\}_{k= 1}^K$ satisfy the uncertainty quantifiers condition (see Assumption \ref{assumption: uncertainty quantifier}). The proof of Proposition~\ref{thm: uncertainty quantifier KNR} relies on the following two technical lemmas.

\begin{lemma}
\label{lemma: bounded rv is subgaussian}
Let \( Y \in \mathbb{R}^d \) be a random vector satisfying:
\begin{itemize}
    \item[(i)] \( \mathbb{E}[Y] = 0 \),
    \item[(ii)] \( \|Y\|_2 \le \sigma \) almost surely, for some \( \sigma > 0 \).
\end{itemize}
Then \( Y \) is \(\sigma\)-sub-Gaussian.
\end{lemma}

\begin{lemma}
\label{lemma: W-W* in KNR}
Suppose $n$ tuples $\{ (h_i,a_i,s_i) \}_{i=1}^n$ are observed, and $(h_i, a_i, s_i)$ satisfy the transition $s_i= W^* \phi(h_i,a_i)+ \varepsilon_i$, $\varepsilon_1, ..., \varepsilon_n \stackrel{iid}{\sim} F$. Assume $F$ is $\sigma$-sub-Gaussian. Let $\widehat{W} = (\sum_{i=1}^n s_i \phi_i^T) (\sum_{i=1}^n \phi_i \phi_i^T + \lambda I)^{-1}$ for some $\lambda>0$, where $\phi_i$ denotes $\phi(h_i,a_i)$. Let $\Lambda_n = \sum_{i=1}^n \phi_i \phi_i^T + \lambda I$, then with probability at least $1-\delta$:
\begin{equation*}
\left\| (\widehat{W}-W^*) \Lambda_n^{1/2} \right\|_2 \le \sqrt{\lambda} \|W^*\|_2 + \sqrt{8\sigma^2 d_s \log(5) + 8\sigma^2 \log \left( \frac{\det (\Lambda_n) ^{1/2} / \det (\lambda I) ^{1/2}}{\delta} \right)}.
\end{equation*}
\end{lemma}

The proofs of Lemma \ref{lemma: bounded rv is subgaussian} and Lemma \ref{lemma: W-W* in KNR} can be found in {Section} \ref{pf:lemma prop 1}.

\begin{proof}[Proof of Proposition~\ref{thm: uncertainty quantifier KNR}]

By Assumption~\ref{assumption: noise distribution in KNR} (ii), the following bound holds for any \( s_{k+1} \in \mathcal{S}_{k+1} \):
\[
\left| \widehat{P}_k(s_{k+1} \mid h_k, a_k) - P^*_k(s_{k+1} \mid h_k, a_k) \right|
\le C_L \cdot \| \widehat{W} \phi_k(h_k, a_k) - W^* \phi_k(h_k, a_k) \|_2,
\]
where \( C_L \) is the Lipschitz constant of the conditional density.

Moreover, by Assumption~\ref{assumption: noise distribution in KNR} (i), both distributions have compact support and vanish outside the union of two \( \ell_2 \)-balls: $B(\widehat{W} \phi_k(h_k, a_k), \sigma) \cup B(W^* \phi_k(h_k, a_k), \sigma),$
each of radius \( \sigma \) in \( \mathbb{R}^{d_{k+1}^{(s)}} \). Let \( B_d \) denote the Lebesgue volume of the unit ball in \( \mathbb{R}^d \). Then the \( \ell_1 \) distance between the two distributions can be bounded as:
\[
\| \widehat{P}_k(\cdot \mid h_k, a_k) - P^*_k(\cdot \mid h_k, a_k) \|_1 
\le 2 B_{d_{k+1}^{(s)}} \sigma C_L \cdot \| \widehat{W} \phi_k(h_k, a_k) - W^* \phi_k(h_k, a_k) \|_2.
\]

We now aim to bound the error between the means of the estimated and true conditional distributions, namely, $\| \widehat{W} \phi_k(h_k, a_k) - W^* \phi_k(h_k, a_k) \|_2.$ To do so, we insert the identity matrix \( I_n=\Lambda_n^{1/2} \Lambda_n^{-1/2} \) and apply the sub-multiplicative property of norms:
\begin{equation*}
\begin{split}
\| \widehat{W} \phi_k(h_k, a_k) - W^* \phi_k(h_k, a_k) \|_2 
&= \| (\widehat{W} - W^*) \Lambda_n^{1/2} \Lambda_n^{-1/2} \phi_k(h_k, a_k) \|_2 \\
&\le \| (\widehat{W} - W^*) \Lambda_n^{1/2} \|_2 \cdot \| \Lambda_n^{-1/2} \phi_k(h_k, a_k) \|_2.
\end{split}
\end{equation*}

Therefore, we have:
\begin{equation}
\label{ineq: transition auxiliary}
\widehat{P}_k(\cdot \mid h_k, a_k) - P^*_k(\cdot \mid h_k, a_k) \|_1 
\le 2 B_{d_{k+1}^{(s)}} \sigma C_L \| (\widehat{W} - W^*) \Lambda_n^{1/2} \|_2 \cdot \| \Lambda_n^{-1/2} \phi_k(h_k, a_k) \|_2.
\end{equation}

Next, we turn to bound the term $\| (\widehat{W} - W^*) \Lambda_n^{1/2} \|_2$. We begin by invoking Lemma \ref{lemma: bounded rv is subgaussian}. Under Assumption~\ref{assumption: noise distribution in KNR} (i), the additive noise terms are bounded, and thus each \( \varepsilon_k^{(i)} \) is \( \sigma \)-sub-Gaussian. Consequently, the conditions of Lemma~\ref{lemma: W-W* in KNR} are satisfied, yielding the following high-probability bound:
\begin{equation}
\label{ineq: w sub-gaussian bound}
\left\| (\widehat{W} - W^*) \Lambda_n^{1/2} \right\|_2 
\le \sqrt{\lambda} \|W^*\|_2 + \sqrt{8\sigma^2 d_{k+1}^{(s)} \log(5) + 8\sigma^2 \log \left( \frac{\det (\Lambda_n)^{1/2}}{\det (\lambda I)^{1/2} \delta} \right)} 
= \beta_{k,n}.
\end{equation}
Applying inequality~\eqref{ineq: w sub-gaussian bound} to \eqref{ineq: transition auxiliary}, we obtain with probability at least \( 1 - \delta \) that:
\begin{equation*}
\| \widehat{P}_k(\cdot \mid h_k, a_k) - P^*_k(\cdot \mid h_k, a_k) \|_1 \le 2 B_{d_{k+1}^{(s)}} \sigma C_L \cdot \beta_{k,n} 
\cdot \sqrt{ \phi_k(h_k, a_k)^\top \Lambda_n^{-1} \phi_k(h_k, a_k) },
\end{equation*}
where $\beta_{k,n}=\sqrt{\lambda} \|W^*\|_2 + \sqrt{8\sigma^2 d_{k+1}^{(s)} \log(5) + 8\sigma^2 \log \left( \frac{\det (\Lambda_n)^{1/2}}{\det (\lambda I)^{1/2} \delta} \right)}.$
\end{proof}

\subsection{Gaussian Process}
\label{subsec: GP}
\subsubsection{Model Specification and Estimation}
Unlike the example in Section \ref{subsec: kernelized}, we no longer assume a parametric form for $P_k$. Instead, we consider a nonparametric setting. We assume the transition dynamic $P_{k}(s_{k+1}\mid h_k,a_k)$ is structured as: $s_{k+1} = g_k(h_k, a_k) + \varepsilon_k$, where $\varepsilon_k \sim \mathcal{N}(0, \sigma^2 I)$. 

To estimate this model, we have access to an offline dataset \( \mathcal{D} \) consisting of \( n \) i.i.d. observations at each stage \( k \), given by \( \{(h_k^{(i)}, a_k^{(i)}, s_{k+1}^{(i)}) \}_{i=1}^n \). Our goal is to construct both an estimator \( \widehat{P}_k \) for the transition model and an uncertainty quantifier \( \Gamma_k \). For notational simplicity, we omit the subscript \( k \) when it is clear from the context. Moreover, we simplify notation by letting \( x := (h_k, a_k) \) to represent the input pair when \( k \) is fixed. Under these conventions, we can reformulate the problem using Gaussian process (GP) regression.

Under the GP prior, the posterior distribution over the unknown function \( g \) remains a Gaussian process with posterior mean function \( \widehat{g}(\cdot) \) and covariance function \( \widehat{h}(\cdot, \cdot) \), given by:
\begin{equation}
\label{eq: GP, g hat}
\widehat{g}(x) = (y_1, \dots, y_n)(H_{n,n} + \sigma^2 I)^{-1} H_n(x),
\end{equation}
\begin{equation}
\label{eq: GP, eq 1}
\widehat{h}(x, x') = h(x, x') - H_n(x)^\top (H_{n,n} + \sigma^2 I)^{-1} H_n(x'),
\end{equation}
where \( (y_1, \dots, y_n) \) is a \( d_{k+1}^{(s)} \times n \) matrix of observed outputs, \( H_{n,n} \in \mathbb{R}^{n \times n} \) is the kernel matrix with entries \( h(x_i, x_j) \), and \( H_n: \mathcal{X} \rightarrow \mathbb{R}^n \) is defined as \( H_n(x) = (h(x, x_1), \dots, h(x, x_n))^\top \).

\subsubsection{Theoretical Properties}

To quantify the uncertainty in the GP-based transition model, we define the uncertainty quantifier \( \Gamma \) in the following form: $\Gamma_k(x) = \frac{\beta_{k,n}}{\sigma} \sqrt{\widehat{h}(x,x)},$
where \( \widehat{h} \) is the posterior covariance function defined in equation~(\ref{eq: GP, eq 1}), and \( \beta_{k,n} \) is a constant given by:
\begin{equation}
\label{eq: GP, eq 4}
\beta_{k,n} = \sqrt{d_{k+1}^{(s)} \left(2 + 150 \log^3\left( \frac{d_{k+1}^{(s)} n}{\delta} \right) \log\left( \det\left( I_n + \frac{1}{\sigma^2} H_{n,n} \right) \right) \right)},
\end{equation}
where \( H_{n,n} \) is the kernel matrix defined earlier and \( \delta \in (0,1) \) is the confidence level.

The next proposition guarantees that the quantifier \( \Gamma(x) \) satisfies our Assumption \ref{assumption: uncertainty quantifier}.

\begin{proposition}
\label{thm: uncertainty quantifier GP}
Let \( \widehat{g}(x) \) and \( \Gamma(x) \) be defined as in equations~(\ref{eq: GP, g hat}) and~(\ref{eq: GP, eq 3}), respectively. Then, with probability at least \( 1 - \delta \),
\begin{equation*}
\left\| \widehat{P}(\cdot \mid x) - P^*(\cdot \mid x) \right\|_1 \le \Gamma(x), \quad \forall x \in \mathcal{X}.
\end{equation*}
\end{proposition}

\begin{proof}[Proof of Proposition~\ref{thm: uncertainty quantifier GP}]
See Lemma 14 of \cite{chang2021mitigating} for a detailed proof of this result.
\end{proof}

\section{Practical Algorithm for Section \ref{subsec: method overview}}
\label{subsec: summarized algorithm}

\begin{algorithm}[ht]
\caption{Practical algorithm}
\label{algo: sieve regression}
\KwIn{Offline dataset $\mathcal{D}$; stepsizes $\{\eta_k^{(t)}\}$ and Monte Carlo sample sizes $\{m_k^{(t)}\}$ at $k$-th stage and $t$-th iteration; hyperparameters $\{\widetilde{c}_k\}$}

Construct estimated transition dynamics $\widehat{P}$, estimated reward functions $\{\widehat{r}_k\}$ and uncertainty quantifiers $\{\Gamma\}_k$ from the offline data $\mathcal{D}$ \\
Construct the modified DTR model $\widetilde{M} = (\widehat{P}, \widetilde{r})$, with $\widetilde{r}_k = \widehat{r}_k - \widetilde{c}_k \Gamma_k$ \\
Initialize $\theta^{(0)}=(\theta^{(0)}_1,... \theta^{(0)}_K) = (\bf{0}, \bf{0}, ..., \bf{0})$ (i.e. $\bf{0}$ is a zero matrix $\in \mathbb{R}^{|L_k| \times N_k ^{(A)}} $)\\
$\pi^{(0)}\gets \pi(\theta^{(0)}), t=0 $\\
\For{$t=0,...,T-1$}{
\For{$k=1,...,K$}{
\For{$\overline{a}_k ~ \mathrm{in} ~ \mathcal{A}_1 \times \mathcal{A}_2 \times \cdots \times \mathcal{A}_k$}{
Sample $m_k$ i.i.d. samples $\{\overline{s}_k^{(1)}, ..., \overline{s}_k^{(m_k)} \}$ from $\mathrm{Unif}(\mathcal{S}_1 \times \mathcal{S}_2 \times \cdots \times \mathcal{S}_k)$ \\
Evaluate $Q_{k, \widetilde{M}} ^{\pi^{(t)}}$ on $m_k$ points: $\widehat{Q}_k^{(t)} (\overline{s} _k^{(i)}, \overline{a}_k) \gets Q_{k, \widetilde{M}} ^{\pi^{(t)}} (\overline{s} _k^{(i)}, \overline{a}_k ) $ for all $i$\\
Solve the linear regression problem that projects $\widehat{Q}_k^{(t)}(\cdot)$ onto $\Upsilon_{L_k}(\cdot)$: $\widehat{\theta} _k^{(t)}( \overline{a}_k ) \gets \arg \min_{\theta\in \mathbb{R}^{L_k}} \sum_{i=1}^{m_k} \left( \theta ^T \Upsilon_{L_k}(\overline{s}_k) - \widehat{Q}_k^{(t)} (\overline{s} _k^{(i)}, \overline{a}_k) \right)^2$ \\
}
Concatenate the $N_k^{(A)}$ columns: $\widehat{\theta} _k^{(t)} \gets \left( \widehat{\theta} _k^{(t)}\Bigl( \overline{a} _k^{(1)} \Bigr),~...~,\widehat{\theta} _k^{(t)}\Bigl( \overline{a} _k^{(N_k^{(A)})} \Bigr)  \right)$ \\
Update $\theta_k$: $\theta_k^{(t+1)} \gets \theta_k^{(t)} + \eta_k^{(t)} \widehat{\theta}_k^{(t)}$
}
$\pi^{(t+1)}\gets \pi({\theta^{(t+1)}})$ where $\theta^{(t+1)} = (\theta_1^{(t+1)}, ..., \theta_K^{(t+1)})$
}
\KwOut{$\pi^{(T)}$}
\end{algorithm}

Under the setting of linear sieves \eqref{eq: sieve function class}, $f_k(\theta_k, \cdot, \cdot)$ belongs to the linear space spanned by $L_k$ basis functions.  Our theoretical analysis allows $L_k$ to increase with $m_k$, but we suppose $m_k$ and $L_k$ are both fixed when implementing Algorithm \ref{algo: sieve regression}.

\section{Proofs}
\label{proof:proof}
This section provides complete proofs of the theoretical results in Section \ref{sec: theory}. Section~\ref{proof: main thm} covers the regret decomposition under uncertainty quantification. Section~\ref{proof:linear} establishes suboptimality bounds for the linear model using kernel-based estimation. Section~\ref{proof sec: GP} presents the analysis for the Gaussian Process-based model, leveraging posterior contraction for regret control. 

\subsection{Proofs of Section \ref{subsec: theory overview}}
\label{proof: main thm}

The suboptimality of our learned policy can be decomposed into model shift error and policy optimization error. We begin with Theorem \ref{thm: final regret real M}
that upper bounds the model shift error through the uncertainty quantifiers. Next, we bound the policy optimization error in $\widetilde{M}$ in Theorem \ref{thm: M^tilde + general e}, which depends on the function class. Combining Theorem \ref{thm: final regret real M} and Theorem \ref{thm: M^tilde + general e}, we can prove our main Theorem \ref{thm: combined main thm} by carefully setting the stepsize.
In particular, we adopt linear sieves as our function class and establish the corresponding Theorem \ref{thm: M^tilde + sieve reg} to bound the policy optimization error.

\subsubsection{Proofs of Theorem \ref{thm: combined main thm}}

\begin{STheorem}
\label{thm: final regret real M}
Suppose Assumption~\ref{assumption: uncertainty quantifier} holds. Define $\widetilde{c}_k := \sum_{j=k}^K \|\overline{r}_j\|_\infty$, and let $\widetilde{r}_k := \widehat{r}_k - \widetilde{c}_k \Gamma_k$. Consider the modified model $\widetilde{M} = (\widehat{P}, \widetilde{r})$. Then, for any policy $\pi$, the following inequality holds with probability at least $1-K\delta$:
\begin{equation}
\label{eq:regret bound in thm S1}
   \left( V_{M^*}^{\pi^\dagger} - V_{\widetilde{M}}^{\pi^\dagger} \right) + \left( V_{\widetilde{M}}^{\pi} - V_{M^*}^{\pi} \right) \le V_{P^*, \underline{r}}^{\pi^\dagger}, 
\end{equation}
where $\underline{r}_k = b_k \Gamma_k$ and $b_k = 2 \widetilde{c}_k + \sum_{j=k+1}^K \widetilde{c}_j \|\Gamma_j\|_\infty.$ Here, $V_{P^*, \underline{r}}^{\pi^\dagger}$ denotes the value function of $\pi^\dagger$ under the DTR model with transition dynamics $\{P^*_k\}_{k=1}^K$ and stage-wise expected rewards $\{\underline{r}_k\}_{k=1}^K$.
\end{STheorem}

The proof of Theorem \ref{thm: final regret real M} relies on the following supporting lemmas. The proofs of these lemmas are deferred to {Section} \ref{pf:lemma them S1}.

\begin{lemma}[Simulation Lemma]
\label{lemma: simulation}
Consider two DTR models $M = (P, r)$ and $\widehat{M} = (\widehat{P}, r)$ that share the same reward function $r$. For any transition dynamics $P$, $\widehat{P}$, any policy $\pi$, any index $k$, and any pair $(h_k, a_k)$, the following identities hold:

\begin{itemize}
    \item[(i)] The difference in action-value functions satisfies:
    \begin{align*}
    &\quad Q_k^\pi(h_k, a_k) - \widehat{Q}_k^\pi(h_k, a_k) \\
    &= \sum_{t=k}^{K-1} \mathbb{E}_{(h_t, a_t) \sim d_{t,P}^\pi(\cdot,\cdot \mid H_k = h_k, A_k = a_k)} \left[
    \mathbb{E}_{h' \sim P_t(\cdot \mid h_t, a_t)} \left[ \widehat{V}_{t+1}^\pi(h') \right]
    - \mathbb{E}_{h' \sim \widehat{P}_t(\cdot \mid h_t, a_t)} \left[ \widehat{V}_{t+1}^\pi(h') \right]
    \right].
    \end{align*}
    
    \item[(ii)] The difference in state-value functions satisfies:

\[
\begin{aligned}
V_{k}^{\pi}(h_k) - \widehat{V}_{k}^{\pi}(h_k)
&= \sum_{t=k}^{K-1} \mathbb{E}_{(h_t, a_t) \sim d_{t, P}^{\pi}(\cdot, \cdot \mid H_k = h_k)} \Biggl[
\mathbb{E}_{h' \sim P_t(\cdot \mid h_t, a_t)} \bigl[ \widehat{V}_{t+1}^{\pi}(h') \bigr] \\
&\qquad -
\mathbb{E}_{h' \sim \widehat{P}_t(\cdot \mid h_t, a_t)} \bigl[ \widehat{V}_{t+1}^{\pi}(h') \bigr]
\Biggr].
\end{aligned}
\]
\cbend
\end{itemize}
\end{lemma}

\begin{lemma}[Reward Difference Lemma]
\label{lemma: r - rhat}
Under Assumption~\ref{assumption: uncertainty quantifier} and the reward setup in Section~\ref{sec: preliminary}, the following inequality holds with probability at least $1 - \delta$ (under $\mathbb{P}_{\mathcal{D}}$):
\[
\left| r_k(h_k, a_k) - \widehat{r}_k(h_k, a_k) \right| \le \| \overline{r}_k \|_\infty \cdot \Gamma_k(h_k, a_k), \quad \forall k,\ \forall (h_k, a_k) \in \mathcal{H}_k \times \mathcal{A}_k.
\]
\end{lemma}

\begin{proof}[Proof of Theorem \ref{thm: final regret real M}]

We begin by analyzing the difference between the values of policy $\pi$ under the modified model $\widetilde{M}$ and the true model $M^*$. Specifically, we decompose:
\[
\begin{aligned}
V_{\widetilde{M}}^{\pi} - V_{M^*}^{\pi}
&= V_{\widehat{P}, \widetilde{r}}^{\pi} - V_{P^*, r}^{\pi} \\
&= V_{\widehat{P}, r}^{\pi} - V_{P^*, r}^{\pi} - V_{\widehat{P}, \{\widetilde{c}_k \Gamma_k\}}^{\pi} + V_{\widehat{P}, \widehat{r} - r}^{\pi}.
\end{aligned}
\]
This expression follows from the definition $\widetilde{r}_k = \widehat{r}_k - \widetilde{c}_k \Gamma_k$ and the linearity of the value function in the reward.

To bound the term $V_{\widehat{P}, r}^{\pi} - V_{P^*, r}^{\pi}$, we invoke the \textit{simulation lemma} (Lemma~\ref{lemma: simulation}), which quantifies the impact of discrepancies in transition dynamics. Applying the lemma yields:
\begin{equation}
\label{eq: proof thm final regret M^*, eq 3}
    V_{\widehat{P},r}^{\pi} - V_{P^*,r}^{\pi}= \sum_{k=1}^{K-1} \mathbb{E}_{(h_k,a_k)\sim d_{k,\widehat{P}}^{\pi} (\cdot,\cdot)} \left[ \mathbb{E}_{h'\sim \widehat{P}_k(\cdot|h_k,a_k)}\left[ V_{k+1, P^*, r}^{\pi} (h') \right] - \mathbb{E}_{h'\sim P^*_k(\cdot|h_k,a_k)}\left[ V_{k+1, P^*, r}^{\pi} (h') \right] \right].
\end{equation}

This difference can be bounded via total variation and the value function norm:
$$
\begin{aligned}
   &\quad \mathbb{E}_{h' \sim \widehat{P}_k(\cdot \mid h_k, a_k)} \left[V_{k+1, P^*, r}^{\pi}(h')\right] -
\mathbb{E}_{h' \sim P_k^*(\cdot \mid h_k, a_k)} \left[V_{k+1, P^*, r}^{\pi}(h') \right]\\
&\le \|V_{k+1, P^*, r}^{\pi}(\cdot)\|_\infty \|\widehat{P}_k(\cdot \mid h_k, a_k) - P_k^*(\cdot \mid h_k, a_k)\|_1. 
\end{aligned}
$$
To control this, we further bound the supremum norm of the value function as:
\[
\|V_{k, r}(\cdot)\|_\infty \le \sum_{j=k}^{K} \|\overline{r}_j\|_\infty = \widetilde{c}_k.
\]

Substituting back, we obtain:
\begin{equation}
    \label{eq: proof thm final regret M^*, eq 4}
    \mathbb{E}_{h'\sim \widehat{P}_k(\cdot|h_k,a_k)}\left[ V_{k+1, P^*, r}^{\pi} (h') \right] - \mathbb{E}_{h'\sim P^*_k(\cdot|h_k,a_k)}\left[ V_{k+1, P^*, r}^{\pi} (h') \right] \leq \widetilde{c}_{k+1} \Gamma_k(h_k,a_k).
\end{equation}

Combining (\ref{eq: proof thm final regret M^*, eq 3}) and (\ref{eq: proof thm final regret M^*, eq 4}), and noting that $d_{k,\widehat{P}}^{\pi}$ is a non-negative distribution over histories and actions, we conclude:
\begin{equation*}
\begin{split}
V_{\widehat{P},r}^{\pi} - V_{P^*,r}^{\pi} &\le \sum_{k=1}^{K-1} \mathbb{E}_{(h_k,a_k)\sim d_{k,\widehat{P}}^{\pi} (\cdot,\cdot)} \left[ \widetilde{c}_{k+1} \Gamma_k(h_k,a_k) \right]\\
&\stackrel{(i)}{=} \sum_{k=1}^{K} \mathbb{E}_{(h_k,a_k)\sim d_{k,\widehat{P}}^{\pi} (\cdot,\cdot)} \left[ \widetilde{c}_{k+1} \Gamma_k(h_k,a_k) \right]\\
&\stackrel{(ii)}{=} V_{\widehat{P},\{\widetilde{c}_{k+1} \Gamma_k\}}^{\pi},
\end{split}
\end{equation*}
where $(i)$ follows by setting $\widetilde{c}_{K+1} = 0$ for completeness, and $(ii)$ uses the definition of the value function with structured reward $\{\widetilde{c}_{k+1} \Gamma_k\}$.

Substituting this into the earlier decomposition and regrouping the reward terms via linearity, we obtain:
\[
V_{\widetilde{M}}^{\pi} - V_{M^*}^{\pi} \le V_{\widehat{P}, \{\widetilde{c}_{k+1} \Gamma_k - \widetilde{c}_k \Gamma_k + \widehat{r}_k - r_k\}}^{\pi}.
\]

To further simplify this, observe that $\widetilde{c}_k - \widetilde{c}_{k+1} = \|\overline{r}_k\|_\infty$, and Lemma~\ref{lemma: r - rhat} implies that $\widehat{r}_k - r_k \le \|\overline{r}_k\|_\infty \Gamma_k$. Thus, $\widetilde{c}_{k+1} \Gamma_k - \widetilde{c}_k \Gamma_k + \widehat{r}_k - r_k \le 0,$
which leads to the inequality:
\begin{equation}
\label{eq:value negative}
V_{\widetilde{M}}^{\pi} - V_{M^*}^{\pi} \le 0.
\end{equation}
We now proceed to analyze the {first term} in \eqref{eq:regret bound in thm S1}, namely the gap between the optimal policy's performance under the true model $M^*$ and under the modified model $\widetilde{M}$: $V_{M^*}^{\pi^\dagger} - V_{\widetilde{M}}^{\pi^\dagger} = V_{P^*, r}^{\pi^\dagger} - V_{\widehat{P}, \widetilde{r}}^{\pi^\dagger}.$ To handle this difference, we split it into two components:

\begin{equation}
\label{eq: proof thm final regret M^*, eq 7}
V_{M^*} ^{\pi^\dagger} -V_{\widetilde{M}} ^{\pi^\dagger} = V_{P^*, r} ^{\pi^\dagger} - V_{\widehat{P}, \widetilde{r}} ^{\pi^\dagger} = V_{P^*, r-\widetilde{r}} ^{\pi^\dagger} + \left(V_{P^*, \widetilde{r}} ^{\pi^\dagger} - V_{\widehat{P}, \widetilde{r}} ^{\pi^\dagger}\right).
\end{equation}

To control the latter one, we once again invoke the \emph{simulation lemma} (Lemma~\ref{lemma: simulation}), this time with fixed reward functions $\widetilde{r}$.
\begin{equation}
\label{eq: proof thm final regret M^*, eq 8}
 V_{P^*,\widetilde{r}} ^{\pi^\dagger} - V_{\widehat{P},\widetilde{r}} ^{\pi^\dagger} = \sum_{k=1}^{K-1} \mathbb{E}_{(h_k,a_k)\sim d_{k,P^*}^{\pi^\dagger} (\cdot,\cdot)} \left[ \mathbb{E}_{h'\sim P^*_k(\cdot|h_k,a_k)}\left[ V_{k+1, \widehat{P}, \widetilde{r}}^{\pi^\dagger} (h') \right] - \mathbb{E}_{h'\sim \widehat{P}_k(\cdot|h_k,a_k)}\left[ V_{k+1, \widehat{P}, \widetilde{r}}^{\pi^\dagger} (h') \right] \right].
\end{equation}

A key step is bounding the norm of the value function under $\widetilde{r}$. Recall that for any $k$, $\|V_{k, \widetilde{r}}(\cdot)\|_\infty \le \sum_{j=k}^K \|\widetilde{r}_j\|_\infty,$
then by Assumption \ref{assumption: uncertainty quantifier}, we get
\begin{equation*}
\begin{split}
&\quad \mathbb{E}_{h'\sim P^*_k(\cdot|h_k,a_k)}\left[ V_{k+1, \widehat{P}, \widetilde{r}}^{\pi^\dagger} (h') \right] - \mathbb{E}_{h'\sim \widehat{P}_k(\cdot|h_k,a_k)}\left[ V_{k+1, \widehat{P}, \widetilde{r}}^{\pi^\dagger} (h') \right]\\
&\le \| V_{k+1, \widehat{P}, \widetilde{r}}^{\pi^\dagger} (\cdot) \|_\infty \left\| \widehat{P}_k (\cdot|h_k,a_k) - P^*_k(\cdot |h_k,a_k) \right\|_1 \\
&\le \Bigl( \sum_{j=k+1}^K \|\widetilde{r}_j\|_\infty \Bigr) \Gamma_k(h_k,a_k)\\
&= \widetilde{q}_{k+1} \Gamma_k(h_k,a_k),
\end{split}
\end{equation*}
where $\widetilde{q}_{k} = \sum_{j=k}^K \|\widetilde{r}_j\|_\infty$. Then from (\ref{eq: proof thm final regret M^*, eq 8}), we further get:
\begin{equation*}
\begin{split}
V_{P^*,\widetilde{r}} ^{\pi^\dagger} - V_{\widehat{P},\widetilde{r}} ^{\pi^\dagger} &\le \sum_{k=1}^{K-1} \mathbb{E}_{(h_k,a_k)\sim d_{k,P^*}^{\pi^\dagger} (\cdot,\cdot)} \left[ \widetilde{q}_{k+1} \Gamma_k(h_k,a_k) \right]\\
&= \sum_{k=1}^{K} \mathbb{E}_{(h_k,a_k)\sim d_{k,P^*}^{\pi^\dagger} (\cdot,\cdot)} \left[ \widetilde{q}_{k+1} \Gamma_k(h_k,a_k) \right],
\end{split}
\end{equation*}
where the last step is because we define $\widetilde{q}_{K+1}=0$ for completeness.
Hence, the following important relation holds:
\begin{equation}
\label{eq: proof thm final regret M^*, eq 9}
V_{P^*,\widetilde{r}} ^{\pi^\dagger} - V_{\widehat{P},\widetilde{r}} ^{\pi^\dagger} \leq \sum_{k=1}^{K} \mathbb{E}_{(h_k,a_k)\sim d_{k,P^*}^{\pi^\dagger} (\cdot,\cdot)} \left[ \widetilde{q}_{k+1} \Gamma_k(h_k,a_k) \right]
\end{equation}

Note that the right hand side of (\ref{eq: proof thm final regret M^*, eq 9}) can be rewritten as a value function under $\pi^\dagger$ and $P^*$, with expected value function $\{ \widetilde{q}_{k+1} \Gamma_k\}_{k=1}^K$. Then we can plug (\ref{eq: proof thm final regret M^*, eq 9}) into (\ref{eq: proof thm final regret M^*, eq 7}), and combine them into one value function using the linearity of expected reward function:
\begin{equation}
\label{eq: proof thm final regret M^*, eq 10}
V_{M^*} ^{\pi^\dagger} -V_{\widetilde{M}} ^{\pi^\dagger} \le V^{\pi^\dagger} _{P^*, \{ r_k - \widetilde{r}_k + \widetilde{q}_{k+1} \Gamma_k \} }.
\end{equation}
To further simplify the reward difference term, we apply Lemma~\ref{lemma: r - rhat}, which bounds the deviation between the true reward and its estimated version:
\begin{equation*}
r_k - \widetilde{r}_k = r_k - \widehat{r}_k + \widetilde{c}_k \Gamma_k \le (\|\overline{r}_k\|_\infty + \widetilde{c}_k) \Gamma_k.
\end{equation*}
Adding the term $\widetilde{q}_{k+1} \Gamma_k$ to both sides yields:
\begin{equation*}
\begin{split}
r_k - \widetilde{r}_k + \widetilde{q}_{k+1} \Gamma_k &\le \Bigl( \|\overline{r}_k\|_\infty + \widetilde{c}_k + \widetilde{q}_{k+1} \Bigr) \Gamma_k \\
&\stackrel{(i)}{=} \Bigl( \|\overline{r}_k\|_\infty + \widetilde{c}_k + \sum_{j=k}^K \|\widetilde{r}_j\|_\infty \Bigr) \Gamma_k \\
&\stackrel{(ii)}{=} \Bigl( \|\overline{r}_k\|_\infty + \widetilde{c}_k + \sum_{j=k+1}^K \|\widehat{r}_j - \widetilde{c}_j \Gamma_j \|_\infty \Bigr) \Gamma_k \\
&\stackrel{(iii)}{\leq} \Bigl( \|\overline{r}_k\|_\infty + \widetilde{c}_k + \sum_{j=k+1}^K \|\overline{r}_j \|_\infty + \sum_{j=k+1}^K \widetilde{c}_j \| \Gamma_j \|_\infty \Bigr) \Gamma_k. 
\end{split}
\end{equation*}
$(i), (ii)$ come from the definition $\widetilde{q}_{k} = \sum_{j=k}^K \|\widetilde{r}_j\|_\infty 
 \text{ and }\|\widetilde{r}_j\|_\infty= \|\widehat{r}_j - \widetilde{c}_j \Gamma_j \|_\infty$. $(iii)$ is due to triangle inequality. 
 
Since $\widetilde{c}_k = \sum_{j=k}^K \|\overline{r}_j\|_\infty$, we can simplify the bound as:
\begin{equation}
\label{eq: proof thm final regret M^*, eq 11}
 r_k - \widetilde{r}_k + \widetilde{q}_{k+1} \Gamma_k \leq   \Bigl( 2 \widetilde{c}_k + \sum_{j=k+1}^K \widetilde{c}_j \| \Gamma_j \|_\infty \Bigr) \Gamma_k
\end{equation}

Combining \eqref{eq: proof thm final regret M^*, eq 10} and \eqref{eq: proof thm final regret M^*, eq 11}, we have:
\begin{equation}
\label{eq: value DTR bound}
    V_{M^*} ^{\pi^\dagger} -V_{\widetilde{M}} ^{\pi^\dagger} \le V^{\pi^\dagger} _{P^*, \underline{r}},
\end{equation}
where we define $\underline{r}_k := \left( 2 \widetilde{c}_k + \sum_{j=k+1}^K \widetilde{c}_j \| \Gamma_j \|_\infty \right) \Gamma_k$.

Finally, we complete the proof by combining \eqref{eq:value negative} and \eqref{eq: value DTR bound}.

\end{proof}

\begin{STheorem}
\label{thm: M^tilde + general e}
Suppose Assumption~\ref{assumption: sup norm of e} holds, and let $\widetilde{M} = (\widehat{P}, \widetilde{r})$ be the modified DTR model constructed in Algorithm~\ref{algo: general}. Assume that the step sizes are constant over iterations, i.e., $\eta_k^{(t)} = \eta_k$ for all $t$. Then, Algorithm~\ref{algo: general} returns a policy $\pi^{(T)}$ that satisfies the following performance guarantee:
\begin{equation*}
V_{\widetilde{M}}^{\pi^\dagger} - V_{\widetilde{M}}^{\pi^{(T)}} \le \sum_{k=1}^K \left( \frac{\log |\mathcal{A}_k|}{\eta_k T} + \frac{\eta_k}{2} \widetilde{q}_k^2 + \frac{2}{\eta_k} \mathcal{O}_P(G_k) + \frac{T}{8\eta_k} \mathcal{O}_P(G_k^2) \right),
\end{equation*}
where $\widetilde{q}_k := \sum_{j=k}^K \|\widetilde{r}_j\|_\infty$ denotes the cumulative reward supremum norm from stage $k$ onward.
\end{STheorem}

The proof of Theorem \ref{thm: M^tilde + general e} relies on the following supporting lemmas. The proofs of these lemmas can be found in Section \ref{pf:lemma them S2}.

\begin{lemma}[Performance Difference Lemma]
\label{lemma: perf diff}
Let $P$ be the transition dynamics of a fixed DTR model. For any two policies $\pi$ and $\widehat{\pi}$, the following holds for any stage $k$ and any history $h_k$:
\begin{equation*}
\begin{aligned}
V_k^\pi(h_k) - V_k^{\widehat{\pi}}(h_k)
&= \sum_{t=k}^K \mathbb{E}_{(h_t, a_t) \sim d_t^\pi(\cdot,\cdot \mid H_k = h_k)} \left[ A_t^{\widehat{\pi}}(h_t, a_t) \right] \\
&= \sum_{t=k}^K \mathbb{E}_{h_t \sim d_t^\pi(\cdot \mid H_k = h_k)} \left\langle Q_t^{\widehat{\pi}}(h_t, \cdot),\ \pi_t(\cdot \mid h_t) - \widehat{\pi}_t(\cdot \mid h_t) \right\rangle,
\end{aligned}
\end{equation*}
where $A_t^{\widehat{\pi}}(h_t, a_t) := Q_t^{\widehat{\pi}}(h_t, a_t) - V_t^{\widehat{\pi}}(h_t)$ is the advantage function of policy $\widehat{\pi}$ at step $t$.
\end{lemma}

\begin{lemma}
\label{lemma: 3 KL terms}
For any $1\le k\le K$, suppose $\theta_k, \theta'_k \in \Theta_k$, $\pi_k(\theta_k)$ and $\pi_k(\theta'_k)$ are two policies in the policy class $\Pi_k$ defined in (\ref{eq: policy class}). Let $p(\cdot)$ be any probability distribution on $\mathcal{A}_k$. Then for any $h_k\in \mathcal{H}_k$, we have
\begin{equation*}
\begin{split}
& \quad\operatorname{KL}(p,\pi_k(h_k, \theta_k)) - \operatorname{KL}(p,\pi_k(h_k, \theta'_k)) - \operatorname{KL}(\pi_k(h_k, \theta'_k), \pi_k(h_k, \theta_k)) \\
&= \Bigl\langle f_k(\theta'_k, h_k,\cdot)-f_k(\theta_k, h_k,\cdot),\ p(\cdot)-\pi_k(\cdot|h_k, \theta'_k) \Bigr\rangle.
\end{split}
\end{equation*}
\end{lemma}

\begin{lemma}
\label{lemma: <Q,pi_t+1 - pi_t>}
For any $0\le t\le T-1$, $1\le k\le K$, and any $h_k\in \mathcal{H}_k$, we have
\begin{equation*}
\Bigl\langle Q^{\pi^{(t)}}_{k,\widetilde{M}}(h_k,\cdot), \pi_k(\cdot|h_k, \theta_k^{(t+1)}) - \pi_k(\cdot|h_k, \theta_k^{(t)}) \Bigr\rangle + \frac{1}{4 \eta_k^{(t)}} \|e^{(t)}_k(h_k,\cdot)\|_\infty^2 \ge 0.
\end{equation*}
\end{lemma}

\begin{lemma}
\label{lemma: <Q,pi* - pi_t>}
For any $0 \le t \le T-1$, $1 \le k \le K$, and any $h_k \in \mathcal{H}_k$, we have:
\begin{equation*}
\begin{split}
\left\langle Q^{\pi^{(t)}}_{k,\widetilde{M}}(h_k,\cdot),\ \pi_k^\dagger(\cdot|h_k) - \pi_k^{(t)}(\cdot|h_k) \right\rangle 
\le\ & \frac{1}{\eta_k^{(t)}} \left[ \operatorname{KL}(\pi_k^\dagger(h_k),\ \pi_k^{(t)}(h_k)) - \operatorname{KL}(\pi_k^\dagger(h_k),\ \pi_k^{(t+1)}(h_k)) \right] \\
& + \frac{\eta_k^{(t)}}{2} \left\| Q^{\pi^{(t)}}_{k,\widetilde{M}}(h_k,\cdot) \right\|_\infty^2 + \frac{2}{\eta_k^{(t)}} \left\| e_k^{(t)}(h_k,\cdot) \right\|_\infty.
\end{split}
\end{equation*}
\end{lemma}

\begin{proof}[Proof of Theorem \ref{thm: M^tilde + general e}]
By the performance difference lemma (Lemma~\ref{lemma: perf diff}), we have
\begin{equation*}
V_{1, \widetilde{M}}^{\pi^\dagger}(h_1) - V_{1, \widetilde{M}}^{\pi^{(t)}}(h_1) = \sum_{k=1}^K \mathbb{E}_{h_k \sim d_k^{\pi^\dagger}(\cdot | H_1 = h_1)} \left\langle Q_{k, \widetilde{M}}^{\pi^{(t)}}(h_k,\cdot),\ \pi_k^\dagger(\cdot | h_k) - \pi_k^{(t)}(\cdot | h_k) \right\rangle.
\end{equation*}
Taking expectation over \( h_1 \sim \mu_1 \), we obtain:
\begin{equation*}
V_{1, \widetilde{M}}^{\pi^\dagger}(\mu_1) - V_{1, \widetilde{M}}^{\pi^{(t+1)}}(\mu_1) = \sum_{k=1}^K \mathbb{E}_{h_k \sim d_k^{\pi^\dagger}} \left\langle Q_{k, \widetilde{M}}^{\pi^{(t+1)}}(h_k,\cdot),\ \pi_k^\dagger(\cdot | h_k) - \pi_k^{(t+1)}(\cdot | h_k) \right\rangle.
\end{equation*}
Summing over \( t = 0, 1, \dots, T-1 \), we have:
\begin{equation}
\label{eq:sum over}
\sum_{t=0}^{T-1} \left(V_{\widetilde{M}}^{\pi^\dagger} - V_{\widetilde{M}}^{\pi^{(t+1)}}\right) = \sum_{k=1}^K \mathbb{E}_{h_k \sim d_k^{\pi^\dagger}} \sum_{t=0}^{T-1} \left\langle Q_{k, \widetilde{M}}^{\pi^{(t+1)}}(h_k,\cdot),\ \pi_k^\dagger(\cdot | h_k) - \pi_k^{(t+1)}(\cdot | h_k) \right\rangle.
\end{equation}

Using Lemma~\ref{lemma: <Q,pi* - pi_t>}, we upper bound each inner product by:
\begin{equation}
\label{lemma 6 <,>}
\begin{split}
\left\langle Q_{k, \widetilde{M}}^{\pi^{(t+1)}}(h_k,\cdot), \pi_k^\dagger(\cdot | h_k) - \pi_k^{(t+1)}(\cdot | h_k) \right\rangle
\le & \frac{1}{\eta_k} \left[\mathrm{KL}(\pi_k^\dagger(h_k), \pi_k^{(t+1)}) - \mathrm{KL}(\pi_k^\dagger(h_k), \pi_k^{(t+2)}) \right] \\
&+ \frac{\eta_k}{2} \|Q_{k, \widetilde{M}}^{\pi^{(t+1)}}(h_k, \cdot)\|_\infty^2 + \frac{2}{\eta_k} \|e_k^{(t+1)}(h_k, \cdot)\|_\infty.
\end{split}
\end{equation}

Substituting \eqref{lemma 6 <,>} into \eqref{eq:sum over}, we obtain three additive terms that we now upper bound individually. First, the KL divergence term forms a telescoping sum over iterations, which simplifies to
\[
\sum_{t=0}^{T-1} \frac{1}{\eta_k} \left[ \mathrm{KL}(\pi_k^\dagger(h_k), \pi_k^{(t+1)}) - \mathrm{KL}(\pi_k^\dagger(h_k), \pi_k^{(t+2)}) \right] \le \frac{\log |\mathcal{A}_k|}{\eta_k},
\]
since the KL-divergence is non-negative.

Second, the sum of squared Q-values across iterations is bounded as
\[
\sum_{t=0}^{T-1} \frac{\eta_k}{2} \left\| Q_{k,\widetilde{M}}^{\pi^{(t+1)}}(h_k, \cdot) \right\|_\infty^2 \le \frac{\eta_k}{2} T \widetilde{q}_k^2,
\quad \text{where} \quad \widetilde{q}_k := \sum_{j=k}^K \| \widetilde{r}_j \|_\infty.
\]
Finally, the model error accumulates linearly over time and is controlled via the bound
\[
\sum_{t=0}^{T-1} \frac{2}{\eta_k} \left\| e_k^{(t+1)}(h_k, \cdot) \right\|_\infty \le \frac{2T}{\eta_k} \mathcal{O}_P(G_k).
\]
Combining these three bounds yields a total regret bound:
\begin{equation}
\label{eq: proof thm M^tilde + general e, eq 4}
\sum_{t=0}^{T-1} \left(V_{\widetilde{M}}^{\pi^\dagger} - V_{\widetilde{M}}^{\pi^{(t+1)}}\right) \le \sum_{k=1}^K \left( \frac{\log |\mathcal{A}_k|}{\eta_k} + \frac{\eta_k}{2} T \widetilde{q}_k^2 + \frac{2T}{\eta_k} \mathcal{O}_P(G_k) \right).
\end{equation}

We now use Lemma~\ref{lemma: <Q,pi_t+1 - pi_t>}, which implies:
\begin{equation}
\label{eq: proof thm M^tilde + general e, eq 5}
V_{\widetilde{M}}^{\pi^{(t)}} - V_{\widetilde{M}}^{\pi^{(t+1)}} \le \sum_{k=1}^K \frac{1}{4\eta_k} \mathcal{O}_P(G_k^2).
\end{equation}
For any \( t \in \{0, \dots, T-1\} \), it follows that:
\begin{equation}
\label{eq: proof thm M^tilde + general e, eq 6}
V_{\widetilde{M}}^{\pi^\dagger} - V_{\widetilde{M}}^{\pi^{(T)}} \le V_{\widetilde{M}}^{\pi^\dagger} - V_{\widetilde{M}}^{\pi^{(t+1)}} + (T-1-t) \sum_{k=1}^K \frac{1}{4\eta_k} \mathcal{O}_P(G_k^2).
\end{equation}

Summing \eqref{eq: proof thm M^tilde + general e, eq 6} over all \( t = 0, \dots, T-1 \), and using \eqref{eq: proof thm M^tilde + general e, eq 4}, we get:
\begin{equation*}
T \left(V_{\widetilde{M}}^{\pi^\dagger} - V_{\widetilde{M}}^{\pi^{(T)}}\right) \le \sum_{k=1}^K \left( \frac{\log |\mathcal{A}_k|}{\eta_k} + \frac{\eta_k}{2} T \widetilde{q}_k^2 + \frac{2T}{\eta_k} \mathcal{O}_P(G_k) + \frac{T^2}{8\eta_k} \mathcal{O}_P(G_k^2) \right).
\end{equation*}

Dividing both sides by \( T \) completes the proof.
\end{proof}

\begin{proof}[Proof of Theorem \ref{thm: combined main thm}]
From Theorem \ref{thm: M^tilde + general e}, we have
\begin{equation*}
V_{\widetilde{M}}^{\pi^\dagger} - V_{\widetilde{M}}^{\pi^{(T)}} \le \sum_{k=1}^K \left( \frac{\log |\mathcal{A}_k|}{\eta_k T} + \frac{\eta_k}{2} (\widetilde{q}_k)^2 + \frac{2}{\eta_k} \mathcal{O}_P(G_k) + \frac{T}{4\eta_k} \mathcal{O}_P(G_k^2) \right).
\end{equation*}

We next invoke Theorem \ref{thm: final regret real M}, which ensures that, with probability at least $1 - K\delta$,
\begin{equation*}
\left( V_{M^*}^{\pi^\dagger} - V_{\widetilde{M}}^{\pi^\dagger} \right) + \left( V_{\widetilde{M}}^{\pi^{(T)}} - V_{M^*}^{\pi^{(T)}} \right) \le V^{\pi^\dagger}_{P^*, \underline{r}}.
\end{equation*}

Combining the two inequalities yields:
\begin{equation*}
V_{M^*}^{\pi^\dagger} - V_{M^*}^{\pi^{(T)}} 
\le V^{\pi^\dagger}_{P^*, \underline{r}} + \sum_{k=1}^K \left( \frac{\log |\mathcal{A}_k|}{\eta_k T} + \frac{\eta_k}{2} (\widetilde{q}_k)^2 + \frac{2}{\eta_k} \mathcal{O}_P(G_k) + \frac{T}{4\eta_k} \mathcal{O}_P(G_k^2) \right),
\end{equation*}
which holds with probability at least $1 - K\delta$.

Finally, by setting $\eta_k = \frac{c_{\eta}}{\sqrt{T}}$ in Theorem \ref{thm: M^tilde + general e}, the desired bound follows directly.
\end{proof}

\subsubsection{Proof of Proposition \ref{prop: G_k}}
\label{subsub:pf prop}
To ensure effective approximation of the Q-function,
we adopt the B-splines as the function basis $\Upsilon_{L_k}(\cdot) = \{\upsilon_{L_k,1}(\cdot), \upsilon_{L_k,2}(\cdot), ..., \upsilon_{L_k,L_k}(\cdot)\} ^T.$ The well-definedness of the B-spline functions requires the following assumption on state spaces.

\begin{SAssumption}
\label{assumption: state space is rectangular}
For any $1\le k\le K$, the $k$-th state space $\mathcal{S}_k$ is the Cartesian product of $d_k^{(s)}$ compact and non-degenerate intervals, i.e. $\mathcal{S}_k$ is a rectangle in $\mathbb{R}^{d_k^{(s)}}$ with a nonempty interior.
\end{SAssumption}

We also need to impose assumptions on the target function $Q_{k, \widetilde{M}} ^{\pi^{(t)}}$ so that our selected function class has sufficient representation power. In this case, we present a series of technical assumptions regarding the modified DTR model $\widetilde{M}$.
\begin{SAssumption}[$p$-smoothness of the modified reward]
\label{assumption: reward is p-smooth} 
For any $1\le k\le K$, any fixed $(a_1,...,a_k) \in \mathcal{A}_1 \times \cdots \times \mathcal{A}_k$, the modified reward $\widetilde{r}_k (h_k, a_k) = \widetilde{r}_k (s_1, a_1,..., s_k, a_k)$ as a function of $(s_1,s_2,...,s_k)$ is $p$-smooth in $(s_1,s_2,...,s_k)$.
\end{SAssumption}

The smoothness assumption can also be found in \cite{Huang1998, chen2007large} . See Supplement \ref{subsec:p-smooth} for a detailed description of $p-$smoothness. 
We assume the modified expected reward function, the transition dynamics in the model class, and the policies in the policy class are all $p$-smooth in some certain sense.

\begin{SAssumption}[$p$-smoothness of transition density]
\label{assumption: transition is p-smooth} 
For any $1\le k\le K$, any $P_k\in \mathcal{P}_k$, any $h_k\in \mathcal{H}_k$, there exists transition density $p_k(\cdot|h_k,a_k)$ such that $P_k(\mathrm{d} s |h_k,a_k) = p_k(s|h_k,a_k) \mathrm{d} s$.
  Furthermore, for any fixed $(a_1,...,a_k) \in \mathcal{A}_1 \times \cdots \times \mathcal{A}_k$, the transition density $p_k(s_{k+1}| h_k,a_k) = p_k(s_{k+1}| s_1,a_1,..., s_k,a_k)$ as a function of $(s_1,s_2,..., s_k, s_{k+1})$ is $p$-smooth in $(s_1,s_2,..., s_{k+1})$.
\end{SAssumption}
This assumption can be satisfied if the policy class $\mathcal{P}_k$ belongs to the class of absolutely continuous distributions.

\begin{SAssumption} [$p$-smoothness of policy]
\label{assumption: policy is p-smooth}
For any $\pi \in \Pi$, any $1\le k\le K$, any fixed $(a_1,...,a_k) \in \mathcal{A}_1 \times \cdots \times \mathcal{A}_k$, the probability $\pi_k (a_k| h_k) = \pi_k (a_k| s_1, a_1,..., s_k)$ as a function of $(s_1,s_2,...,s_k)$ is $p$-smooth in $(s_1,s_2,...,s_k)$.
\end{SAssumption}


The Monte Carlo approximation error from line 9-10 of Algorithm \ref{algo: sieve regression} is defined as $\epsilon_k^{(i)}(t, \overline{a}_k) = \widehat{Q}_k^{(t)} (\overline{s}_k^{(i)}, \overline{a}_k) - Q_{k, \widetilde{M}} ^{\pi^{(t)}} (\overline{s}_k^{(i)}, \overline{a}_k)$. We need to restrict these errors by the following assumption. For notational simplicity, we fix $t$ and $\overline{a}_k$, and denote $\epsilon_k^{(i)}$ as the shorthand for $\epsilon_k^{(i)}(t, \overline{a}_k)$.
\begin{SAssumption}
\label{assumption: MC error}
For any fixed $t$ and fixed $\overline{a}_k$, the $m_k$ error terms $\epsilon_k^{(1)}, ..., \epsilon_k^{(m_k)}$ defined above satisfy 
\begin{enumerate}
    \item $\mathbb{E}[\epsilon_k^{(i)} \mid \overline{s}_k^{(i)}] = 0$, and $(\overline{s}_k^{(1)}, \epsilon_k^{(1)}), ..., (\overline{s}_k^{(m_k)}, \epsilon_k^{(m_k)})$ are i.i.d.
    \item $\mathbb{E}[ (\epsilon_k^{(i)} )^2 \mid \overline{s}_k^{(i)}]$ are uniformly bounded for all $i$ almost surely.
    \item $\mathbb{E}[ (\epsilon_k^{(i)} )^{2+\delta_k} ] < \infty$ for some $\delta_k>0.$
\end{enumerate}
\end{SAssumption}

\begin{STheorem}
\label{thm: M^tilde + sieve reg} \textbf{[Proposition \ref{prop: G_k} Restated] }Suppose \(\widetilde{M}=(\widehat{P},\widetilde{r})\) is the modified DTR model constructed in Algorithm~\ref{algo: sieve regression}, and that Assumptions~\ref{assumption: state space is rectangular},~\ref{assumption: reward is p-smooth},~\ref{assumption: transition is p-smooth},~\ref{assumption: policy is p-smooth}, and~\ref{assumption: MC error} hold. Let \(\eta_k^{(t)} = \frac{\sqrt{ \log |\mathcal{A}_k|}}{ \widetilde{q}_k \sqrt{T} }\), constant in \(t\). If \(L_k \asymp \left(\frac{m_k}{\log m_k}\right)^{\overline{d}_k /(2p+\overline{d}_k)}\) and \(p \delta_k \ge \overline{d}_k\), then Algorithm~\ref{algo: sieve regression} returns a policy \(\pi^{(T)}\) satisfying:
\begin{equation*}
V_{ \widetilde{M}}^{\pi^\dagger} - V_{ \widetilde{M}}^{\pi^{(T)}} =  \mathcal{O} \left( \frac{K}{\sqrt{T}} \right) + K \sqrt{T} \cdot \mathcal{O}_p\left( \left( \frac{\log m_k}{m_k} \right)^ {2p/ (2p+\overline{d}_k)} \right)
\end{equation*}
as \(m_k, L_k \to \infty\), where \(\widetilde{q}_k := \sum_{j=k}^K \|\widetilde{r}_j\|_\infty\), and \(\overline{d}_k = \dim(\mathcal{S}_1 \times \cdots \times \mathcal{S}_k)\).
\end{STheorem}

The proof of Theorem \ref{thm: M^tilde + sieve reg} relies on the following auxiliary lemma.
\begin{lemma}
\label{lemma: upperbound e}
Suppose Assumptions \ref{assumption: state space is rectangular}, \ref{assumption: reward is p-smooth}, \ref{assumption: transition is p-smooth}, \ref{assumption: policy is p-smooth}, and \ref{assumption: MC error} hold. If $L_k \asymp \left(\frac{m_k}{\log m_k}\right)^{\overline{d}_k /(2p+\overline{d}_k)}$ and $p \delta_k \ge \overline{d}_k$, then
\[
\frac{1}{ \eta_k^{(t)} \| Q_{k,\widetilde{M}}^{\pi^{(t)}} \|_\infty } \sup_{(h_k,a_k)} \left| e^{(t)}_k(h_k,a_k) \right| = \mathcal{O}_P\left( \left( \frac{\log m_k}{m_k} \right)^ {p/ (2p+\overline{d}_k)} \right),
\]
as $m_k, L_k \to \infty$, where $\overline{d}_k = \dim(\mathcal{S}_1 \times \mathcal{S}_2 \times \cdots \times \mathcal{S}_k)$.
\end{lemma}
The proof of Lemma \ref{lemma: upperbound e} can be found in Section \ref{subsec:p-smooth}.
\begin{proof}[Proof of Theorem~\ref{thm: M^tilde + sieve reg}]
We follow the proof of Theorem~\ref{thm: M^tilde + general e}, as Algorithm~\ref{algo: sieve regression} is a special version of Algorithm~\ref{algo: general}. All previous results that do not rely on Assumption~\ref{assumption: sup norm of e} still hold. Similarly to \eqref{eq: proof thm M^tilde + general e, eq 4}, we obtain that:
\begin{equation}
\label{eq: proof thm M^tilde + sieve reg, eq 1}
\sum_{t=0}^{T-1} \left( V_{1, \widetilde{M}}^{\pi^\dagger}(\mu_1) - V_{1, \widetilde{M}}^{\pi^{(t+1)}}(\mu_1) \right) \le \sum_{k=1}^K \frac{\log |\mathcal{A}_k|}{\eta_k} + \sum_{k=1}^K \frac{\eta_k}{2} T (\widetilde{q}_k)^2 + \sum_{k=1}^K \sum_{t=0}^{T-1} \frac{2}{\eta_k} \sup_{(h_k,a_k)} |e^{(t)}_k(h_k,a_k)|.
\end{equation}
We also adapt the analysis in equations~\eqref{eq: proof thm M^tilde + general e, eq 5} and~\eqref{eq: proof thm M^tilde + general e, eq 6}, retaining all terms involving \(e_k^{(t)}\):
\begin{equation}
\label{eq: proof thm M^tilde + sieve reg, eq 2}
\begin{split}
V_{1, \widetilde{M}}^{\pi^\dagger}(\mu_1) - V_{1, \widetilde{M}}^{\pi^{(T)}}(\mu_1) 
&\le V_{1, \widetilde{M}}^{\pi^\dagger}(\mu_1) - V_{1, \widetilde{M}}^{\pi^{(t+1)}}(\mu_1) + \sum_{j=t+1}^{T-1} \sum_{k=1}^K \frac{1}{4 \eta_k} \left( \sup_{(h_k,a_k)} |e^{(j)}_k(h_k,a_k)| \right)^2.
\end{split}
\end{equation}
Summing \eqref{eq: proof thm M^tilde + sieve reg, eq 2} over \(t = 0, \ldots, T-1\) and combining with \eqref{eq: proof thm M^tilde + sieve reg, eq 1}, we obtain:
\begin{equation}
\label{eq: proof thm M^tilde + sieve reg, eq 3}
\begin{split}
T(V_{1, \widetilde{M}}^{\pi^\dagger}(\mu_1) - V_{1, \widetilde{M}}^{\pi^{(T)}}(\mu_1)) 
&\le \sum_{k=1}^K \left( \frac{\log |\mathcal{A}_k|}{\eta_k} + \frac{\eta_k}{2} T (\widetilde{q}_k)^2 \right) + \sum_{k=1}^K \sum_{t=0}^{T-1} \frac{2}{\eta_k} \sup_{(h_k,a_k)} |e^{(t)}_k(h_k,a_k)| \\
&\quad + \sum_{j=1}^{T-1} \sum_{t=0}^{j} \sum_{k=1}^K \frac{1}{4 \eta_k} \left( \sup_{(h_k,a_k)} |e^{(j)}_k(h_k,a_k)| \right)^2.
\end{split}
\end{equation}
By Lemma~\ref{lemma: upperbound e}, and the fact that \(\| Q_{k,\widetilde{M}}^{\pi^{(t)}} \|_\infty \le \widetilde{q}_k\), we have:
\begin{equation}
\label{eq: proof thm M^tilde + sieve reg, eq 5}
\sup_{(h_k,a_k)} |e^{(t)}_k(h_k,a_k)| \le \eta_k \widetilde{q}_k \mathcal{O}_p\left( \left( \frac{\log m_k}{m_k} \right)^ {p/ (2p+\overline{d}_k)} \right).
\end{equation}
Substituting~\eqref{eq: proof thm M^tilde + sieve reg, eq 5} into~\eqref{eq: proof thm M^tilde + sieve reg, eq 3}, and simplifying yields the desired bound:
\begin{equation*}
V_{1, \widetilde{M}}^{\pi^\dagger}(\mu_1) - V_{1, \widetilde{M}}^{\pi^{(T)}}(\mu_1) = \mathcal{O}\left(\frac{K}{\sqrt{T}}\right) + K\sqrt{T}\cdot \mathcal{O}_p\left(\left(\frac{\log m_k}{m_k}\right)^{2p/(2p+\overline{d}_k)}\right).
\end{equation*}
\end{proof}

\subsection{Proof of Theorem \ref{thm: suboptimality for linear} and Theorem \ref{thm:lower-bound-hp} in Section \ref{sec:linear model}}
\label{proof:linear}
To establish Theorem~\ref{thm: suboptimality for linear}, we impose two assumptions to govern the growth rate of the feature dimensions and bound the magnitude of the feature mappings.
\begin{SAssumption}
\label{assumption: sample complexity linear}
There exist constants $c_1$, $c_2$ such that the following two conditions hold for all $1\le k\le K$:
(i) $dim(\mathcal{S}_k) \le c_1$; (ii) 
$\dim(\phi_k)\le c_2 k$.
\end{SAssumption}
Assumption \ref{assumption: sample complexity linear} (i) trivially holds because $K$ is finite. Therefore, $dim(\mathcal{S}_k) = \mathcal{O}(1)$, which further implies $dim(\mathcal{H}_k \times \mathcal{A}_k) = k+\sum_{j=1}^k dim(\mathcal{S}_j) = \mathcal{O}(k)$. Recall that $\phi_k: \mathcal{H}_k \times \mathcal{A}_k \to 
\mathbb{R}^{\dim(\phi_k)} $, so Assumption \ref{assumption: sample complexity linear} (ii) is a mild condition that naturally requires the feature mapping $\phi_k$ to reserve the dimensional structure of $\mathcal{H}_k \times \mathcal{A}_k$.

\begin{SAssumption}[Uniformly Bounded Features]\label{assump:bounded-features}


The feature map 
\(
  \phi_k:\mathcal H_k\times\mathcal A_k\to \mathbb{R}^{\dim(\phi_k)}
\)
is uniformly bounded. There exists a constant \(R > 0\) such that
\[
  \|\phi_k(h,a)\|_2^2 \;\le\; R^2 \dim(\phi_k),
  \quad \text{for all }(h,a)\in\mathcal H_k\times\mathcal A_k.
\]

\cbend

\end{SAssumption}

This assumption is standard in high-dimensional regression \citep{Wainwright2019,jin2019provablyefficientreinforcementlearning,Abbasi-Yadkori2011}, which ensures that our uncertainty quantifiers of the linear transition model are well-controlled.

Before proceeding to the main proof, we state a key corollary that upper bounds the model shift error of our linear transition model under the above assumptions.

\begin{Scorollary}
\label{cor: partial coverage linear}
Suppose Assumption \ref{assumption: partial coverage linear} holds under the setting of Theorem \ref{thm: final regret real M}. If $\widehat{W}_k$ and $\Gamma_k$ follow the form in (\ref{eq: KNR, W hat}) and (\ref{eq: KNR, Gamma}) respectively, 
then for any policy $\pi$, the following inequality holds with probability at least $1-K\delta$:
$$\left( V_{M^*} ^{\pi^\dagger} -V_{\widetilde{M}} ^{\pi^\dagger} \right) + \left( V_{\widetilde{M}}^{\pi} -V_{M^*}^{\pi} \right) \lesssim   {\mathcal{O}}\left( \frac{K^3\sqrt{\log n}}{\sqrt{n}} + \frac{K^5}{n} \right).$$
\end{Scorollary}

\begin{proof}[Proof of Corollary~\ref{cor: partial coverage linear}]
By Theorem~\ref{thm: uncertainty quantifier KNR}, the estimators $\widehat{W}_k$ and the associated uncertainty quantifiers $\Gamma_k$ constructed via (\ref{eq: KNR, W hat}) and (\ref{eq: KNR, Gamma}) satisfy Assumption~\ref{assumption: uncertainty quantifier}. Consequently, applying Theorem~\ref{thm: final regret real M}, we obtain that with probability at least $1 - K\delta$,
\begin{equation}
\label{eq:cor_partial_coverage_regret}
\left( V_{M^*}^{\pi^\dagger} - V_{\widetilde{M}}^{\pi^\dagger} \right) + \left( V_{\widetilde{M}}^{\pi} - V_{M^*}^{\pi} \right)
\le \sum_{k=1}^K b_k \, \mathbb{E}_{(h_k,a_k) \sim d_{P^*,k}^{\pi^\dagger}} \Gamma_k(h_k,a_k),
\end{equation}
where $b_k = 2\widetilde{c}_k + \sum_{j=k+1}^K \widetilde{c}_j \|\Gamma_j\|_\infty$ and $\widetilde{c}_k = \sum_{j=k}^K \|\overline{r}_j\|_\infty$.

To bound the right-hand side of \eqref{eq:cor_partial_coverage_regret}, we begin by analyzing the behavior of the uncertainty quantifier $\Gamma_k(h_k, a_k)$ defined in~(\ref{eq: KNR, Gamma}). Recall that
\begin{equation*}
\Gamma_k(h_k,a_k) \le 2 B_{d_{k+1}^{(s)}} \sigma C_L \, \beta_{k,n} \sqrt{ \phi_k(h_k,a_k)^\top \Lambda_{k,n}^{-1} \phi_k(h_k,a_k)},
\end{equation*}
where $\sigma$ and $C_L$ are fixed constants, and $B_{d_{k+1}^{(s)}} = \pi^{d_{k+1}^{(s)}/2} / \Gamma(d_{k+1}^{(s)}/2 + 1)$ denotes the Lebesgue measure of the unit ball in $\mathbb{R}^{d_{k+1}^{(s)}}$. Noting that $B_{d_{k+1}^{(s)}} \le 1$, we simplify the expression as $\Gamma_k(h_k,a_k) \lesssim \beta_{k,n} \sqrt{ \phi_k(h_k,a_k)^\top \Lambda_{k,n}^{-1} \phi_k(h_k,a_k)}$. The scaling factor $\beta_{k,n}$ is defined in~(\ref{eq: KNR, eq 2}) as
\[
\beta_{k,n} = \sqrt{\lambda} \|W^*_k\|_2 + \sqrt{8\sigma^2 d_{k+1}^{(s)} \log(5) + 8\sigma^2 \log \left( \frac{\det (\Lambda_{k,n})^{1/2}}{\det (\lambda I)^{1/2} \, \delta} \right)}.
\]

Since $\Lambda_{k,n}$ is the regularized empirical covariance matrix based on $n$ i.i.d.\ samples, we have $\det(\Lambda_{k,n}) = \mathcal{O}(n^{d_k^{(s)}})$.

Therefore, under Assumption~\ref{assumption: sample complexity linear} (ii), it follows that
$\beta_{k,n} \lesssim \sqrt{d_k^{(s)} \log n} \lesssim \sqrt{\log n}$. Hence,
\begin{equation}
\label{eq:Gamma_bound}
\Gamma_k(h_k,a_k) \lesssim \sqrt{\log n} \cdot \sqrt{ \phi_k(h_k,a_k)^\top \Lambda_{k,n}^{-1} \phi_k(h_k,a_k)}.
\end{equation}

Taking expectation under $d_{P^*,k}^{\pi^\dagger}$ in \eqref{eq:Gamma_bound}, and applying Jensen's inequality to the concave square root function, yields
\begin{equation}
\label{eq:Gamma_expectation_bound}
\mathbb{E}_{(h_k,a_k) \sim d_{P^*,k}^{\pi^\dagger}} \Gamma_k(h_k,a_k)
\lesssim \sqrt{ \log n \cdot \mathbb{E}_{(h_k,a_k) \sim d_{P^*,k}^{\pi^\dagger}} \phi_k(h_k,a_k)^\top \Lambda_{k,n}^{-1} \phi_k(h_k,a_k) }.
\end{equation}

To bound the quadratic form inside the square root, observe that $\Lambda_{k,n}^{-1}$ is symmetric positive definite, and thus admits the eigendecomposition $\Lambda_{k,n}^{-1} = \sum_{i=1}^{d_k^{(s)}} \lambda_i u_i u_i^\top$. By Assumption~\ref{assumption: partial coverage linear}, for any eigenvector $u_i$,
\[
\mathbb{E}_{(h_k,a_k) \sim d_{P^*,k}^{\pi^\dagger}} (u_i^\top \phi_k(h_k,a_k))^2 \le C_k \cdot \mathbb{E}_{(h_k,a_k) \sim \rho_k} (u_i^\top \phi_k(h_k,a_k))^2,
\]
This yields 
\begin{align*}
\mathbb{E}_{(h_k,a_k)\sim d_{P^*,k}^{\pi^\dagger}} \phi_k(h_k,a_k)^\top \Lambda_{k,n}^{-1} \phi_k(h_k,a_k)
&= \sum_{i=1}^{d_k^{(s)}} \lambda_i \, \mathbb{E}_{(h_k,a_k)\sim d_{P^*,k}^{\pi^\dagger}} \left[ (u_i^\top \phi_k(h_k,a_k))^2 \right] \nonumber \\
&\le C_k \sum_{i=1}^{d_k^{(s)}} \lambda_i \, \mathbb{E}_{(h_k,a_k)\sim \rho_k} \left[ (u_i^\top \phi_k(h_k,a_k))^2 \right] \nonumber \\
&= C_k \cdot \mathrm{Tr}\left( \Psi_k \Lambda_{k,n}^{-1} \right),
\label{eq:trace_bound}
\end{align*}
where $\Psi_k = \mathbb{E}_{\rho_k} [ \phi_k(h_k,a_k)\phi_k(h_k,a_k)^\top ]$.

Since $\Psi_k$ is full rank and $\Lambda_{k,n}/n \to \Psi_k$ elementwise by the law of large numbers, the trace term satisfies that $\mathrm{Tr}(\Psi_k \Lambda_{k,n}^{-1}) = \mathcal{O}(d_k^{(s)} / n)$. Substituting this into \eqref{eq:Gamma_expectation_bound}, we conclude
 
\begin{equation}
\label{eq:Gamma_expectation_final}
\mathbb{E}_{(h_k,a_k)\sim d_{P^*,k}^{\pi^\dagger}} \Gamma_k(h_k,a_k)
\lesssim \sqrt{ \frac{C_k \log n}{n} }.
\end{equation}
 
We now turn to bounding $\|\Gamma_k\|_\infty$ for use in $b_k$. By the expression for $\Gamma_k$ in \eqref{eq: KNR, Gamma} and using similar simplification as above, we have:
\[
\|\Gamma_k\|_\infty \lesssim \beta_{k,n} \cdot \sup_{(h_k,a_k)} \sqrt{ \phi_k(h_k,a_k)^\top \Lambda_{k,n}^{-1} \phi_k(h_k,a_k)}.
\]

To bound the supremum, we recall from Assumption~\ref{assump:bounded-features} that the feature map $\phi_k$ is uniformly bounded. In particular, $\|\phi_k(h_k,a_k)\|_2^2 \le R^2k$ for all $(h_k, a_k)$.  Therefore, we have

\[\phi_k(h_k,a_k)^\top \Lambda_{k,n}^{-1} \phi_k(h_k,a_k) \le \frac{\|\phi_k(h_k,a_k)\|_2^2}{\lambda_{\min}(\Lambda_{k,n})}\lesssim \frac{R^2k}{\lambda_{\min}(\Lambda_{k,n})}\]

Now we apply the Matrix Chernoff II bound (Corollary 5.2 in \citep{Tropp_2011}), it follows that $\lambda_{\min}(\Lambda_{k,n}) = \Omega(n)$ with high probability $1-\delta$. Hence,
\[
\phi_k(h_k,a_k)^\top \Lambda_{k,n}^{-1} \phi_k(h_k,a_k) \lesssim \frac{k}{n},
\]
uniformly over $(h_k,a_k)$, and thus

\begin{equation}
\label{eq:Gammainf_bound}
\|\Gamma_k\|_\infty \lesssim  \sqrt{ \frac{k\log n}{n} }.
\end{equation}

Since $\widetilde{c}_k = \sum_{j=k}^K \|\overline{r}_j\|_\infty = \mathcal{O}(K)$, we obtain the bound by plugging \eqref{eq:Gammainf_bound}:
\[
b_k = 2\widetilde{c}_k + \sum_{j=k+1}^K \widetilde{c}_j \|\Gamma_j\|_\infty
= \mathcal{O}\left( K + \frac{K^{5/2} \sqrt{\log n}}{\sqrt{n}} \right).
\]
Substituting the bounds on $\mathbb{E}_{(h_k,a_k) \sim d_{P^*,k}^{\pi^\dagger}} \Gamma_k(h_k,a_k)$ \eqref{eq:Gamma_bound} and $b_k$ into \eqref{eq:cor_partial_coverage_regret} yields the desired guarantee:
\[
\left( V_{M^*}^{\pi^\dagger} - V_{\widetilde{M}}^{\pi^\dagger} \right)
+ \left( V_{\widetilde{M}}^{\pi} - V_{M^*}^{\pi} \right)
\le {\mathcal{O}}\left( \frac{K^2\sqrt{\log n}}{\sqrt{n}} + \frac{K^{7/2} \log n}{n} \right),
\]
which completes the proof.

\cbend
\normalcolor

\end{proof}

\begin{proof}[Proof of Theorem \ref{thm: suboptimality for linear}]
By Theorem \ref{thm: combined main thm} , with high probability $1-K\delta$, we have
\begin{equation*}
\begin{aligned}
\text{Subopt}(\pi^{(T)}; M^*)
&\le \sum_{k=1}^K b_k \mathbb{E}_{(h_k, a_k)\sim d^{\pi^\dagger}_{P^*}} \Gamma_k(h_k, a_k )
+ \mathcal{O}\!\left( \frac{K}{\sqrt{T}} \right)
+ K \cdot \mathcal{O}_P\!\left( T^{1/2} G_k + T^{3/2} G_k^2 \right),\\
&\stackrel{(i)}{\le} \widetilde{\mathcal{O}}\!\left( \frac{ K^2}{ \sqrt{n}}+ \frac{K^{7/2}}{n} \right)
+ \mathcal{O}\!\left( \frac{K}{\sqrt{T}} \right)
+ K \sqrt{T} \cdot \mathcal{O}_p\!\left( \left( \frac{\log m_k}{m_k} \right)^{\alpha_1} \right),
\cbend
\end{aligned}
\end{equation*}
where $\widetilde{O}$ hides all logarithmic terms and $(i)$ comes from Corollary \ref{cor: partial coverage linear}.
\end{proof}

\begin{proof}[Proof of Theorem \ref{thm:lower-bound-hp}]
Define the index set $\mathcal{U}
  := \prod_{k=1}^K \{-1, +1\}^{k}
  = \big\{ u = (u_1, \dots, u_K) : u_k \in \{-1, +1\}^{k} \big\},$
and let $ m := \sum_{k=1}^K k = \frac{K(K+1)}{2}.$ For each $u\in\mathcal U$ we construct a DTR model $M_u$ with linear
transition kernels as follows. 
The transition kernels follow the  linear structure and do not depend on $u$:
\[
  \mathbb{E}[S_{k+1} \mid H_k,A_k]
  \;=\;
  W_k \,\phi_k(H_k,A_k),
  \qquad k=1,\dots,K,
\]
for some weight matrices $W_k$ and feature mapping
$\phi_k : \mathcal H_k \times \mathcal A_k \to \mathbb R^{\dim(\phi_k)}$ satisfying
$\dim(\phi_k) \lesssim k$. In our hard instance, we
choose $\dim(\phi_k) = k$ and, at each stage $k$, set the action space to $\mathcal A_k = \{-1,+1\}^{k}$
and define $ \phi_k(H_k,a_k) := a_k \in \mathbb R^{k}, \;\; a_k\in\mathcal A_k.$

The behavior policy is chosen to be uniform over actions: $\mu_k(a_k \mid H_k) = 2^{-k}, \;\; a_k\in\mathcal A_k,$
for all stages $k$ and all histories $H_k$. The mean reward at stage $k$ is assumed to be linear in the same feature mapping
$\phi_k$:
\[
  r_{u,k}(H_k,A_k)
  \;=\;
  r_{0,k}(H_k,A_k)
  \;+\;
  \frac{\zeta}{\sqrt{m}}\,
  \big\langle \phi_k(H_k,A_k), u_k \big\rangle
  \;=\;
  r_{0,k}(H_k,A_k)
  \;+\;
  \frac{\zeta}{\sqrt{m}}\,
  \big\langle A_k, u_k \big\rangle,
\]
for a scalar $\zeta>0$ to be specified later. The observed reward is $R_{u,k}(H_k,A_k)
  \sim
  \mathcal{N}\big( r_{u,k}(H_k,A_k),\,1 \big),$
independently across stages and trajectories. We denote by $\mathcal M := \{M_u : u\in\mathcal U\}.$

For any deterministic policy $\pi$, we associate a sign pattern
$u^\pi = (u_1^\pi,\dots,u_K^\pi)$ with $u_k^\pi \in \{-1,+1\}^k$ as follows:
at stage $k$, $\pi$ chooses some action $A_k^\pi(H_k)\in\{-1,+1\}^k$, and we
set $u_k^\pi := A_k^\pi(H_k).$
We then define the Hamming distance as follows
\[
  D_H(u^\pi,u)
  :=
  \sum_{k=1}^K \sum_{j=1}^k \mathbf{1}\!\big\{ u_k^\pi[j] \neq u_k[j] \big\},
  \qquad u\in\mathcal U.
\]

For a fixed $u$ and policy $\pi$, let $V_{M_u}^\pi$ denote its undiscounted expected
cumulative value.  At stage $k$, given the history $H_k$, the optimal policy $\pi^*$
for $M_u$ chooses the action $A_k^{\pi^*}(H_k) = u_k,$
which maximizes $\langle A_k,u_k\rangle$ over $A_k\in\{-1,+1\}^k$. In contrast, policy
$\pi$ chooses $A_k^\pi(H_k)=u_k^\pi$. 
Since $r_{0,k}$ does not depend on $u$, we have the reward difference:
\[
 r_{u,k}(H_k, A_k^{\pi^*}(H_k))-r_{u,k}(H_k, A_k^\pi(H_k))
  = \frac{\zeta}{\sqrt{m}}\Big(
      \langle u_k,u_k\rangle - \langle u_k,u_k^\pi\rangle
    \Big).
\]
Using $u_k[j],u_k^\pi[j]\in\{-1,+1\}$, we compute
\[
  \langle u_k,u_k\rangle - \langle u_k,u_k^\pi\rangle
  = \sum_{j=1}^k \big( u_k[j]^2 - u_k[j]u_k^\pi[j] \big)
  = \sum_{j=1}^k \big(1 - u_k[j]u_k^\pi[j]\big)
  = 2 \sum_{j=1}^k \mathbf 1\{u_k[j]\neq u_k^\pi[j]\}.
\]
Hence, for every history $H_k$,
\[
  r_{u,k}(H_k, A_k^{\pi^*}(H_k))-r_{u,k}(H_k, A_k^\pi(H_k))
  \;=\;
  \frac{2\zeta}{\sqrt{m}}
  \sum_{j=1}^k \mathbf 1\{u_k^\pi[j]\neq u_k[j]\}.
\]
Taking expectation and summing over $k=1,\dots,K$, we obtain that
\[
  V_{M_u}^{\pi^*} - V_{M_u}^{\pi}
  =
  \sum_{k=1}^K
  \mathbb E_{M_u}\big[ r_{u,k}(H_k, A_k^\pi(H_k))-r_{u,k}(H_k, A_k^\pi(H_k)) \big]
  =
  \frac{2\zeta}{\sqrt{m}}
  \sum_{k=1}^K \sum_{j=1}^k \mathbf 1\{u_k^\pi[j]\neq u_k[j]\}.
\]
Recall the definition of $D_H(u^\pi,u)$, this can be rewritten as
\begin{equation}
  V_{M_u}^{\pi^*} - V_{M_u}^{\pi}
  \;\ge\;
  \gamma_{\mathrm{v}} \,\frac{\zeta}{\sqrt{m}} \,
  D_H(u^\pi,u),
  \label{eq:value-vs-hamming}
\end{equation}
where we may take $\gamma_{\mathrm{v}} := 2$.

Now we bound the KL divergence between neighboring models. Let $Q_u$
denote the distribution of $\mathcal D$ under $M_u$, and let $P_u^{(1)}$ be
the distribution of a single trajectory under $M_u$. If $u,u'\in\mathcal U$
satisfy $D_H(u,u') = 1$, then the corresponding models $M_u$ and $M_{u'}$
have identical transition kernels, but their reward parameters differ in
exactly one component, say at stage $k^\star$ and coordinate $\ell^\star$ of
$u_{k^\star}$.

At stage $k^\star$, 
the reward distributions under $M_u$ and $M_{u'}$ conditioned on the action $A_{k^\star}=a\in\{-1,+1\}^k$ are
\[
  R_{u,k^\star} \mid A_{k^\star}=a
  \sim \mathcal N\!\left(
    r_{0,k^\star}(H_{k^\star},a)
    + \frac{\zeta}{\sqrt{m}}\langle u_{k^\star},a\rangle,\;1
  \right),
\]
\[
  R_{u',k^\star} \mid A_{k^\star}=a
  \sim \mathcal N\!\left(
    r_{0,k^\star}(H_{k^\star},a)
    + \frac{\zeta}{\sqrt{m}}\langle u'_{k^\star},a\rangle,\;1
  \right),
\]
while at all other stages $k\neq k^\star$ the reward distributions coincide.
Since $u_{k^\star}$ and $u'_{k^\star}$ differ only in coordinate
$\ell^\star$, we have $u_{k^\star} - u'_{k^\star} = \pm 2 e_{\ell^\star}.$ Therefore, for any $a\in\{-1,+1\}^k$,
\[
  \langle u_{k^\star},a\rangle - \langle u'_{k^\star},a\rangle
  = (u_{k^\star}[\ell^\star] - u'_{k^\star}[\ell^\star])\,a[\ell^\star]
  \in \{\pm 2\}.
\]

Therefore, conditioned on $A_{k^\star}=a$,
\[
  \mathrm{KL}\!\Big(
    \text{Law}(R_{u,k^\star}\mid A_{k^\star}=a)
    \,\big\|\,
    \text{Law}(R_{u',k^\star}\mid A_{k^\star}=a)
  \Big)
  = \frac{1}{2}\cdot \frac{4\zeta^2}{m}
  = \frac{2\zeta^2}{m},
\]
and this quantity does not depend on $a$. Taking expectation over the random
action $A_{k^\star}$ 
yields
\[
  \mathrm{KL}\big(P_u^{(1)} \,\|\, P_{u'}^{(1)}\big)
  = \frac{2\zeta^2}{m}.
\]
Because the $n$ trajectories in $\mathcal D$ are i.i.d., the KL divergence
between $Q_u$ and $Q_{u'}$ is
\begin{equation}
  \mathrm{KL}(Q_u \,\|\, Q_{u'})
  = n\,\mathrm{KL}\big(P_u^{(1)} \,\|\, P_{u'}^{(1)}\big)
  = \frac{2n\zeta^2}{m}.
  \label{eq:neighbor-KL-rigorous}
\end{equation}
Now choose $\zeta^2 = \frac{\alpha_{\mathrm{KL}}m}{n}$
for some sufficiently small constant $\alpha_{\mathrm{KL}}>0$.
Then \eqref{eq:neighbor-KL-rigorous} gives
\[
  \mathrm{KL}(Q_u \,\|\, Q_{u'}) \;<\; 2\alpha_{\mathrm{KL}}
  =: C_{\mathrm{KL}},
\]
for all neighbors $u,u'$ with $D_H(u,u')=1$, where $C_{\mathrm{KL}}$ is a
fixed constant. This uniform bound on the KL divergences of
neighboring models is exactly the condition required in Assouad’s lemma
\citep{tsybakov2009introduction}.

By Assouad's lemma (Lemma~2.12 in \citet{tsybakov2009introduction}), there exists a constant $\alpha_{\mathrm{A}}>0$
(depending only on~$C_{\mathrm{KL}}$) such that, for the final policy $\widehat\pi$ of any learning algorithm,
\begin{equation}
  \sup_{u\in\mathcal{U}}
\mathbb{E}_u\big[D_H(u^{\widehat\pi},u)\big]
  \;\ge\;
  \alpha_{\mathrm{A}} \, m,
  \label{eq:hamming-lb}
\end{equation}
where $m := \sum_{k=1}^K k = \frac{K(K+1)}{2}.$

Since $0 \le D_H(u^{\widehat\pi},u)\le m$  for any fixed $u \in \mathcal{U}
  := \prod_{k=1}^K \{-1, +1\}^{k}$,
\begin{align*}
    \mathbb E_u\big[D_H(u^{\widehat\pi},u)\big]
 & =
  \mathbb E_u\!\left[D_H(u^{\widehat\pi},u)\,
  \mathbf 1\!\left\{D_H(u^{\widehat\pi},u) < \frac{\alpha_{\mathrm{A}}}{2}m\right\}\right] \nonumber\\
  &+
  \mathbb E_u\!\left[D_H(u^{\widehat\pi},u)\,
  \mathbf 1\!\left\{D_H(u^{\widehat\pi},u) \ge \frac{\alpha_{\mathrm{A}}}{2}m\right\}\right]\\
  &\le \frac{\alpha_{\mathrm{A}}}{2}m
  + m\,
  \mathbb P_u\!\left(
    D_H(u^{\widehat\pi},u) \ge \frac{\alpha_{\mathrm{A}}}{2}m
  \right).
\end{align*}

Taking the supremum over $u$ and using
$\sup_u \mathbb E_u[D_H(u^{\widehat\pi},u)] \ge \alpha_{\mathrm{A}} m$, we obtain
\[
  \alpha_{\mathrm{A}} m
  \;\le\;
  \frac{\alpha_{\mathrm{A}}}{2}m
  + m \sup_{u\in\mathcal U}
  \mathbb P_u\!\left(
    D_H(u^{\widehat\pi},u) \ge \frac{\alpha_{\mathrm{A}}}{2}m
  \right),
\]
so there exists some $u^\star\in\mathcal U$ such that
\begin{equation}
  \mathbb P_{u^\star}\!\left(
    D_H(u^{\widehat\pi},u^\star) \ge \frac{\alpha_{\mathrm{A}}}{2}m
  \right)
  \;\ge\; \frac{\alpha_{\mathrm{A}}}{2}.
  \label{eq:hp-hamming}
\end{equation}
Using \eqref{eq:value-vs-hamming} with $\pi=\widehat\pi$ and $u=u^\star$,
\[
  V_{M_{u^\star}}^{\pi^*} - V_{M_{u^\star}}^{\widehat\pi}
  \;\ge\;
  \gamma_{\mathrm{v}} \frac{\zeta}{\sqrt m}\,
  D_H(u^{\widehat\pi},u^\star),
\]
and on the event in \eqref{eq:hp-hamming} this yields
\[
  V_{M_{u^\star}}^{\pi^*} - V_{M_{u^\star}}^{\widehat\pi}
  \;\ge\;
  \gamma_{\mathrm{v}} \frac{\zeta}{\sqrt m} \cdot \frac{\alpha_{\mathrm{A}}}{2}m
  = \frac{\gamma_{\mathrm{v}} \alpha_{\mathrm{A}}}{2}\,\zeta \sqrt m.
\]
With $\zeta^2 = \alpha_{\mathrm{KL}} m/n$ and $m\asymp K^2$, there exists a
universal constant $C_{\text{lower}}
:= \alpha_{\mathrm{A}} \,\sqrt{\alpha_{\mathrm{KL}}}>0$ such that
\[
  \mathbb P_{M_{u^\star}}\!\left(
    V_{M_{u^\star}}^{\pi^*} - V_{M_{u^\star}}^{\widehat\pi}
    \;\ge\; C_{\text{lower}} \frac{K^2}{\sqrt n}
  \right)
  \;\ge\; \frac{\alpha_{\mathrm{A}}}{2}.
\]
If we pick $\alpha_{\mathrm{A}}=1/2$, this shows that for any algorithm
$\mathcal A$,
\[
  \sup_{M\in\mathcal M}
  \mathbb P_M\!\left(
    V_M^{\pi^*} - V_M^{\widehat\pi} \;\ge\; C_{\text{lower}} \frac{K^2}{\sqrt n}
  \right)
  \;\ge\; \frac{1}{4},
\]
and taking $\inf_{\mathcal A}$ over all algorithms yields the theorem.
\end{proof}
\cbend

\begin{proof}[Proof of Corollary \ref{Cor:Sample Complexity of Linear Model}]
By Theorem~\ref{thm: suboptimality for linear}, with probability at least \(1 - K\delta\), the suboptimality of the learned policy satisfies
\[
\text{Subopt}(\pi^{(T)}; M^*) 
\le \widetilde{\mathcal{O}}\left( \frac{K^3}{\sqrt{n}} + \frac{K^{5}}{n} \right)
+ \mathcal{O}\left( \frac{K}{\sqrt{T}} \right) 
+ K \sqrt{T} \cdot \mathcal{O}_p\left( \left( \frac{\log m_k}{m_k} \right)^{\alpha_1} \right),
\]
where the latter two terms are independent of the offline sample size \(n\), and vanish under appropriate choices of hyperparameters \(T\) and \(m_k \). Therefore, to achieve \(\text{Subopt}(\pi^{(T)}; M^*) \le \epsilon\), it suffices to set $n = \widetilde{\Theta}\left( \frac{K^6}{\epsilon^2} \right),$ where the $\widetilde{\Theta}$ hides a logarithmic dependence on $n$. Then the regret becomes
\[
\text{Subopt}(\pi^{(T)}; M^*) 
\le \widetilde{\mathcal{O}}\left( \frac{K^3}{\sqrt{n}} + \frac{K^{5}}{n} \right)
= \mathcal{O}\left( \epsilon + \frac{\epsilon^2}{K} \right)
= \mathcal{O}(\epsilon).
\]
\end{proof}

\subsection{Proof of Theorem \ref{thm:suboptimality-GP} in Section \ref{sec:GP-model}}
\label{proof sec: GP}

In this section, we build upon the discussion in Section \ref{subsec: GP}. We begin with two kernel assumptions and four auxiliary lemmas that deliver the variance bounds needed for the proof.

\begin{SAssumption}
\label{assumption: GP kernel}
The kernel function \( h(\cdot,\cdot) \) is continuous and bounded on \( \mathcal{X} \times \mathcal{X} \).
\end{SAssumption}

This assumption imposes standard regularity conditions on the kernel function \( h \), namely continuity and boundedness over the domain \( \mathcal{X} \times \mathcal{X} \). These conditions are commonly assumed in the Gaussian process and kernel methods literature to ensure the well-posedness of the associated integral operator and to facilitate theoretical analysis \citep{kanagawa2018gaussian, steinwart2009optimal, chang2021mitigating}.

\begin{SAssumption}
\label{assumption:eigenvalue decay}
The eigenvalues \( \{ \mu_j \}_{j=1}^{\infty} \) of the integral operator defined in \eqref{eq: lemma Mercer eq 1} satisfy the decay condition \( \mu_j = \mathcal{O}(j^{-\omega}) \) for some \( \omega > 5 \).
\end{SAssumption}

Assumption~\ref{assumption:eigenvalue decay} guarantees that the eigenvalue sequence of the integral operator defined in \eqref{eq: lemma Mercer eq 1} is summable, i.e., \( \sum_{j=1}^{\infty} \mu_j < \infty \). This summability condition implies a sufficiently fast decay rate, which ensures the regularization properties and numerical stability of kernel-based estimators. A wide range of commonly used kernels satisfy this condition under mild parameter choices. For instance, the Matérn kernel exhibits eigenvalue decay of the form \( \mu_j \sim j^{-(2\nu + 1)} \), which meets the assumption when the smoothness parameter \( \nu > 2 \). The Rational Quadratic kernel satisfies the condition if its shape parameter \( \beta_1 > 5/2 \), since its eigenvalues decay as \( \mu_j \sim j^{-2\beta_1} \). Similarly, the Sobolev kernel, characterized by \( \mu_j \sim j^{-2s} \), meets the requirement when the smoothness parameter \( s > 5/2 \). The Periodic kernel, with eigenvalue decay \( \mu_j \sim j^{-p} \), satisfies the assumption when the periodicity parameter \( p > 5 \). These examples illustrate that Assumption~\ref{assumption:eigenvalue decay} holds for a broad class of kernels widely used in Gaussian Process modeling.

\begin{lemma}
\label{lemma: Mercer}
Suppose Assumption~\ref{assumption: GP kernel} holds, where \( h: \mathcal{X} \times \mathcal{X} \to \mathbb{R} \) is a continuous, symmetric, and positive semi-definite kernel, and \( \rho \) is a Borel probability measure on \( \mathcal{X} \). Then, there exists a non-increasing sequence of non-negative eigenvalues \( \{ \mu_j \}_{j=1}^\infty \) and a corresponding sequence of orthonormal eigenfunctions \( \{ \psi_j \}_{j=1}^\infty \) forming a complete orthonormal basis for the Hilbert space \( L^2(\mathcal{X}, \rho) \), such that:
\begin{equation}
\label{eq: lemma Mercer eq 1}
\int_\mathcal{X} h(x, x') \psi_j(x') \, \rho(x') \, dx' = \mu_j \psi_j(x), \quad \text{for all } j \ge 1.
\end{equation}
Moreover, the kernel \( h \) admits the following spectral expansion:
\begin{equation}
\label{eq: lemma Mercer eq 2}
h(x, x') = \sum_{j=1}^\infty \mu_j \psi_j(x) \psi_j(x'),
\end{equation}
with convergence in \( L^2(\mathcal{X} \times \mathcal{X}, \rho \otimes \rho) \).
\end{lemma}

\begin{lemma}[Expected Posterior GP Variance]
\label{lemma: GP posterior expected variance}
Suppose Assumptions~\ref{assumption: GP kernel} and~\ref{assumption:eigenvalue decay} hold. Let \( \{\mu_j\}_{j=1}^{\infty} \) and \( \{\psi_j\}_{j=1}^{\infty} \) denote the eigenvalues and corresponding eigenfunctions of the kernel \( h \), as given in Lemma~\ref{lemma: Mercer}. Let \( \widehat{h}(x,x) \) denote the posterior variance of a Gaussian process at point \( x \in \mathcal{X} \), after observing \( n \) samples drawn according to the distribution \( \rho \). Then the expected posterior variance satisfies
\[
\mathbb{E}_{x \sim \rho}\left[\widehat{h}(x, x)\right] = \mathcal{O}\left(n^{\frac{1}{\omega} - 1}\right).
\]
\end{lemma}
\begin{lemma}
\label{lemma: GP posterior variance}
Suppose Assumptions  \ref{assumption: GP kernel} and \ref{assumption:eigenvalue decay} hold. Let $\{\mu_j\}_{j=1} ^{\infty}$, $\{\psi_j\}_{j=1} ^{\infty}$ be the eigenvalues and corresponding eigenfunctions from Lemma \ref{lemma: Mercer}. 
For any fixed $x\in \mathcal{X}$, with high probability $1-\delta$, it satisfies that
\begin{equation*}
\widehat{h}(x,x) \le M,
\end{equation*}
where $M=\sup_{x \in \mathcal{X}} h(x,x)$.

\end{lemma}

\begin{lemma}
\label{Lemma:beta_n bound GP}
Recall that \( \beta_{k,n} \) is defined as
\[
\beta_{k,n} = \sqrt{d_{k+1}^{(s)} \left(2 + 150 \log^3\left( \frac{d_{k+1}^{(s)} n}{\delta} \right) \log\left( \det\left( I_n + \frac{1}{\sigma^2} H_{n,n} \right) \right) \right)},
\]
where \( H_{n,n} \in \mathbb{R}^{n \times n} \) is the kernel matrix with entries \( h(x_i, x_j) \), and \( d_{k+1}^{(s)}\) is the dimension of the state space $\mathcal{S}_{k+1}$. Then with probability at least \( 1 - \delta \), the following upper bound holds:
\[\beta_n  = \mathcal{O}\left( (\log n)^2 \cdot n^{1/\omega} \right),\]
where \( \omega \) is the kernel eigenvalue decay exponent defined in Assumption~\ref{assumption:eigenvalue decay}.
\end{lemma}

The proofs of Lemma \ref{lemma: Mercer} - Lemma \ref{Lemma:beta_n bound GP} are provided in Section \ref{pf:lemma thm3}.

\begin{proof}[Proof of Theorem~\ref{thm:suboptimality-GP}]
By Theorem~\ref{thm: combined main thm}, with probability at least \( 1 - K\delta \), the suboptimality of the learned policy \( \pi^{(T)} \) is bounded by
\begin{equation*}
\text{Subopt}(\pi^{(T)}; M^*) \le \sum_{k=1}^K b_k \, \mathbb{E}_{(h_k, a_k) \sim d^{\pi^\dagger}_{P^*}} \Gamma_k(h_k, a_k) + \mathcal{O}\left( \frac{K}{\sqrt{T}} \right) + K \cdot \mathcal{O}_P\left( T^{1/2} G_k + T^{3/2} G_k^2 \right).
\end{equation*}

From Lemma~\ref{lemma: GP posterior variance}, it holds that \( \widehat{h}(x,x) \le M \) for any \( x \in \mathcal{X} \). Recall from equation~\eqref{eq: GP, eq 3} that $\Gamma_k(x) = \frac{\beta_{n,k}}{\sigma} \sqrt{\widehat{h}(x,x)}.$

By combining Lemmas ~\ref{lemma: GP posterior variance} and ~\ref{Lemma:beta_n bound GP}, it follows that with probability at least \( 1 - \delta \),
\begin{equation*}
\Gamma_k(x) \lesssim (\log n)^2 n^{\frac{1}{\omega}} \cdot \sqrt{\widehat{h}(x,x)} \lesssim \widetilde{\mathcal{O}}\left(n^{\frac{1}{\omega}}\right),
\end{equation*}
where \( \widetilde{\mathcal{O}} \) hides logarithmic factors in \( n \).

Since $\widetilde{c}_k = \sum_{j=k}^K \|\overline{r}_j\|_\infty = \mathcal{O}(K)$, substituting into the definition of \( b_k \), we obtain
\[
b_k = 2\widetilde{c}_k + \sum_{j = k+1}^K \widetilde{c}_j \|\Gamma_j\|_\infty \lesssim \widetilde{\mathcal{O}}\left(K + K^2 n^{\frac{1}{\omega}}\right).
\]

Next, by Lemma~\ref{lemma: GP posterior expected variance}, we have $\mathbb{E}_{x \sim \rho^k} \widehat{h}(x,x) \lesssim n^{\frac{1}{\omega} - 1},$
which implies
\[
\mathbb{E}_{(h_k, a_k) \sim d^{\pi^\dagger}_{P^*}} \Gamma_k(h_k, a_k) \lesssim \widetilde{\mathcal{O}}\left(n^{\frac{3}{2\omega} - \frac{1}{2}}\right).
\]

Combining the bounds above, the overall suboptimality is bounded as
\begin{align*}
\text{Subopt}(\pi^{(T)}; M^*) 
&\le \sum_{k=1}^K b_k \, \mathbb{E}_{(h_k, a_k) \sim d^{\pi^\dagger}_{P^*}} \Gamma_k(h_k, a_k) + \mathcal{O}\left( \frac{K}{\sqrt{T}} \right) + K \cdot \mathcal{O}_P\left( T^{1/2} G_k + T^{3/2} G_k^2 \right) \\
&\le \mathcal{O}\left(K^2 (\log n)^2 n^{\frac{3}{2\omega} - \frac{1}{2}} + K^3 (\log n)^4 n^{\frac{5}{2\omega} - \frac{1}{2}}\right) + \mathcal{O}\left( \frac{K}{\sqrt{T}} \right) \\
&\quad + K \sqrt{T} \cdot \mathcal{O}_P\left( \left( \frac{\log m_k}{m_k} \right)^{\alpha_1} \right),
\end{align*}
which concludes the proof.
\end{proof}

\begin{proof}[Proof of Corollary~\ref{cor:sample-complexity-GP}]
By Theorem~\ref{thm:suboptimality-GP}, with probability at least \(1 - K\delta\), the suboptimality of the learned policy satisfies
\[
\text{Subopt}(\pi^{(T)}; M^*) 
\le 
\widetilde{\mathcal{O}}\left( K^2  n^{\frac{3}{2\omega}-\frac{1}{2}}+K^3 n^{\frac{5}{2\omega}-\frac{1}{2}} \right) 
+ \mathcal{O} \left( \frac{K}{\sqrt{T}} \right)
+ K \sqrt{T} \cdot \mathcal{O}_p \left( \left( \frac{\log m_k}{m_k} \right)^{\alpha_1} \right),
\]
where the last two terms are independent of the offline sample size \(n\) and vanish under appropriate choices of hyperparameters \(T\) and \(m_k \). Therefore, to achieve \(\text{Subopt}(\pi^{(T)}; M^*) \le \epsilon\), it suffices to set $n =  \widetilde{\Theta} \left(\frac{K^3}{\epsilon}\right)^{\frac{2\omega}{\omega-5}},$
where the notation \(\widetilde{\Theta}\) hides logarithmic dependencies on \(n\). Then this regret becomes
\[
\text{Subopt}(\pi^{(T)}; M^*) 
= \mathcal{O}(\epsilon).
\]
\end{proof}

\section{Proofs of Auxiliary Lemmas}
\label{sec:lemmasproof}

\subsection{Supporting Lemmas and Their Proofs of Proposition \ref{thm: uncertainty quantifier KNR}}
\label{pf:lemma prop 1}
\begin{definition}[$\sigma$-Sub-Gaussian Condition]
\label{definition: sigma-sub-gaussian}
\begin{itemize}
    \item[(i)] A real-valued random variable \( X \) is said to be \(\sigma\)-sub-Gaussian for some \(\sigma > 0\) if it satisfies the inequality
    \[
    \mathbb{E}\left[\exp(\lambda X)\right] \le \exp\left(\frac{\lambda^2 \sigma^2}{2}\right), \quad \text{for all } \lambda \in \mathbb{R}.
    \]
    \item[(ii)] A random vector \( X \in \mathbb{R}^d \) is \(\sigma\)-sub-Gaussian if for every unit vector \( w \in \mathbb{R}^d \) (i.e., \( \|w\|_2 = 1 \)), the one-dimensional projection \( w^\top X \) is \(\sigma\)-sub-Gaussian in the sense of (i).
\end{itemize}
\end{definition}

The scalar case in Definition~\ref{definition: sigma-sub-gaussian} (i) follows the standard formulation as in \citet{Abbasi-Yadkori2011}, while the extension to random vectors in Definition~\ref{definition: sigma-sub-gaussian} (ii) adopts the projective sub-Gaussian criterion used in \citet{kakade2020information}.

\begin{lemma}
\label{lemma: self normalized}
Let $\{\mathcal{F}_i \}_{i=0}^{\infty}$ be a filtration and $\{\varepsilon_i \}_{i=0}^{\infty}$ be an $\mathbb{R}^d$-valued stochastic process such that $\varepsilon_i$ is $\mathcal{F}_i$-measurable, $\mathbb{E}[\varepsilon_i| \mathcal{F}_i] =0$, and $\varepsilon_i| \mathcal{F}_i$ is $\sigma$-sub-Gaussian.
Meanwhile, let $\{\phi_i \}_{i=0}^{\infty}$ be an $\mathbb{R}^m$-valued, $\mathcal{F}_i$-measurable stochastic process. Let \(\Lambda_0 \in \mathbb{R}^{m \times m}\) be a deterministic positive-definite real matrix, and define $\Lambda_t := \Lambda_0 + \sum_{i=1}^t \phi_i \phi_i^T$. Then, with probability at least $1-\delta$, the following bound holds for all $t\ge 1$:
\[
\left\| \left( \sum_{i=1}^t \varepsilon_i \phi_i^\top \right) \Lambda_t^{-1/2} \right\|_2^2 
\le 8\sigma^2 d \log(5) + 8\sigma^2 \log \left( \frac{\det(\Lambda_t)^{1/2}}{\det(\Lambda_0)^{1/2} \, \delta} \right),
\]
where $\|\cdot\|_2$ denotes the spectral norm of a $d\times m$ matrix.
\end{lemma}

\begin{proof}[Proof of Lemma \ref{lemma: self normalized}]
See Lemma C.4 of \cite{kakade2020information} for a detailed proof.
\end{proof}



\begin{proof}[Proof of Lemma~\ref{lemma: bounded rv is subgaussian}]
Let \( w \in \mathbb{R}^d \) be any unit vector (i.e., \( \|w\|_2 = 1 \)), and define the scalar random variable \( X = w^\top Y \). Then \( \mathbb{E}[X] = w^\top \mathbb{E}[Y] = 0 \), and by the Cauchy–Schwarz inequality,
\[
|X| = |w^\top Y| \le \|w\|_2 \|Y\|_2 \le \sigma.
\]
Thus, \( X \) is a zero-mean, bounded random variable supported on \( [-\sigma, \sigma] \). We now verify that \( X \) is \(\sigma\)-sub-Gaussian. By Definition~\ref{definition: sigma-sub-gaussian}(i), we must show that for all \( \lambda \in \mathbb{R} \), $\mathbb{E}[e^{\lambda X}] \le \exp\left(\frac{\lambda^2 \sigma^2}{2}\right).$ Since \( \lambda X \in [-|\lambda| \sigma, |\lambda| \sigma] \), we apply the convexity of the exponential function:
\[
e^{\lambda X} \le \frac{e^{|\lambda| \sigma} - e^{-|\lambda| \sigma}}{2 |\lambda| \sigma}(X + |\lambda| \sigma) + e^{-|\lambda| \sigma}.
\]
Taking expectations on both sides and using \( \mathbb{E}[X] = 0 \), we obtain:
\[
\mathbb{E}[e^{\lambda X}] \le \frac{1}{2} e^{|\lambda| \sigma} + \frac{1}{2} e^{-|\lambda| \sigma}.
\]
Letting \( t = |\lambda| \sigma \), we finish the proof by using the inequality:
\[
\frac{1}{2}(e^t + e^{-t}) \le \exp\left( \frac{t^2}{2} \right), \quad \forall t \ge 0.
\]
\end{proof}


\begin{proof}[Proof of Lemma \ref{lemma: W-W* in KNR}]
Since $s_i= W^* \phi(h_i,a_i)+ \varepsilon_i$, we then can plug it into the construction of $\widehat{W}$:
\begin{equation*}
\begin{split}
\widehat{W} &= \left( \sum_{i=1}^n s_i \phi_i^T \right) \left( \sum_{i=1}^n \phi_i \phi_i^T + \lambda I \right)^{-1}\\
&= \left( \sum_{i=1}^n W^* \phi(h_i,a_i) \phi_i^T + \sum_{i=1}^n \varepsilon_i \phi_i^T \right) \left( \sum_{i=1}^n \phi_i \phi_i^T + \lambda I \right)^{-1}\\
&= \left( W^* \left(\sum_{i=1}^n \phi(h_i,a_i) \phi_i^T +\lambda I \right) - \lambda W^*+ \sum_{i=1}^n \varepsilon_i \phi_i^T \right) \left( \sum_{i=1}^n \phi_i \phi_i^T + \lambda I \right)^{-1} \\
&= W^* - \lambda W^* \left( \sum_{i=1}^n \phi_i \phi_i^T + \lambda I \right)^{-1} + \sum_{i=1}^n \varepsilon_i \phi_i^T \left( \sum_{i=1}^n \phi_i \phi_i^T + \lambda I \right)^{-1}.
\end{split}
\end{equation*}
For simplicity, denote $\Lambda_n = \sum_{i=1}^n \phi_i \phi_i^T + \lambda I$, then
\begin{equation*}
\widehat{W}-W^* = - \lambda W^* \Lambda_n^{-1} + \sum_{i=1}^n \varepsilon_i \phi_i^T \Lambda_n^{-1}.
\end{equation*}
Multiply $\Lambda_n^{1/2}$ on the right, and take the spectral norm on both sides:
\begin{equation}
\label{eq: spec norm both sides}
\left\| (\widehat{W}-W^*) \Lambda_n^{1/2} \right\|_2 \le \left\| \lambda W^* \Lambda_n^{-1/2} \right\|_2 + \left\| \sum_{i=1}^n \varepsilon_i \phi_i^T \Lambda_n^{-1/2} \right\|_2.
\end{equation}
For the first term, note that $\Lambda_n = \sum_{i=1}^n \phi_i \phi_i^T + \lambda I \succeq \lambda I$, so the smallest eigenvalue of $\Lambda_n$ satisfies $\lambda_{min} (\Lambda_n) \ge \lambda >0$. We have $\lambda_{min} (\Lambda_n^{1/2}) \ge \sqrt{\lambda}$, $\lambda_{max} (\Lambda_n^{-1/2}) \le \frac{1}{ \sqrt{\lambda}}$. It suggests that $\|\Lambda_n^{-1/2}\|_2 \le \frac{1}{ \sqrt{\lambda}}$. By the sub-multiplicative property of matrix spectral norm,
\begin{equation}
\label{eq: matrix norm proof}
\left\| \lambda W^* \Lambda_n^{-1/2} \right\|_2 \le \lambda \|W^*\|_2 \|\Lambda_n^{-1/2}\|_2 \le \sqrt{\lambda} \|W^*\|_2.
\end{equation}
For the second term, apply Lemma \ref{lemma: self normalized}: with probability at least $1-\delta$,
\begin{equation}
\label{eq: sec term lemma 7}
\left\| \sum_{i=1}^n \varepsilon_i \phi_i^T \Lambda_n^{-1/2} \right\|_2^2 \le 8\sigma^2 d_s \log(5) + 8\sigma^2 \log \left( \frac{\det (\Lambda_n) ^{1/2} / \det (\lambda I) ^{1/2}}{\delta} \right).
\end{equation}
Merge the inequalities \eqref{eq: spec norm both sides}, \eqref{eq: matrix norm proof} and \eqref{eq: sec term lemma 7} above:
\begin{equation*}
\left\| (\widehat{W}-W^*) \Lambda_n^{1/2} \right\|_2 \le \sqrt{\lambda} \|W^*\|_2 + \sqrt{8\sigma^2 d_s \log(5) + 8\sigma^2 \log \left( \frac{\det (\Lambda_n) ^{1/2} / \det (\lambda I) ^{1/2}}{\delta} \right)}.
\end{equation*}
\end{proof}

\subsection{Supporting Lemmas' Proofs of Theorem \ref{thm: final regret real M}}
\label{pf:lemma them S1}

\begin{proof}[Proof of Lemma~\ref{lemma: simulation}]
We prove part (i) and part (ii) follows directly by taking expectation over $a_k \sim \pi_k(\cdot \mid h_k)$.

By definition of $Q_k^\pi$ under dynamics $P$,
\begin{equation}
\label{eq:proof-simulation-1}
Q_k^\pi(h_k, a_k) = \mathbb{E}_\pi\left[\sum_{t=k}^K R_t \mid H_k = h_k, A_k = a_k\right] = \sum_{t=k}^K \mathbb{E}_{(h_t,a_t)\sim d_{t,P}^\pi(\cdot,\cdot \mid H_k = h_k, A_k = a_k)} \left[r_t(h_t, a_t)\right].
\end{equation}

From the Bellman equation under model $\widehat{M}$:
\begin{equation}
\label{eq:proof-simulation-2}
\widehat{Q}_t^\pi(h_t, a_t) = r_t(h_t, a_t) + \mathbb{E}_{h_{t+1} \sim \widehat{P}_t(\cdot \mid h_t, a_t)} \left[ \widehat{V}_{t+1}^\pi(h_{t+1}) \right].
\end{equation}

Substituting the reward in~\eqref{eq:proof-simulation-1} using~\eqref{eq:proof-simulation-2}, we get:
\[
Q_k^\pi(h_k, a_k) = \sum_{t=k}^K \mathbb{E}_{(h_t,a_t)\sim d_{t,P}^\pi(\cdot,\cdot \mid H_k = h_k, A_k = a_k)} \left[
\widehat{Q}_t^\pi(h_t, a_t) - \mathbb{E}_{h'\sim \widehat{P}_t(\cdot \mid h_t, a_t)}\left[ \widehat{V}_{t+1}^\pi(h') \right]
\right].
\]

For $t = K$, the second term vanishes since $\widehat{V}_{K+1}^\pi \equiv 0$. Therefore:
\begin{equation}
\label{eq:proof-simulation-3}
\begin{split}
Q_k^\pi(h_k, a_k) =\ & \widehat{Q}_k^\pi(h_k, a_k) + \sum_{t=k+1}^K \mathbb{E}_{(h_t, a_t) \sim d_{t,P}^\pi(\cdot,\cdot \mid H_k = h_k, A_k = a_k)} \left[ \widehat{Q}_t^\pi(h_t, a_t) \right] \\
& - \sum_{t=k}^{K-1} \mathbb{E}_{(h_t,a_t) \sim d_{t,P}^\pi(\cdot,\cdot \mid H_k = h_k, A_k = a_k)} \left[ \mathbb{E}_{h' \sim \widehat{P}_t(\cdot \mid h_t, a_t)} \left[ \widehat{V}_{t+1}^\pi(h') \right] \right].
\end{split}
\end{equation}

Next, we express the sum in the first line as:
\begin{equation}
\label{eq:proof-simulation-4}
\sum_{t=k+1}^K \mathbb{E}_{(h_t,a_t)\sim d_{t,P}^\pi(\cdot,\cdot \mid H_k = h_k, A_k = a_k)} \left[ \widehat{Q}_t^\pi(h_t, a_t) \right]
= \sum_{t=k}^{K-1} \mathbb{E}_{(h_t, a_t),\ h'\sim P_t,\ a'\sim \pi_{t+1}(h')} \left[ \widehat{Q}_{t+1}^\pi(h', a') \right].
\end{equation}

Since \( a' \sim \pi_{t+1}(h') \), the inner expectation 
\( \mathbb{E}_{a'} \left[ \widehat{Q}_{t+1}^\pi(h', a') \right] \) 
evaluates to \( \widehat{V}_{t+1}^\pi(h') \) by definition of the value function.
Combining \eqref{eq:proof-simulation-3} and \eqref{eq:proof-simulation-4}, and regrouping, yields:
\[
\begin{aligned}
Q_k^\pi(h_k, a_k) - \widehat{Q}_k^\pi(h_k, a_k)
&= \sum_{t=k}^{K-1} \mathbb{E}_{(h_t, a_t) \sim d_{t, P}^\pi(\cdot,\cdot \mid h_k, a_k)} \Big[
\mathbb{E}_{h' \sim P_t(\cdot \mid h_t, a_t)} \widehat{V}_{t+1}^\pi(h') \\
&\quad - \mathbb{E}_{h' \sim \widehat{P}_t(\cdot \mid h_t, a_t)} \widehat{V}_{t+1}^\pi(h')
\Big],
\end{aligned}
\]

which completes the proof.
\end{proof}


\begin{proof}[Proof of Lemma~\ref{lemma: r - rhat}]
By the definition of $r_k$ and $\widehat{r}_k$:
\begin{align*}
\left| r_k(h_k, a_k) - \widehat{r}_k(h_k, a_k) \right| 
&= \left| \mathbb{E}_{s'\sim P^*_k(\cdot \mid h_k, a_k)} \left[ \overline{r}_k(h_k, a_k, s') \right] - \mathbb{E}_{s'\sim \widehat{P}_k(\cdot \mid h_k, a_k)} \left[ \overline{r}_k(h_k, a_k, s') \right] \right| \\
&\le \| \overline{r}_k \|_\infty \cdot \left\| P^*_k(\cdot \mid h_k, a_k) - \widehat{P}_k(\cdot \mid h_k, a_k) \right\|_1.
\end{align*}

By Assumption~\ref{assumption: uncertainty quantifier}, the $L_1$ distance between the true and estimated transition distributions is bounded by $\Gamma_k(h_k, a_k)$ with probability at least $1 - \delta$. We show our desired result.
\end{proof}

\subsection{Supporting Lemmas' Proofs of Theorem \ref{thm: M^tilde + general e}}
\label{pf:lemma them S2}

\begin{proof}[Proof of Lemma~\ref{lemma: perf diff}]
We first prove the second equality. By the definition of the advantage function:
\begin{align*}
\mathbb{E}_{a_t \sim \pi_t(\cdot \mid h_t)} A_t^{\widehat{\pi}}(h_t, a_t)
&= \mathbb{E}_{a_t \sim \pi_t(\cdot \mid h_t)} \left[ Q_t^{\widehat{\pi}}(h_t, a_t) \right] - V_t^{\widehat{\pi}}(h_t) \\
&= \mathbb{E}_{a_t \sim \pi_t(\cdot \mid h_t)} \left[ Q_t^{\widehat{\pi}}(h_t, a_t) \right] - \mathbb{E}_{a_t \sim \widehat{\pi}_t(\cdot \mid h_t)} \left[ Q_t^{\widehat{\pi}}(h_t, a_t) \right] \\
&\stackrel{(i)}{=} \left\langle Q_t^{\widehat{\pi}}(h_t, \cdot),\ \pi_t(\cdot \mid h_t) - \widehat{\pi}_t(\cdot \mid h_t) \right\rangle,
\end{align*}
where $(i)$ is because the action space is discrete.
We now prove the first equality. Let $\tau$ denote the trajectory $(H_k, A_k, \dots, H_K, A_K)$ generated under policy $\pi$ and dynamics $P$. Then:
\[
V_k^\pi(h_k) = \mathbb{E}_\pi \left[ \sum_{t=k}^K R_t \mid H_k = h_k \right]
= \mathbb{E}_{\tau \sim d_P^\pi(\cdot \mid H_k = h_k)} \left[ \sum_{t=k}^K r_t(H_t, A_t) \right].
\]

Adding and subtracting $V_t^{\widehat{\pi}}(H_t)$ inside the summation gives:
\begin{equation}
\label{eq:proof-pdiff-1}
\begin{aligned}
V_k^\pi(h_k)
&= \mathbb{E}_{\tau \sim d_P^\pi(\cdot \mid H_k = h_k)} \sum_{t=k}^K \left[ r_t(H_t, A_t) - V_t^{\widehat{\pi}}(H_t) + V_t^{\widehat{\pi}}(H_t) \right] \\
&= \mathbb{E}_{\tau \sim d_P^\pi(\cdot \mid H_k = h_k)} \sum_{t=k}^K \left[ r_t(H_t, A_t) - V_t^{\widehat{\pi}}(H_t) \right]
+ \mathbb{E}_{\tau \sim d_P^\pi(\cdot \mid H_k = h_k)} \sum_{t=k}^K V_t^{\widehat{\pi}}(H_t).
\end{aligned}
\end{equation}

Since $V_k^{\widehat{\pi}}(H_k) = V_k^{\widehat{\pi}}(h_k)$ deterministically, we rewrite the second sum as:
\begin{equation}
\label{eq:proof-pdiff-2}
\begin{aligned}
\mathbb{E}_{\tau \sim d_P^\pi(\cdot \mid H_k = h_k)} \sum_{t=k}^K V_t^{\widehat{\pi}}(H_t)
&= V_k^{\widehat{\pi}}(h_k) + \mathbb{E}_{\tau \sim d_P^\pi(\cdot \mid H_k = h_k)} \sum_{t=k+1}^K V_t^{\widehat{\pi}}(H_t) \\
&= V_k^{\widehat{\pi}}(h_k) + \mathbb{E}_{\tau \sim d_P^\pi(\cdot \mid H_k = h_k)} \sum_{t=k}^K V_{t+1}^{\widehat{\pi}}(H_{t+1}),
\end{aligned}
\end{equation}
where we define $V_{K+1}^{\widehat{\pi}} \equiv 0$.

Substituting \eqref{eq:proof-pdiff-2} into \eqref{eq:proof-pdiff-1}, we obtain:
\[
V_k^\pi(h_k) - V_k^{\widehat{\pi}}(h_k)
= \mathbb{E}_{\tau \sim d_P^\pi(\cdot \mid H_k = h_k)} \sum_{t=k}^K \left[ r_t(H_t, A_t) + V_{t+1}^{\widehat{\pi}}(H_{t+1}) - V_t^{\widehat{\pi}}(H_t) \right].
\]

Since $H_{t+1} \sim P_t(\cdot \mid H_t, A_t)$ under model $P$, we can write:
\begin{equation}
\label{eq:proof-pdiff-3}
\begin{aligned}
&\quad \mathbb{E}_{\tau \sim d_P^\pi(\cdot \mid H_k = h_k)} \left[ r_t(H_t, A_t) + V_{t+1}^{\widehat{\pi}}(H_{t+1}) \right]\\
&= \mathbb{E}_{(H_t, A_t) \sim d_t^\pi(\cdot,\cdot \mid H_k = h_k)} \left[
r_t(H_t, A_t) + \mathbb{E}_{H_{t+1} \sim P_t(\cdot \mid H_t, A_t)} V_{t+1}^{\widehat{\pi}}(H_{t+1})
\right] \\
&= \mathbb{E}_{(H_t, A_t) \sim d_t^\pi(\cdot,\cdot \mid H_k = h_k)} \left[ Q_t^{\widehat{\pi}}(H_t, A_t) \right].
\end{aligned}
\end{equation}

Putting it all together:
\[
\begin{aligned}
V_k^\pi(h_k) - V_k^{\widehat{\pi}}(h_k)
&= \sum_{t=k}^K \mathbb{E}_{(H_t, A_t) \sim d_t^\pi(\cdot,\cdot \mid H_k = h_k)} \left[ Q_t^{\widehat{\pi}}(H_t, A_t) - V_t^{\widehat{\pi}}(H_t) \right] \\
&= \sum_{t=k}^K \mathbb{E}_{(H_t, A_t) \sim d_t^\pi(\cdot,\cdot \mid H_k = h_k)} \left[ A_t^{\widehat{\pi}}(H_t, A_t) \right],
\end{aligned}
\]
which completes the proof.
\end{proof}

\begin{proof}[Proof of Lemma~\ref{lemma: 3 KL terms}]
By definition of KL divergence and the formulation of softmax policies, we expand:
\begin{align*}
&\quad \operatorname{KL}(p,\pi_k(h_k, \theta_k)) - \operatorname{KL}(p,\pi_k(h_k, \theta_k')) - \operatorname{KL}(\pi_k(h_k, \theta_k'), \pi_k(h_k, \theta_k)) \\
&= \mathbb{E}_{a \sim p} \left[ \log \frac{p(a)}{\pi_k(a \mid h_k, \theta_k)} \right]
- \mathbb{E}_{a \sim p} \left[ \log \frac{p(a)}{\pi_k(a \mid h_k, \theta_k')} \right]
- \mathbb{E}_{a \sim \pi_k(\cdot \mid h_k, \theta_k')} \left[ \log \frac{\pi_k(a \mid h_k, \theta_k')}{\pi_k(a \mid h_k, \theta_k)} \right] \\
&= \left\langle \log \frac{\pi_k(\cdot \mid h_k, \theta_k')}{\pi_k(\cdot \mid h_k, \theta_k)},\ p(\cdot) - \pi_k(\cdot \mid h_k, \theta_k') \right\rangle.
\end{align*}

Recall the parameterization of the softmax policy: $\log \pi_k(a \mid h_k, \theta) = f_k(\theta, h_k, a) - \log Z_k(h_k, \theta),$
where the normalization constant is $Z_k(h_k, \theta) := \sum_{a' \in \mathcal{A}_k} \exp(f_k(\theta, h_k, a')).$

Substituting into the inner product:
\begin{align*}
&\quad \left\langle \log \frac{\pi_k(\cdot \mid h_k, \theta_k')}{\pi_k(\cdot \mid h_k, \theta_k)},\ p(\cdot) - \pi_k(\cdot \mid h_k, \theta_k') \right\rangle \\
&= \left\langle f_k(\theta_k', h_k, \cdot) - \log Z_k(h_k, \theta_k') - f_k(\theta_k, h_k, \cdot) + \log Z_k(h_k, \theta_k),\ p(\cdot) - \pi_k(\cdot \mid h_k, \theta_k') \right\rangle.
\end{align*}

Since $\log Z_k(h_k, \theta)$ is constant in $a$, and both $p(\cdot)$ and $\pi_k(\cdot \mid h_k, \theta_k')$ are
probability distributions over $\mathcal{A}_k$, we have:
\[
\left\langle \log Z_k(h_k, \theta),\ p(\cdot) - \pi_k(\cdot \mid h_k, \theta_k') \right\rangle
= \log Z_k(h_k, \theta) \cdot \sum_{a \in \mathcal{A}_k} \left( p(a) - \pi_k(a \mid h_k, \theta_k') \right) = 0.
\]

Thus, we simplify the expression to: $\left\langle f_k(\theta_k', h_k, \cdot) - f_k(\theta_k, h_k, \cdot),\ p(\cdot) - \pi_k(\cdot \mid h_k, \theta_k') \right\rangle.$
\end{proof}


\begin{proof}[Proof of Lemma~\ref{lemma: <Q,pi_t+1 - pi_t>}]
We analyze the policy update from $\theta_k^{(t)}$ to $\theta_k^{(t+1)}$. Applying Lemma~\ref{lemma: 3 KL terms} with $\theta = \theta_k^{(t)}$, $\theta' = \theta_k^{(t+1)}$, for any $p \in \Delta(\mathcal{A}_k)$, we obtain:
\begin{equation}
\label{eq:lemma-kl}
\begin{aligned}
&\quad \operatorname{KL}(p,\pi_k(h_k, \theta_k^{(t)})) - \operatorname{KL}(p,\pi_k(h_k, \theta_k^{(t+1)})) - \operatorname{KL}(\pi_k(h_k, \theta_k^{(t+1)}), \pi_k(h_k, \theta_k^{(t)})) \\
&= \left\langle f_k(\theta_k^{(t+1)}, h_k,\cdot) - f_k(\theta_k^{(t)}, h_k,\cdot),\ p(\cdot) - \pi_k(\cdot|h_k, \theta_k^{(t+1)}) \right\rangle.
\end{aligned}
\end{equation}

Now take $p = \pi_k(\cdot|h_k, \theta_k^{(t)})$ in~\eqref{eq:lemma-kl}, which yields:
\begin{equation}
\label{eq:kl-diff}
- \operatorname{KL}(\pi_k^{(t)}, \pi_k^{(t+1)}) - \operatorname{KL}(\pi_k^{(t+1)}, \pi_k^{(t)})= \left\langle f_k(\theta_k^{(t+1)}, h_k,\cdot) - f_k(\theta_k^{(t)}, h_k,\cdot),\ \pi_k^{(t)} - \pi_k^{(t+1)} \right\rangle,
\end{equation}
where we abbreviate $\pi_k^{(t)} := \pi_k(\cdot|h_k, \theta_k^{(t)})$ and $\pi_k^{(t+1)} := \pi_k(\cdot|h_k, \theta_k^{(t+1)})$.

From the policy update rule~\eqref{eq: e_t}, we know that for each $a \in \mathcal{A}_k$,
\begin{equation}
\label{eq:f-diff}
f_k(\theta_k^{(t+1)}, h_k, a) - f_k(\theta_k^{(t)}, h_k, a) = e_k^{(t)}(h_k, a) + \eta_k^{(t)} Q^{\pi^{(t)}}_{k,\widetilde{M}}(h_k, a).
\end{equation}

Substituting~\eqref{eq:f-diff} into~\eqref{eq:kl-diff}, we get:
\begin{equation}
\label{eq: inner-split}
    \operatorname{KL}(\pi_k^{(t)}, \pi_k^{(t+1)}) + \operatorname{KL}(\pi_k^{(t+1)}, \pi_k^{(t)})= \left\langle e_k^{(t)}(h_k, \cdot),\ \pi_k^{(t+1)} - \pi_k^{(t)} \right\rangle + \eta_k^{(t)} \left\langle Q^{\pi^{(t)}}_{k,\widetilde{M}}(h_k,\cdot),\ \pi_k^{(t+1)} - \pi_k^{(t)} \right\rangle.
\end{equation}

We now apply Pinsker’s inequality, which states:
\begin{equation}
\label{eq:pinsker}
\operatorname{KL}(p,q) \ge \frac{1}{2} \|p - q\|_1^2,
\end{equation}
for any two distributions $p, q$. Using~\eqref{eq: inner-split} and~\eqref{eq:pinsker}, we obtain:
\begin{align*}
\eta_k^{(t)} \left\langle Q^{\pi^{(t)}}_{k,\widetilde{M}}(h_k,\cdot),\ \pi_k^{(t+1)} - \pi_k^{(t)} \right\rangle
&= \operatorname{KL}(\pi_k^{(t)}, \pi_k^{(t+1)}) + \operatorname{KL}(\pi_k^{(t+1)}, \pi_k^{(t)}) \\
&\quad - \left\langle e_k^{(t)}(h_k,\cdot),\ \pi_k^{(t+1)} - \pi_k^{(t)} \right\rangle \\
&\ge \| \pi_k^{(t+1)} - \pi_k^{(t)} \|_1^2 - \| e_k^{(t)}(h_k,\cdot) \|_\infty \cdot \| \pi_k^{(t+1)} - \pi_k^{(t)} \|_1.
\end{align*}

To bound the cross term, we apply the inequality \( ab \le \frac{1}{2}a^2 + \frac{1}{2}b^2 \) and get:
\[
\| \pi_k^{(t+1)} - \pi_k^{(t)} \|_1^2 - \| e_k^{(t)} \|_\infty \cdot \| \pi_k^{(t+1)} - \pi_k^{(t)} \|_1
\ge -\frac{1}{4} \| e_k^{(t)}(h_k,\cdot) \|_\infty^2.
\]

Therefore, for any $h_k \in \mathcal{H}_k$, we conclude:
\[
\left\langle Q^{\pi^{(t)}}_{k,\widetilde{M}}(h_k,\cdot),\ \pi_k^{(t+1)} - \pi_k^{(t)} \right\rangle
+ \frac{1}{4 \eta_k^{(t)}} \| e_k^{(t)}(h_k,\cdot) \|_\infty^2 \ge 0.
\]
\end{proof}


\begin{proof}[Proof of Lemma \ref{lemma: <Q,pi* - pi_t>}]
We begin by applying Lemma~\ref{lemma: 3 KL terms}, which states that for any $p \in \Delta(\mathcal{A}_k)$,
\begin{equation*}
\begin{split}
&\quad\operatorname{KL}(p,\ \pi_k(h_k,\theta_k^{(t)})) - \operatorname{KL}(p,\ \pi_k(h_k,\theta_k^{(t+1)})) - \operatorname{KL}(\pi_k(h_k,\theta_k^{(t+1)}),\ \pi_k(h_k,\theta_k^{(t)})) \\
&= \left\langle f_k(\theta_k^{(t+1)}, h_k,\cdot) - f_k(\theta_k^{(t)}, h_k,\cdot),\ p(\cdot) - \pi_k(\cdot|h_k,\theta_k^{(t+1)}) \right\rangle.
\end{split}
\end{equation*}

Letting $p = \pi_k^\dagger(\cdot|h_k)$ and using the definition of $e_k^{(t)}$ in \eqref{eq: e_t}, we have:
\[
f_k(\theta_k^{(t+1)}, h_k, a) - f_k(\theta_k^{(t)}, h_k, a) = e_k^{(t)}(h_k,a) + \eta_k^{(t)} Q^{\pi^{(t)}}_{k,\widetilde{M}}(h_k,a).
\]

Substituting into the identity above:
\begin{equation}
\label{eq:KL-Q-split}
\begin{split}
&\eta_k^{(t)} \left\langle Q^{\pi^{(t)}}_{k,\widetilde{M}}(h_k,\cdot),\ \pi_k^\dagger(\cdot|h_k) - \pi_k^{(t+1)}(\cdot|h_k) \right\rangle \\
= &\operatorname{KL}(\pi_k^\dagger(h_k),\ \pi_k^{(t)}(h_k)) - \operatorname{KL}(\pi_k^\dagger(h_k),\ \pi_k^{(t+1)}(h_k)) - \operatorname{KL}(\pi_k^{(t+1)}(h_k),\ \pi_k^{(t)}(h_k)) \\
&- \left\langle e_k^{(t)}(h_k,\cdot),\ \pi_k^\dagger(\cdot|h_k) - \pi_k^{(t+1)}(\cdot|h_k) \right\rangle.
\end{split}
\end{equation}

We now express the desired inner product as:
\begin{equation*}
\begin{split}
&\left\langle Q^{\pi^{(t)}}_{k,\widetilde{M}}(h_k,\cdot),\ \pi_k^\dagger(\cdot|h_k) - \pi_k^{(t)}(\cdot|h_k) \right\rangle \\
= &\left\langle Q^{\pi^{(t)}}_{k,\widetilde{M}}(h_k,\cdot),\ \pi_k^\dagger(\cdot|h_k) - \pi_k^{(t+1)}(\cdot|h_k) \right\rangle + \left\langle Q^{\pi^{(t)}}_{k,\widetilde{M}}(h_k,\cdot),\ \pi_k^{(t+1)}(\cdot|h_k) - \pi_k^{(t)}(\cdot|h_k) \right\rangle.
\end{split}
\end{equation*}

Substituting from \eqref{eq:KL-Q-split} and using Pinsker's inequality \eqref{eq:pinsker}:
\begin{equation*}
\begin{split}
&\quad\left\langle Q^{\pi^{(t)}}_{k,\widetilde{M}}(h_k,\cdot),\ \pi_k^\dagger(\cdot|h_k) - \pi_k^{(t)}(\cdot|h_k) \right\rangle \\
&\le \frac{1}{\eta_k^{(t)}} \left[ \operatorname{KL}(\pi_k^\dagger(h_k),\ \pi_k^{(t)}(h_k)) - \operatorname{KL}(\pi_k^\dagger(h_k),\ \pi_k^{(t+1)}(h_k)) \right] \\
&\quad - \frac{1}{2 \eta_k^{(t)}} \left\| \pi_k^{(t+1)}(h_k) - \pi_k^{(t)}(h_k) \right\|_1^2 + \left\| Q^{\pi^{(t)}}_{k,\widetilde{M}}(h_k,\cdot) \right\|_\infty \left\| \pi_k^{(t+1)}(h_k) - \pi_k^{(t)}(h_k) \right\|_1 \\
&\quad + \frac{1}{\eta_k^{(t)}} \left| \left\langle e_k^{(t)}(h_k,\cdot),\ \pi_k^\dagger(\cdot|h_k) - \pi_k^{(t+1)}(\cdot|h_k) \right\rangle \right|.
\end{split}
\end{equation*}

Applying Cauchy–Schwarz inequality and \( \| \pi \|_1 = 1 \), we upper bound the last term by \( \|e_k^{(t)}(h_k,\cdot)\|_\infty \cdot \| \pi_k^\dagger - \pi_k^{(t+1)} \|_1 \le 2 \|e_k^{(t)}(h_k,\cdot)\|_\infty \). Then completing the square on the second line leads to the final result:
$$
\begin{aligned}
   \left\langle Q^{\pi^{(t)}}_{k,\widetilde{M}}(h_k,\cdot),\ \pi_k^\dagger(\cdot|h_k) - \pi_k^{(t)}(\cdot|h_k) \right\rangle 
&\le \frac{1}{\eta_k^{(t)}} \left[ \operatorname{KL}(\pi_k^\dagger(h_k),\ \pi_k^{(t)}(h_k)) - \operatorname{KL}(\pi_k^\dagger(h_k),\ \pi_k^{(t+1)}(h_k)) \right] 
\\
&\quad+ \frac{\eta_k^{(t)}}{2} \| Q^{\pi^{(t)}}_{k,\widetilde{M}} \|_\infty^2 + \frac{2}{\eta_k^{(t)}} \| e_k^{(t)} \|_\infty. 
\end{aligned}
$$
\end{proof}

\subsection{Supporting Lemmas and Their Proofs  for Theorem \ref{thm: M^tilde + sieve reg}}

\label{subsec:p-smooth}

\begin{definition}[$p$-smoothness]
\label{definition: p-smooth}
Let \( \mathcal{X} \subset \mathbb{R}^d \) be a compact axis-aligned box, i.e., the Cartesian product of \( d \) closed intervals.

\begin{enumerate}[(i)]
    \item Let \( 0 < \beta \le 1 \). A function \( h: \mathcal{X} \to \mathbb{R} \) is said to satisfy a H\"older condition with exponent \( \beta \) if there exists a constant \( c > 0 \) such that
    \[
    |h(x) - h(y)| \le c \|x - y\|_2^{\beta}, \quad \text{for all } x, y \in \mathcal{X}.
    \]

    \item Let \( p > 0 \), and write \( m = \lfloor p \rfloor \) and \( \beta = p - m \). For any multi-index \( \kappa = (\kappa_1, \dots, \kappa_d) \in \mathbb{N}_0^d \), define the differential operator $ D^\kappa := \frac{\partial^{\|\kappa\|_1}}{\partial x_1^{\kappa_1} \cdots \partial x_d^{\kappa_d}}, \text{where } \|\kappa\|_1 = \sum_{j=1}^d \kappa_j.$ A function \( h: \mathcal{X} \to \mathbb{R} \) is said to be \( p \)-smooth if the following two conditions are satisfied:
    \begin{enumerate}[(a)]
        \item \( h \in C^m(\mathcal{X}) \), i.e., all partial derivatives up to order \( m \) exist and are continuous on \( \mathcal{X} \),
        \item for every multi-index \( \kappa \in \mathbb{N}_0^d \) with \( \|\kappa\|_1 = m \), the partial derivative \( D^\kappa h \) satisfies a H\"older condition with exponent \( \beta \) on \( \mathcal{X} \).
    \end{enumerate}
\end{enumerate}
\end{definition}

\begin{lemma}
\label{lemma: holder condition}
Let \( \mathcal{X}_1, \dots, \mathcal{X}_n \subset \mathbb{R} \) be closed intervals, and define \( \mathcal{X} = \mathcal{X}_1 \times \cdots \times \mathcal{X}_n \subset \mathbb{R}^n \). Let \( 0 < \beta \le 1 \). Then:

\begin{enumerate}[(i)]
    \item If functions \( f, g: \mathcal{X} \to \mathbb{R} \) both satisfy a H\"older condition with exponent \( \beta \), then the product \( fg \) also satisfies a H\"older condition with exponent \( \beta \).

    \item Suppose \( f \in C^1(\mathcal{X}) \), and for each \( i = 1, \dots, n \), the partial derivative \( \frac{\partial f}{\partial x_i} \) satisfies a H\"older condition with exponent \( \beta \). Then \( f \) itself satisfies a H\"older condition with exponent \( \beta \).

    \item Let \( m \in \mathbb{N} \), and suppose \( f \in C^m(\mathcal{X}) \). If for every multi-index \( \kappa \in \mathbb{N}_0^n \) with \( \|\kappa\|_1 = m \), the partial derivative \( D^\kappa f \) satisfies a H\"older condition with exponent \( \beta \), then the same holds for all multi-indices \( \kappa \) with \( \|\kappa\|_1 \le m \).
\end{enumerate}
\end{lemma}

\begin{proof}[Proof of Lemma~\ref{lemma: holder condition}]
The notation follows Definition~\ref{definition: p-smooth}.

(i) Since both \( f \) and \( g \) satisfy a H\"older condition with exponent \( \beta \), there exist constants \( c_f, c_g > 0 \) such that
\begin{equation*}
|f(x) - f(y)| \le c_f \|x - y\|_2^\beta, \quad 
|g(x) - g(y)| \le c_g \|x - y\|_2^\beta, \quad \forall x, y \in \mathcal{X}.
\end{equation*}
As \( \mathcal{X} \) is compact and \( f \), \( g \) are continuous, we have
\[
|f(x)| \le C_f, \quad |g(x)| \le C_g, \quad \forall x \in \mathcal{X},
\]
for some constants \( C_f, C_g > 0 \). Then for any \( x, y \in \mathcal{X} \),
\[
\begin{aligned}
|f(x)g(x) - f(y)g(y)| 
&= |f(x)(g(x) - g(y)) + g(y)(f(x) - f(y))| \\
&\le |f(x)||g(x) - g(y)| + |g(y)||f(x) - f(y)| \\
&\le C_f c_g \|x - y\|_2^\beta + C_g c_f \|x - y\|_2^\beta \\
&= (C_f c_g + C_g c_f) \|x - y\|_2^\beta,
\end{aligned}
\]
which confirms that \( fg \) satisfies a H\"older condition with exponent \( \beta \).

(ii) Let \( f_i' := \frac{\partial f}{\partial x_i} \) for \( i = 1, \dots, n \), and define the directional derivative of \( f \) at point \( x \) in the direction \( v \in \mathbb{R}^n \) as
\[
f'_v(x) := \lim_{h \to 0^+} \frac{f(x + h v) - f(x)}{h} = \langle v, \nabla f(x) \rangle.
\]
By assumption, for each \( i \), there exists \( c_i > 0 \) such that
\[
|f_i'(x) - f_i'(y)| \le c_i \|x - y\|_2^\beta, \quad \forall x, y \in \mathcal{X}.
\]
Let \( v \in \mathbb{R}^n \) be a unit vector and fix \( y \in \mathcal{X} \), \( t > 0 \) such that \( y + t v \in \mathcal{X} \). Then,
\[
\begin{aligned}
|f'_v(y + t v) - f'_v(y)| 
&= \left| \sum_{i=1}^n v_i (f_i'(y + t v) - f_i'(y)) \right| \\
&\le \sum_{i=1}^n |v_i| \cdot |f_i'(y + t v) - f_i'(y)| \\
&\le \sum_{i=1}^n c_i t^\beta =: c t^\beta.
\end{aligned}
\]
Since \( \mathcal{X} \) is compact and \( f \in C^1(\mathcal{X}) \), we define \( C_i := \sup_{x \in \mathcal{X}} |f_i'(x)| < \infty \), and let \( C := \sum_{i=1}^n C_i \). Then for any \( x, y \in \mathcal{X} \), write \( x - y = s v \) for some \( s = \|x - y\|_2 \ge 0 \) and unit vector \( v \). Since \( f \in C^1 \),
\[
f(x) - f(y) = \int_0^s f'_v(y + t v) \, dt.
\]
Using the bound on \( f'_v \), we have:
\[
\begin{aligned}
|f(x) - f(y)| 
&\le \int_0^s |f'_v(y + t v)| \, dt \\
&\le \int_0^s (|f'_v(y)| + c t^\beta) \, dt \\
&\le C s + \frac{c}{\beta + 1} s^{\beta + 1}.
\end{aligned}
\]
Since \( \mathcal{X} \) is compact, there exists \( d > 0 \) such that \( \|x - y\|_2 \le d \). Then
\[
|f(x) - f(y)| \le \left(C d^{1 - \beta} + \frac{c}{\beta + 1} d \right) \|x - y\|_2^\beta,
\]
showing that \( f \) satisfies a H\"older condition with exponent \( \beta \).

(iii) Let \( m \in \mathbb{N} \), and suppose that \( f \in C^m(\mathcal{X}) \), and for all multi-indices \( \kappa \in \mathbb{N}_0^n \) with \( \|\kappa\|_1 = m \), the derivative \( D^\kappa f \) satisfies a H\"older condition with exponent \( \beta \). We now show by induction that the same holds for all \( \kappa \) with \( \|\kappa\|_1 < m \).

Let \( \kappa \in \mathbb{N}_0^n \) with \( \|\kappa\|_1 = m - 1 \). Then \( D^\kappa f \in C^1(\mathcal{X}) \), and for each \( i \), the partial derivative \( \frac{\partial}{\partial x_i} D^\kappa f = D^{\kappa + e_i} f \) is of order \( m \), and by assumption satisfies a H\"older condition with exponent \( \beta \). By part (ii), \( D^\kappa f \) also satisfies a H\"older condition with exponent \( \beta \).

Repeating this argument inductively for all multi-indices \( \kappa \) with \( \|\kappa\|_1 = m - 2, m - 3, \dots, 0 \), we conclude that \( D^\kappa f \) satisfies a H\"older condition with exponent \( \beta \) for all \( \|\kappa\|_1 \le m \).
\end{proof}






\begin{lemma}
\label{lemma: p-smooth properties}
Let \( \mathcal{X}_1, \dots, \mathcal{X}_n \subset \mathbb{R} \) be closed intervals, and define \( \mathcal{X} = \mathcal{X}_1 \times \cdots \times \mathcal{X}_n \subset \mathbb{R}^n \). The following properties hold for \( p \)-smooth functions on \( \mathcal{X} \):

\begin{enumerate}[(i)]
    \item If \( f \) and \( g \) are \( p \)-smooth on \( \mathcal{X} \), and \( k \in \mathbb{R} \), then \( kf \) and \( f + g \) are also \( p \)-smooth on \( \mathcal{X} \).
    
    \item If \( f \) and \( g \) are \( p \)-smooth on \( \mathcal{X} \), then their product \( fg \) is \( p \)-smooth on \( \mathcal{X} \).
    
    \item Let \( f \) be \( p \)-smooth on \( \mathcal{X} \). Then for any fixed \( t \in \mathcal{X}_n \), the function \( g: \mathcal{X}_1 \times \cdots \times \mathcal{X}_{n-1} \to \mathbb{R} \) defined by
    \(
    g(x_1, \dots, x_{n-1}) := f(x_1, \dots, x_{n-1}, t)
    \)
    is \( p \)-smooth on \( \mathcal{X}_1 \times \cdots \times \mathcal{X}_{n-1} \).
    
    \item Suppose \( f \) is \( p \)-smooth on \( \mathcal{X} \). Define the function \( h: \mathcal{X}_1 \times \cdots \times \mathcal{X}_{n-1} \to \mathbb{R} \) by
    \[
    h(x_1, \dots, x_{n-1}) := \int_{\mathcal{X}_n} f(x_1, \dots, x_n) \, dx_n.
    \]
    Then \( h \) is \( p \)-smooth on \( \mathcal{X}_1 \times \cdots \times \mathcal{X}_{n-1} \).
\end{enumerate}
\end{lemma}

\begin{proof}[Proof of Lemma~\ref{lemma: p-smooth properties}]
Let \( m = \lfloor p \rfloor \) and \( \beta = p - m \).

(i) Since \( f, g \in C^m(\mathcal{X}) \), it follows immediately that \( kf \in C^m(\mathcal{X}) \) and \( f + g \in C^m(\mathcal{X}) \) for any \( k \in \mathbb{R} \). Moreover, for any multi-index \( \kappa \in \mathbb{N}_0^n \) with \( \|\kappa\|_1 = m \), we have:
\[
D^\kappa(kf) = k D^\kappa f, \quad D^\kappa(f + g) = D^\kappa f + D^\kappa g.
\]
Thus, if \( D^\kappa f \) and \( D^\kappa g \) satisfy a H\"older condition with exponent \( \beta \), then so do \( D^\kappa(kf) \) and \( D^\kappa(f + g) \).

(ii) Since \( f, g \in C^m(\mathcal{X}) \), their product \( fg \in C^m(\mathcal{X}) \) by standard calculus. Fix a multi-index \( \kappa \in \mathbb{N}_0^n \) with \( \|\kappa\|_1 = m \). The product rule for higher-order derivatives gives:
\begin{equation}
\label{eq:product_rule_multiindex}
D^\kappa(fg) = \sum_{\mu \le \kappa} \binom{\kappa}{\mu} \left(D^\mu f\right)\left(D^{\kappa - \mu} g\right),
\end{equation}
where the sum is over all multi-indices \( \mu \in \mathbb{N}_0^n \) such that \( \mu_i \le \kappa_i \) for all \( i \), and \( \binom{\kappa}{\mu} := \prod_{i=1}^n \binom{\kappa_i}{\mu_i} \) is the multi-index binomial coefficient.

By Lemma~\ref{lemma: holder condition}(iii), the \( p \)-smoothness of \( f \) and \( g \) implies that for every \( \mu \le \kappa \), both \( D^\mu f \) and \( D^{\kappa - \mu} g \) satisfy a H\"older condition with exponent \( \beta \). By Lemma~\ref{lemma: holder condition}(i), their product is also H\"older continuous with exponent \( \beta \). Since the sum in \eqref{eq:product_rule_multiindex} is finite, and all terms are H\"older with the same exponent, it follows that \( D^\kappa(fg) \) also satisfies a H\"older condition with exponent \( \beta \). Hence, \( fg \) is \( p \)-smooth.

(iii) Fix \( t \in \mathcal{X}_n \), and define
\(
g(x_1, \dots, x_{n-1}) := f(x_1, \dots, x_{n-1}, t).
\)
Since \( f \in C^m(\mathcal{X}) \), all partial derivatives of \( f \) up to order \( m \) exist and are continuous. The same holds for the restriction \( g \), viewed as a function on \( \mathcal{X}_1 \times \cdots \times \mathcal{X}_{n-1} \).

Now fix a multi-index \( \kappa \in \mathbb{N}_0^{n-1} \) with \( \|\kappa\|_1 = m \). Then
\[
D^\kappa g(x_1, \dots, x_{n-1}) = D^\kappa f(x_1, \dots, x_{n-1}, t).
\]
Since \( f \) is \( p \)-smooth, \( D^\kappa f \) satisfies a H\"older condition with exponent \( \beta \), and restricting the final coordinate to \( t \) preserves the inequality:
\[
\begin{aligned}
|D^\kappa g(x) - D^\kappa g(y)| 
&= |D^\kappa f(x_1, \dots, x_{n-1}, t) - D^\kappa f(y_1, \dots, y_{n-1}, t)| \\
&\le c \|x - y\|_2^\beta.
\end{aligned}
\]
Thus, \( g \) is \( p \)-smooth.

(iv) Define
\[
h(x_1, \dots, x_{n-1}) := \int_{\mathcal{X}_n} f(x_1, \dots, x_n) \, dx_n.
\]
Since \( f \in C^m(\mathcal{X}) \), standard results imply that integration preserves differentiability: \( h \in C^m(\mathcal{X}_1 \times \cdots \times \mathcal{X}_{n-1}) \), and for any multi-index \( \kappa \in \mathbb{N}_0^{n-1} \) with \( \|\kappa\|_1 = m \), we have
\[
D^\kappa h(x_1, \dots, x_{n-1}) = \int_{\mathcal{X}_n} D^\kappa f(x_1, \dots, x_n) \, dx_n.
\]
Let \( x = (x_1, \dots, x_{n-1}) \), \( y = (y_1, \dots, y_{n-1}) \). Then:
\[
\begin{aligned}
|D^\kappa h(x) - D^\kappa h(y)| 
&= \left| \int_{\mathcal{X}_n} \left[ D^\kappa f(x, x_n) - D^\kappa f(y, x_n) \right] dx_n \right| \\
&\le \int_{\mathcal{X}_n} \left| D^\kappa f(x, x_n) - D^\kappa f(y, x_n) \right| dx_n \\
&\le \int_{\mathcal{X}_n} c \|x - y\|_2^\beta dx_n \\
&= c |\mathcal{X}_n|\|x - y\|_2^\beta,
\end{aligned}
\]
where \( |\mathcal{X}_n|\) is the Lebesgue measure of \( \mathcal{X}_n \). It follows that \( D^\kappa h \) satisfies a H\"older condition with exponent \( \beta \), and thus \( h \) is \( p \)-smooth.
\end{proof}

\begin{lemma}
\label{lemma: Q function is p-smooth}
Suppose Assumptions \ref{assumption: reward is p-smooth}, \ref{assumption: transition is p-smooth}, \ref{assumption: policy is p-smooth} hold, then $Q_{k,\widetilde{M}}^{\pi^{(t)}} (h_k, a_k)$ is $p$-smooth in $(s_1,s_2,...,s_k)$ for any $0\le t\le T-1$, $1\le k\le K$, and any fixed $(a_1,...,a_k) \in \mathcal{A}_1 \times \cdots \times \mathcal{A}_k$.
\end{lemma}

\begin{proof}[Proof of Lemma \ref{lemma: Q function is p-smooth}]
Recall that the Algorithm \ref{algo: sieve regression} only update the policy within the policy class $\Pi$, so $\pi^{(t)} \in \Pi$ for any $0\le t\le T-1$. We only need to prove $Q_{k,\widetilde{M}}^{\pi} (h_k, a_k)$ is $p$-smooth in $(s_1,s_2,...,s_k)$ for any $\pi \in \Pi$, any $k$, any $(a_1,...,a_k) \in \mathcal{A}_1 \times \cdots \times \mathcal{A}_k$.

We prove by induction on $k$, from $k=K$ to $k=1$. For $k=K$, we trivially have
\begin{equation*}
Q_{K,\widetilde{M}} ^{\pi} (h_{K},a_{K}) = \widetilde{r}_K (h_K,a_K),
\end{equation*}
so $Q_{K,\widetilde{M}} ^{\pi} (h_{K},a_{K})$ is $p$-smooth in $(s_1,..., s_{K})$ under Assumption \ref{assumption: reward is p-smooth}.

Now we focus on the induction step from $k+1$ to $k$. Suppose $Q_{k+1,\widetilde{M}} ^{\pi} (h_{k+1},a_{k+1})$ is $p$-smooth in $(s_1,..., s_{k+1})$ for some $1\le k\le K-1$. By Bellman equation,
\begin{equation}
\label{eq: proof lemma Q function is p-smooth, eq 1}
\begin{split}
Q_{k,\widetilde{M}} ^{\pi} (h_k,a_k) =& \ \widetilde{r}_k (h_k,a_k) \\
&+ \int_{\mathcal{S}_k} \left( \sum_{a_{k+1} \in \mathcal{A}_{k+1}} \pi_{k+1}(a_{k+1}|h_{k+1}) Q_{k+1, \widetilde{M}} ^{\pi} (h_{k+1},a_{k+1}) \right) \widehat{p}_k (s_{k+1}| h_k, a_k) ds_{k+1}.
\end{split}
\end{equation}
By induction hypothesis, $Q_{k+1,\widetilde{M}} ^{\pi} (h_{k+1},a_{k+1})$ is $p$-smooth in $(s_1,..., s_{k+1})$. By Assumption \ref{assumption: policy is p-smooth}, $\pi_{k+1}(a_{k+1}|h_{k+1})$ is $p$-smooth in $(s_1,..., s_{k+1})$ for any $a_{k+1}\in \mathcal{A}_{k+1}$. Note that $|\mathcal{A}_{k+1}|$ is finite, then by Lemma \ref{lemma: p-smooth properties}(i)(ii), $\sum_{a_{k+1} \in \mathcal{A}_{k+1}} \pi_{k+1} Q_{k+1, \widetilde{M}} ^{\pi}$ is $p$-smooth in $(s_1,..., s_{k+1})$. Similarly, Assumption \ref{assumption: transition is p-smooth} and Lemma \ref{lemma: p-smooth properties}(ii) implies $\sum_{a_{k+1}} \pi_{k+1} Q_{k+1, \widetilde{M}} ^{\pi} \widehat{p}_k (s_{k+1}| h_k, a_k)$ is $p$-smooth in $(s_1,..., s_{k+1})$. By Lemma \ref{lemma: p-smooth properties}(iv), the integral $\int_{\mathcal{S}_k} \sum_{a_{k+1}} \pi_{k+1} Q_{k+1, \widetilde{M}} ^{\pi} \widehat{p}_k ds_{k+1}$ is $p$-smooth in $(s_1,..., s_{k})$. By Assumption \ref{assumption: reward is p-smooth}, $\widetilde{r}_k (h_k,a_k)$ is $p$-smooth in $(s_1,..., s_{k})$. Using Lemma \ref{lemma: p-smooth properties}(i) again, it follows from (\ref{eq: proof lemma Q function is p-smooth, eq 1}) that $Q_{k,\widetilde{M}} ^{\pi} (h_{k},a_{k})$ is $p$-smooth in $(s_1,..., s_{k})$, which finishes the induction step.

\end{proof}

\begin{proof}[Proof of Lemma \ref{lemma: upperbound e}]
Fix $\overline{a}_k^{(i)}$ and consider Theorem 2.1 from \cite{Chen_2015}. The correspondence between their notation and ours is as follows. The variable $X_i$ in their framework corresponds to $\overline{s}_k^{(i)}$ in ours. The function $h_0(x)$ corresponds to the scaled quantity:
\[
h_0(x) \leftrightarrow \frac{1}{\eta_k^{(t)} \| Q_{k,\widetilde{M}}^{\pi^{(t)}} \|_\infty } Q_{k,\widetilde{M}}^{\pi^{(t)}}(\overline{s}_k, \overline{a}_k).
\]
The variable $Y_i $ corresponds to the estimated $Q$-function scaled in the same manner:
\[
Y_i (x) \leftrightarrow \frac{1}{\eta_k^{(t)} \| Q_{k,\widetilde{M}}^{\pi^{(t)}} \|_\infty } \widehat{Q}_k^{(t)}(\overline{s}_k^{(i)}, \overline{a}_k).
\]
Finally, the estimator $\widehat{h}(x)$ in their setting corresponds to the scaled importance-weighting function difference between consecutive iterates:
\[
\widehat{h}(x) \leftrightarrow \frac{1}{\eta_k^{(t)} \| Q_{k,\widetilde{M}}^{\pi^{(t)}} \|_\infty } \left( f_k(\theta_k^{(t+1)}, \overline{s}_k, \overline{a}_k) - f_k(\theta_k^{(t)}, \overline{s}_k, \overline{a}_k) \right).
\]

Using equation~(\ref{eq: e_t}), we obtain the following decomposition:
\[
\frac{ e^{(t)}_k(h_k,a_k) }{ \eta_k^{(t)} \| Q_{k,\widetilde{M}}^{\pi^{(t)}} \|_\infty }= \frac{ f_k(\theta_k^{(t+1)}, h_k, a_k) - f_k(\theta_k^{(t)}, h_k, a_k) }{ \eta_k^{(t)} \| Q_{k,\widetilde{M}}^{\pi^{(t)}} \|_\infty } - \frac{ Q_{k,\widetilde{M}}^{\pi^{(t)}}(h_k,a_k) }{ \| Q_{k,\widetilde{M}}^{\pi^{(t)}} \|_\infty }.
\]

To apply Theorem 2.1 in \cite{Chen_2015}, we must verify that its assumptions hold in our setting. The i.i.d. condition on the sample $\{(X_i, Y_i)\}_{i=1}^n$ is directly implied by Assumption~\ref{assumption: MC error} (i). Assumptions 1 (i)-(iii) in \cite{Chen_2015} are satisfied because the samples $\{\overline{s}_k^{(i)}\}$ used in Algorithm~\ref{algo: sieve regression} are drawn independently from a distribution supported on a rectangular state space, as specified in Assumption \ref{assumption: state space is rectangular}. Assumption 1 (iv) in \cite{Chen_2015} is also satisfied due to the uniform sampling of $\{\overline{s}_k^{(i)}\}$ over the structured domain characterized by the same assumption.

Assumption 2(i) in \cite{Chen_2015} holds because of the i.i.d. nature of the sample inputs, and Assumptions 2(ii) and 2(iii) in \cite{Chen_2015} are guaranteed by the moment bounds and stability conditions provided in Assumption~\ref{assumption: MC error}(ii)–(iii). Assumption 6 in \cite{Chen_2015}, which requires smoothness of the target function class, is satisfied via Lemma~\ref{lemma: Q function is p-smooth}, which establishes the $p$-smoothness of the $Q$-function. Finally, Assumption 7 in \cite{Chen_2015}, concerning the regularity of the basis expansion, is directly enforced by our B-spline basis functions.

Thus, all conditions required by Theorem 2.1 in \cite{Chen_2015} are met, and the asserted probabilistic upper bound follows.
\end{proof}

\subsection{Supporting Lemmas' Proofs  for Theorem \ref{thm:suboptimality-GP}}
\label{pf:lemma thm3}

\begin{proof}[Proof of Lemma~\ref{lemma: Mercer}]
Assumption~\ref{assumption: GP kernel} ensures that the kernel \( h \) is continuous, symmetric, and positive semi-definite on the compact domain \( \mathcal{X} = \mathcal{H}_k \times \mathcal{A}_k \), and that it satisfies the integrability condition
\[
\int_{\mathcal{X} \times \mathcal{X}} h(x, x') \, \rho(x) \rho(x') \, dx \, dx' < \infty.
\]
These conditions match those required for Mercer's Theorem (see Section 12.3 of \citet{Wainwright2019}). Therefore, the existence of the spectral decomposition described in Equations~\eqref{eq: lemma Mercer eq 1} and~\eqref{eq: lemma Mercer eq 2} follows directly from an application of the theorem.
\end{proof}


\begin{proof}[Proof of Lemma~\ref{lemma: GP posterior expected variance}]
Let \( \mathcal{H} := L^2(\mathcal{X}, \rho) \) be the Hilbert space associated with the kernel \( h \), and let \( \{\mu_j\}_{j=1}^\infty \) and \( \{\psi_j\}_{j=1}^\infty \) denote the eigenvalues and eigenfunctions of the integral operator associated with \( h \), as in Lemma~\ref{lemma: Mercer}. By Mercer decomposition from Lemma ~\ref{lemma: Mercer}, the kernel admits the representation
\[
h(x, x') = \sum_{j=1}^{\infty} \mu_j \psi_j(x) \psi_j(x') = \langle \phi(x), \phi(x') \rangle_{\mathcal{H}},
\]
where the feature map is defined by \( \phi(x) := \sum_{j=1}^\infty \sqrt{\mu_j} \, \psi_j(x) \, \psi_j \in \mathcal{H} \).

Given \( n \) observations \( \{x_i\}_{i=1}^n \subset \mathcal{X} \), define the linear operator \( \Phi_n : \mathbb{R}^n \to \mathcal{H} \) by
\[
\Phi_n v := \sum_{i=1}^n v_i \phi(x_i),
\]
with adjoint \( \Phi_n^* : \mathcal{H} \to \mathbb{R}^n \) given by
\[
\Phi_n^* f = \left( \langle f, \phi(x_1) \rangle, \ldots, \langle f, \phi(x_n) \rangle \right)^\top.
\]
Then the Gram operator is given by \( H_{n,n} = \Phi_n^* \Phi_n \), and the kernel evaluation vector at any \( x \in \mathcal{X} \) is \( H_n(x) = \Phi_n^* \phi(x) \in \mathbb{R}^n \).

Using the standard formula for the posterior variance of a Gaussian process,
\[
\widehat{h}(x,x) = h(x,x) - H_n(x)^\top (H_{n,n} + \sigma^2 I)^{-1} H_n(x),
\]
and substituting $H_n(x)$ and $H_{n,n}$ into \eqref{eq: GP, eq 1}, we obtain
\begin{align}
\widehat{h}(x,x) &= \langle \phi(x), \phi(x) \rangle_{\mathcal{H}} - \langle \phi(x), \Phi_n ( \Phi_n^* \Phi_n + \sigma^2 I )^{-1} \Phi_n^* \phi(x) \rangle \nonumber \\
&= \langle \phi(x), A \phi(x) \rangle,
\end{align}
where the linear operator \( A \) is defined by $A := I - \Phi_n(\Phi_n^* \Phi_n + \sigma^2 I)^{-1} \Phi_n^*.$

We simplify \( A \) by observing that
\[
A \Phi_n = \sigma^2 \Phi_n (\Phi_n^* \Phi_n + \sigma^2 I)^{-1},
\quad \text{so} \quad
A \Phi_n \Phi_n^* = \sigma^2(I - A),
\]
which yields the identity
\begin{equation}
\label{eq: A definition}
A = (\sigma^2 I + \Phi_n \Phi_n^*)^{-1} \sigma^2 I
= \left(I + \tfrac{1}{\sigma^2} \sum_{i=1}^n \phi(x_i) \phi(x_i)^* \right)^{-1}.
\end{equation}

Taking expectation over \( x \sim \rho \), we may exchange the expectation with the quadratic form to obtain
\[
\mathbb{E}_{x \sim \rho}[\widehat{h}(x,x)] = \mathbb{E}[\langle \phi(x), A \phi(x) \rangle] = \operatorname{Tr}(A \, \mathbb{E}[\phi(x)\phi(x)^*]).
\]

To control this trace, we apply the matrix Bernstein inequality (Theorem 1.4 in \citep{Tropp_2011}): there exist positive constants \(c_{1},c_{2}>0\), such that
\[
\Pr\left\{ \left\| \tfrac{1}{n} \sum_{i=1}^n \phi(x_i)\phi(x_i)^* - \mathbb{E}[\phi(x)\phi(x)^*] \right\| \ge \tfrac{1}{2} \| \mathbb{E}[\phi(x)\phi(x)^*] \| \right\} \le c_1 e^{-c_2 n}.
\]

Define the event $\mathcal{G}
\;=\;
\Bigl\{\|\frac{1}{n}\sum_{i=1}^n \phi(x_i) \phi(x_i)^* - \mathbb{E}[\phi(x)\phi(x)^{*}]\|\;\le\;\tfrac{1}{2}\,\|\mathbb{E}[\phi(x)\phi(x)^{*}]\|\Bigr\}.$
On \(\mathcal{G}\), we have
\[
\tfrac{1}{n} \sum_{i=1}^n \phi(x_i)\phi(x_i)^* \succeq \tfrac{1}{2} \mathbb{E}[\phi(x)\phi(x)^*],
\]
which implies
\[
A = \left(I + \tfrac{n}{\sigma^2} \cdot \tfrac{1}{n} \sum_{i=1}^n \phi(x_i) \phi(x_i)^* \right)^{-1}
\preceq \left( I + \tfrac{n}{2\sigma^2} \mathbb{E}[\phi(x)\phi(x)^*] \right)^{-1} =: A_{\mathrm{pop}}.
\]
Hence,
\[
\operatorname{Tr}(A \, \mathbb{E}[\phi(x)\phi(x)^*]) \le \operatorname{Tr}(A_{\mathrm{pop}} \, \mathbb{E}[\phi(x)\phi(x)^*]).
\]

Let us now compute the right-hand side trace. Using the Mercer basis \( \{\psi_j\} \), we have
\[
\mathbb{E}[\phi(x)\phi(x)^*] = \sum_{j=1}^{\infty} \mu_j \psi_j \otimes \psi_j,
\quad
A_{\mathrm{pop}} = \sum_{j=1}^{\infty} \frac{2 \sigma^2}{2 \sigma^2 + n \mu_j} \psi_j \otimes \psi_j,
\]
which yields
\[
\operatorname{Tr}(A_{\mathrm{pop}} \, \mathbb{E}[\phi(x)\phi(x)^*]) = \sum_{j=1}^{\infty} \frac{2 \sigma^2 \mu_j}{2 \sigma^2 + n \mu_j}.
\]

To bound the sum, we partition the eigenvalues at the index \( j_0 \), where \( \mu_{j_0} \sim \sigma^2 / n \), i.e., \( j_0 \sim (n/\sigma^2)^{1/\omega} \) under the decay \( \mu_j = \mathcal{O}(j^{-\omega}) \) by Assumption \ref{assumption:eigenvalue decay}. Then
\[
\sum_{j=1}^{\infty} \frac{2 \sigma^2 \mu_j}{2 \sigma^2 + n \mu_j}
\lesssim \frac{\sigma^2}{n} j_0 + \sum_{j=j_0+1}^{\infty} \mu_j
\lesssim n^{-1} \cdot \left(\frac{n}{\sigma^2} \right)^{1/\omega} + \int_{j_0}^\infty j^{-\omega} \, dj
= \mathcal{O}(n^{1/\omega - 1}).
\]

On the complement event \( \mathcal{G}^c \), we trivially have \( A \preceq I \), so
\[
\mathbb{E}[\langle \phi(x), A \phi(x) \rangle \cdot \mathbf{1}_{\mathcal{G}^c}] \le \operatorname{Tr}(\mathbb{E}[\phi(x)\phi(x)^*]) \cdot \Pr(\mathcal{G}^c) = o(n^{1/\omega - 1}).
\]

Combining both cases concludes the proof:
\[
\mathbb{E}_{x \sim \rho}[\widehat{h}(x,x)] = \mathcal{O}(n^{1/\omega - 1}).
\]
\end{proof}



\begin{proof}[Proof of Lemma~\ref{lemma: GP posterior variance}]
We adopt the same notation as in the proof of Lemma~\ref{lemma: GP posterior expected variance}. Recall that the posterior variance at a point \( x \in \mathcal{X} \) can be written as $\widehat{h}(x,x) = \langle \phi(x), A \phi(x) \rangle_{\mathcal{H}},$
where \( A  \) is the positive semidefinite, self-adjoint operator.
By spectral decomposition, \( A = \sum_{i=1}^\infty \lambda_i u_i \otimes u_i \), where \( \{u_i\} \) is an orthonormal basis of \( \mathcal{H} \) and \( \lambda_i \ge 0\). Substituting this into the expression for \( \widehat{h}(x,x) \) yields
\[
\widehat{h}(x,x) = \sum_{i=1}^{\infty} \lambda_i \, |\langle u_i, \phi(x) \rangle|^2.
\]

Since \( \{u_i\} \) forms an orthonormal basis, we can expand \( \phi(x) = \sum_{i=1}^\infty \langle u_i, \phi(x) \rangle u_i \), which gives
\[\sum_{i=1}^\infty |\langle u_i, \phi(x) \rangle|^2=\| \phi(x) \|^2 .
\]
By Assumption~\ref{assumption: GP kernel}, the kernel is bounded on \( \mathcal{X} \times \mathcal{X} \), so we define \( M := \sup_{x \in \mathcal{X}} h(x,x)=\sup_{x \in \mathcal{X}} \| \phi(x) \|^2 \). This implies
\[
\sum_{i=1}^\infty |\langle u_i, \phi(x) \rangle|^2 \le M.
\]
and hence,
\begin{equation}
\label{eq: sup posterior bound}
\widehat{h}(x,x) \le  \sup_i \lambda_i \cdot M.
\end{equation}

It remains to bound the spectral norm of \( A \), i.e., \( \max_i \lambda_i \). Recall from the GP posterior derivation that
\[
A = \left(I + \tfrac{1}{\sigma^2} \sum_{i=1}^n \phi(x_i) \phi(x_i)^* \right)^{-1}.
\]
Let \( S := \frac{1}{n} \sum_{i=1}^n \phi(x_i) \phi(x_i)^* \) and note that
\[
\max_i \lambda_i = \left( 1 + \tfrac{n}{\sigma^2} \lambda_{\min}(S) \right)^{-1} \le 1.
\]

Combining with \eqref{eq: sup posterior bound}, we obtain the uniform upper bound:
\[
\widehat{h}(x,x) \le M,
\]
which holds uniformly for all \( x \in \mathcal{X} \).. 
\end{proof}


\begin{proof}[Proof of Lemma~\ref{Lemma:beta_n bound GP}]
By Theorem 24 in \cite{chang2021mitigating}, the following bound holds with probability at least \(1 - \delta\):
\begin{equation}
\label{eq: proof thm uncertainty quantifier GP part2 eq 1}
\log\left( \det \left( I_n + \tfrac{1}{\sigma^2} H_{n,n} \right) \right) \lesssim \left(d^*(n)\right)^2 \log n,
\end{equation}
where the effective dimension \( d^*(n) \) is defined as (see Definition 9 of \cite{chang2021mitigating}):
\[
d^*(n) := \min \left\{ J \in \mathbb{N} : \frac{J}{\sum_{i=J+1}^{\infty} \mu_i} \ge \frac{n}{\sigma^2} \right\}.
\]

Under Assumption~\ref{assumption:eigenvalue decay}, the eigenvalues satisfy \( \sum_{i=J+1}^{\infty} \mu_i \le \int_{J}^{\infty} i^{-\omega} \, di \lesssim J^{-(\omega - 1)} \). This implies that
\[
\frac{J}{\sum_{i=J+1}^{\infty} \mu_i} \gtrsim J^\omega,
\]
and thus the definition of \( d^*(n) \) yields $d^*(n) \lesssim n^{\frac{1}{\omega}}.$

Substituting into the expression for \( \beta_{k,n} \) in equation~\eqref{eq: GP, eq 4}, we obtain
\[
\beta_n \lesssim (\log n)^2 \cdot d^*(n) \lesssim (\log n)^2 \cdot n^{\frac{1}{\omega}}.
\]
This completes the proof.
\end{proof}

\section{Simulation Setup} \label{sec:simulation_setup}

The transition matrix is designed to ensure that each state undergoes a significant transformation after applying the transition dynamics, which are set as follows:

\begin{equation*}
W_1^{0} =
\begin{bmatrix}
0.4 & 0.2 & 0 \\
0.4 & 0 & 0.2
\end{bmatrix}, 
\quad 
W_1^{1} =
\begin{bmatrix}
0.6 & 0 & -0.2 \\
0.4 & 0.2 & 0
\end{bmatrix}.
\end{equation*}

\begin{equation*}
W_2^{0} =
\begin{bmatrix}
0.5 & 0.1 & -0.1 \\
0.5 & -0.1 & 0.1
\end{bmatrix}, 
\quad 
W_2^{1} =
\begin{bmatrix}
0.5 & -0.1 & 0.1 \\
0.5 & 0.1 & -0.1
\end{bmatrix}.
\end{equation*}

\begin{equation*}
W_3^{0} =
\begin{bmatrix}
0.6 & -0.12 & -0.08 \\
0.6 & -0.08 & -0.12
\end{bmatrix}, 
\quad 
W_3^{1} =
\begin{bmatrix}
0.4 & 0.08 & 0.12 \\
0.4 & 0.12 & 0.08
\end{bmatrix}.
\end{equation*}

The noises are defined as \(\varepsilon_k^i=0.8\left(z_k^i-0.5\right), \text { where } z_k^i \sim \operatorname{Beta}(2,2) \text { so that } \varepsilon_k^i \in[-0.4,0.4]\), \(i = 1, 2\).

\section{Further Details on Simulation Studies}
\subsection{Optimal Policy Approximation}
\label{sec: Dynamic Programming}

The optimal policy $\pi^*$ is approximated using Dynamic Programming. Specifically, for a randomly generated initial state $s_1$, the optimal action $a_1^*$ is determined by evaluating the $Q$-values for both possible actions (0 and 1) and selecting the one that maximizes the expected return.

The process of state transition and data construction follows a hierarchical structure, as illustrated below:
\[
s_1 \Rightarrow a_1^* \Rightarrow 
\begin{bmatrix}
s_2^{(1)} \\
\vdots \\
s_2^{(n)}
\end{bmatrix}
\Rightarrow
\begin{bmatrix}
h_2^{(1)} \\
\vdots \\
h_2^{(n)}
\end{bmatrix}
\Rightarrow
\begin{bmatrix}
(h_2^{(1)}, a_2^{*(1)}) \\
\vdots \\
\vdots \\
(h_2^{(n)}, a_2^{*(n)})
\end{bmatrix}
\Rightarrow
\begin{bmatrix}
(h_2^{(1)}, a_2^{*(1)}, s_3^{(1,1)})  \\
\vdots \\
(h_2^{(1)}, a_2^{*(1)}, s_3^{(1,m)})  \\
\vdots \\
\vdots \\
(h_2^{(n)}, a_2^{*(n)}, s_3^{(n,1)})  \\
\vdots \\
(h_2^{(n)}, a_2^{*(n)}, s_3^{(n,m)})  
\end{bmatrix}
=
\begin{bmatrix}
h_3^{(1,1)} \\
\vdots \\
h_3^{(1,m)} \\
\vdots \\
\vdots \\
h_3^{(n,1)} \\
\vdots \\
h_3^{(n,m)}
\end{bmatrix}
\]

At each stage, a state transitions to multiple successor states based on the selected action, leading to a tree-structured expansion of possible future states. The history representations $h_k$ are recursively updated by incorporating previous states, actions, and newly sampled next states.

The optimal value function is recursively defined as follows: 
\[
V_1^*(s_1) = Q_1^*(s_1, a_1^*) = \frac{1}{n} \sum_{i=1}^{n} V_2^*(h_2^{(i)})
\]

where the optimal action at the first step is determined by the policy:
\[
a_1^* = \pi_1^*(s_1), \quad h_2^{(i)} = [s_1, a_1^*, s_2^{(i)}]
\]

and the next states are sampled according to the transition probability:
\[
s_2^{(1)}, \dots, s_2^{(n)} \sim P_1(\cdot | s_1, a_1^*)
\]

Similarly, at the next stage:
\[
V_2^*(h_2^{(i)}) = Q_2^*(h_2^{(i)}, a_2^{*(i)}) = \frac{1}{n} \sum_{j=1}^{n} V_3^*(h_3^{(i,j)})
\]

where
\[
a_2^{*(i)} = \pi_2^*(h_2^{(i)}), \quad s_3^{(i,1)}, \dots, s_3^{(i,n)} \sim P_2(\cdot | h_2^{(i)}, a_2^{*(i)})
\]

\[
h_3^{(i,j)} = [h_2^{(i)}, a_2^{*(i)}, s_3^{(i,j)}]
\]

As a result, the optimal value function at the initial state can be expressed as:
\[
V_1^*\left(s_1\right)=\frac{1}{n^2} \sum_{i=1}^n \sum_{j=1}^n V_3^*\left(h_3^{(i, j)}\right)
\]

We assume that the reward is generated only at the final stage ($r_1 = r_2 = 0$) and is defined by the following nonlinear function:
\[
r=3.8\left[\left(\cos \left(-s_3^1 \pi\right)+2 \cos \left(s_3^2 \pi\right)+s_4^1+2 s_4^2\right)\left(1+a_3\right)-1.37\right]
\]

Under this formulation, the optimal value function satisfies $V_1^*=1$, indicating that the policy successfully maximizes the expected reward given the defined transition dynamics and reward structure.

\subsection{DTR Q-learning} \label{sec:standard_q_learning}

\SetAlgoLined
\LinesNumbered
\SetAlgoLongEnd

\begin{algorithm}[ht]
\caption{DTR Q-learning}
\label{algo: q-learning}
\KwIn{Offline dataset $\{(s_1^{(i)}, a_1^{(i)}, s_2^{(i)}, a_2^{(i)}, s_3^{(i)}, a_3^{(i)}, s_4^{(i)}) \}_{i=1}^N$}

Assume: $Q_k(h_k, a_k; \theta_k) := \langle \theta_k, \Phi(h_k,a_k) \rangle$, and $\overline{r}_3(h_3, a_3, s_4)$ is nonzero, $\overline{r}_1 = \overline{r}_2 = 0$ \\

Solve the regression to obtain the estimator $\widehat{\theta_3}$ for $k =3$ by: 
        \[
        \widehat{\theta}_3 = \arg \min_{\theta_3} \frac{1}{N} \sum_{i=1}^N \left( \overline{r}_3( h_3^{(i)}, a_3^{(i)}, s_4^{(i)}) - Q_3(h_3^{(i)}, a_3^{(i)}; \theta_3) \right)^2
        \]\\
       Define $\widehat{Q}_3(h_3, a_3) := \max_{a_3 \in \{0,1\}} Q_3(h_3, a_3; \widehat{\theta}_3)$
        
 Solve for regression step $\widehat{\theta}_k$ by minimizing:
        \[
        \widehat{\theta}_k = \arg \min_{\theta_k} \frac{1}{N} \sum_{i=1}^N \left( \widehat{Q}_{k+1}( h_{k+1}^{(i)}, a_{k+1}^{(i)}) - Q_k(h_k^{(i)}, a_k^{(i)}; \theta_k) \right)^2
        \]\\
  Define $\widehat{Q}_k(h_k, a_k) := \max_{a_k \in \{0,1\}} Q_k(h_k, a_k; \widehat{\theta}_k)$

For $k = 1, 2, 3$, the estimated optimal policy $\widehat{\pi}_k(\cdot | h_k)$ can be set as any probability distribution supported on: $\arg \min_{a_k} \widehat{Q}_k(h_k, a_k)$\\
\KwOut{Estimated Q-functions $\widehat{Q}_k$ and policies $\widehat{\pi}_k$ for $k=1,2,3$}
\end{algorithm}

\newpage

\subsection{Sensitivity Analysis} \label{sec:sensitivity_analysis}

To assess robustness to transition model misspecification, we conduct a sensitivity analysis. We apply POLAR using a GP transition model to our simulated data generated from a linear transition model, deliberately introducing model mismatch. Specifically, for each decision stage \(k \in \{1,2,3\}\), we fit a GP regression model that maps the state-action input \((s_k, a_k) \in \mathbb{R}^3\) to the subsequent state \(s_{k+1} \in \mathbb{R}^2\). We employ a radial basis function (RBF) kernel with automatic relevance determination and use either exact or sparse variational GP regression depending on the sample size. The GP yields a conditional mean function \(\hat{\mu}_k(x) = \mathbb{E}[s_{k+1} \mid x]\), which serves as the plug-in estimate of the next state. The associated epistemic uncertainty is quantified via the kernel-induced predictive covariance.

Figure~\ref{fig:sensitivity_ananlysis} presents the performance of POLAR under this misspecification across various sample sizes \(n\), probability parameters \(p\), and pessimism coefficients \(c\), alongside comparisons with DTR-\(Q\), DTRreg, DDQN, MILO, and MOPO. The results demonstrate that POLAR continues to outperform all competing methods. Furthermore, the relative ordering of policy values across different pessimism coefficients $c$ remains consistent, indicating stable algorithmic behavior. This confirms that POLAR remains effective even when the transition model is misspecified.

\begin{figure}
    \centering
    \includegraphics[width=0.9\linewidth]{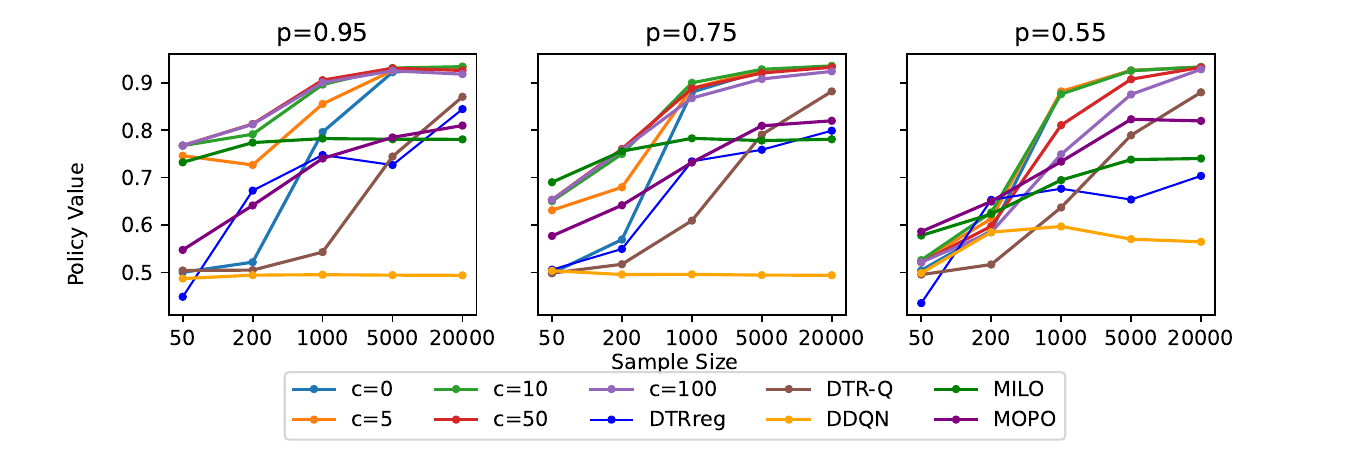}
    \caption{Policy values under transition model misspecification. Data is generated from a linear model, but a Gaussian Process (GP) model is used for estimation. Policy value averaged over 100 repeated simulations, for different $n,p,c$.}
    \label{fig:sensitivity_ananlysis}
\end{figure}
\cbend

\end{document}